%% file: conservative.tex
\documentclass[10pt, letterpaper]{article}

\usepackage{microtype}
\usepackage{graphicx}
\usepackage{subfigure}
\usepackage{booktabs} 

\usepackage{adjustbox}
\usepackage{makecell}

\usepackage{algorithm}
\usepackage{algorithmic}

\usepackage{enumitem}

\usepackage{wrapfig}

\usepackage{geometry}
\usepackage{hyperref}

\usepackage{amsmath}
\usepackage{amssymb}
\usepackage{mathtools}
\usepackage{amsthm}

\usepackage{dsfont}
\usepackage{natbib}
\usepackage[capitalize,noabbrev]{cleveref}
\allowdisplaybreaks
\theoremstyle{plain}
\newtheorem{theorem}{Theorem}[section]
\newtheorem{proposition}[theorem]{Proposition}
\newtheorem{lemma}[theorem]{Lemma}
\newtheorem{corollary}[theorem]{Corollary}
\theoremstyle{definition}
\newtheorem{definition}[theorem]{Definition}

\theoremstyle{remark}
\newtheorem{remark}[theorem]{Remark}

\input{math_commands.tex}

\input{personal_math_commands.tex}

\usepackage[textsize=tiny]{todonotes}
\newcommand{\hrq}[1]{{\color{brown} #1}}
\ifodd 1
\newcommand{\jingc}[1]{{\color{red}(Jing: #1)}}
\else
\newcommand{\jingc}[1]{}
\fi

\newcommand{\ldh}[1]{{\color{blue} #1}}

\title{Near-optimal Conservative Exploration in Reinforcement Learning under Episode-wise Constraints}

\author{
 {Donghao Li}\thanks{Equal contribution} \footnotemark[3] ,~~~
 {Ruiquan Huang}\footnotemark[1] \footnotemark[3],~~~
 {Cong Shen}\thanks{ECE Department, University of Virginia, Charlottesville, VA, USA},~~~
 {Jing Yang}\thanks{School of EECS, The Pennsylvania State University, University Park, PA, USA. Correspondence to:
Jing Yang $<$yangjing@psu.edu$>$}
}

\date{}

\begin{document}

\maketitle
 







\begin{abstract}
This paper investigates conservative exploration in reinforcement learning where the performance of the learning agent is guaranteed to be above a certain threshold throughout the learning process. It focuses on the tabular episodic Markov Decision Process (MDP) setting that has finite states and actions. With the knowledge of an existing safe baseline policy, an algorithm termed as StepMix is proposed to balance the exploitation and exploration while ensuring that the conservative constraint is never violated in each episode with high probability. StepMix features a unique design of a mixture policy that adaptively and smoothly interpolates between the baseline policy and the optimistic policy. Theoretical analysis shows that StepMix achieves near-optimal regret order as in the constraint-free setting, indicating that obeying the stringent episode-wise conservative constraint does not compromise the learning performance. Besides, a randomization-based EpsMix algorithm is also proposed and shown to achieve the same performance as StepMix. The algorithm design and theoretical analysis are further extended to the setting where the baseline policy is not given a priori but must be learned from an offline dataset, and it is proved that similar conservative guarantee and regret can be achieved if the offline dataset is sufficiently large. Experiment results corroborate the theoretical analysis and demonstrate the effectiveness of the proposed conservative exploration strategies.
\end{abstract}

\section{Introduction}
\label{sec:intro}

One of the major obstacles that prevent state-of-the-art reinforcement learning (RL) algorithms from being deployed in real-world systems is the lack of performance guarantee throughout the learning process. In particular, for many practical systems, a reasonable albeit not necessarily optimal {\it baseline policy} is often in place, and RL is later brought in as a (supposedly) superior solution to replace the baseline. System designers want the potentially better RL policy, but are also wary of the possible performance degradation incurred by exploration during the learning process. This dilemma exists in many domains, including digital marketing, robotics, autonomous driving, healthcare, and networking; see \citet{garcia2015comprehensive,wu2016conservative} for a detailed discussion of practical examples. It is desirable to have the RL algorithm perform nearly as well (or better) as the baseline policy {\it at all times}.

To address this challenge, {\it conservative exploration} has received increased interest in RL research over the past few years \citep{garcelon2020conservative,yang2021reduction,efroni2020exploration,zheng2020constrained,xu2020primal,liu2021learning}. 
In the online learning setting, exploration of the unknown environment is necessary for RL to learn about the underlying Markov Decision Process (MDP).  However, ``free'' exploration provides no guarantee on the RL performance, particularly in the early phases where the knowledge of the environment is minimal and the algorithm tends to explore almost randomly.  
To solve this problem, the vast majority of the conservative exploration literature relies on a key idea of invoking the baseline policy early on to build a conservative budget, which can be spent in later episodes to take explorative actions. This intuition, however, critically depends on the definition of the conservative constraint being the {\it cumulative} expected reward over a horizon falling below a certain threshold. 
If a more stringent constraint defined on a {\it per episode} basis is adopted, this idea becomes infeasible and it is unclear how conservative exploration can be achieved.

\begin{table*}[t]
    \caption{Comparison of Related Algorithms}
    \label{table:comp}
    \begin{center}
    \adjustbox{max width= \linewidth}{
    \begin{tabular}{ccccc}
    \Xhline{3\arrayrulewidth}
    Algorithm (Reference) & Regret & Violation & Constraint-type & Baseline Assumption \\
    \Xhline{2\arrayrulewidth}
BPI-UCBVI \citep{menard2021fast} & $\tilde{O}\left(\sqrt{H^3SAK}\right)$ & N/A & N/A & N/A \\
    \midrule
    OptPess-LP \citep{liu2021learning} & $\tilde{O}(\frac{1}{\kappa}\sqrt{H^6S^3AK})$ & $0$ & Episodic, general constraint & Type I \\
    \midrule
    DOPE \citep{bura2022dope} & $\tilde{O}(\frac{1}{\kappa}\sqrt{H^6S^2AK})$ & $0$ & Episodic, general constraint & Type I \\
    \midrule
    Budget-Exporation \citep{yang2021reduction} & $\tilde{O}\left(\sqrt{H^3SAK} + \frac{H^3SA\Delta_0}{\kappa{(\kappa+\Delta_0)}}\right)$ & $0$ & Cumulative, conservative constraint & Type I  \\
    \midrule
    StepMix / EpsMix (this work) & $\tilde{O}\left(\sqrt{H^3SAK}+\frac{H^3SA\Delta_0}{\kappa^2}\right)$ & $0$ & Episodic, conservative constraint & Type I and II \\
    \hline
    \vspace{-11pt}\\
    \hline
    Lower bound \citep{yang2021reduction} & $\Omega \left(\sqrt{H^3SAK} + \frac{H^3SA\Delta_0}{\kappa{(\kappa+\Delta_0)}}\right)$ & $0$ & Cumulative, conservative constraint & Type I \\
    \Xhline{3\arrayrulewidth}
    \end{tabular}
    }
    \end{center}
    
    \begin{center}
        \scriptsize
	    $\Delta_0$: suboptimality gap for the baseline policy; $\kappa$: tolerable reward loss from the baseline policy or the Slater parameter. 
     Type I assumes a known safe baseline policy. Type II assumes availability of an offline dataset generated by an unknown safe behavior policy. The lower bound automatically applies to our problem, {due to its weaker constraint}.
	\end{center}
	\vspace{-0.2in}
		
\end{table*}

In this paper, we focus on conservative exploration in an episodic MDP with finite states and actions. Unlike most of the prior works, we enforce a more strict conservative constraint that the expected reward of the RL policy cannot be much worse than that of a baseline policy {\it for every episode}. One fundamental question we aim to answer is:

\begin{center}
 {\it Is it possible to design a conservative exploration algorithm to achieve the optimal learning regret while satisfying the episode-wise conservative constraint throughout the learning process?}
\end{center}
In this work, we provide an affirmative answer to this question. Our main contributions are summarized as follows.
\begin{itemize}[leftmargin=*]\itemsep=0pt
    \item First, we investigate the scenario where a safe baseline policy is explicitly given upfront, and propose a model-based learning algorithm coined {StepMix}. In contrast to conventional linear programming or primal-dual based approaches in constrained MDPs~\citep{liu2021learning,bura2022dope,wei2022triple,efroni2020exploration}, StepMix features several unique design components. First, in order to achieve the optimal learning regret, StepMix relies on a Bernstein inequality-based design to closely track the estimation uncertainty in learning and construct an efficient optimistic policy correspondingly. Then, a set of candidate policies are explicitly constructed by smoothly interpolating between the safe baseline policy and the optimistic policy. Finally, a mixture of two candidate policies is obtained when necessary to achieve the near-optimal tradeoff between safe exploration and efficient exploitation. 
    \item Second, we theoretically analyze the performance of StepMix, and rigorously show that it achieves  $\tilde{O}(\sqrt{H^3SAK})$ regret, which is the order-optimal learning regret in the unconstrained setting, while never violating the conservative constraint during the learning process with high probability. The conservative constraint turns out to only incur an {\it additive} regret term, as opposed to a multiplicative coefficient in \citet{bura2022dope, liu2021learning}. Furthermore, the additive term differs from that in the lower bound in \citet{yang2021reduction} by a small constant factor, while our constraint is more stringent. Besides, we extend the analysis to a randomization mechanism-based EpsMix algorithm and show that it achieves the same learning regret as StepMix and satisfies the conservative constraint as well. A comparison of our work and these relevant papers is presented in Table~\ref{table:comp}.
    \item Next, instead of assuming a safe baseline policy is explicitly provided, we investigate the scenario where the agent only has access to an offline dataset collected under an unknown safe behavior policy. The agent thus needs to first extract an approximately safe baseline policy from the dataset and then to use it as an input to the StepMix or EpsMix algorithm. We explicitly characterize the impact of the dataset size and the quality of the behavior policy on the safety and regret of StepMix/EpsMix. Our results indicate that similar regret and safety guarantees can be achieved, as long as the dataset is sufficiently large. 
    \item Finally, due to the explicit algorithmic design of the optimistic policy, the candidate policies and the mixture policies, we are able to implement StepMix and EpsMix efficiently and validate their performances through synthetic experiments. The experimental results corroborate our theoretical findings, and showcase the superior performances of StepMix/EpsMix compared with other baseline algorithms. 
\end{itemize}

\if{0}
In this paper, we focus on constrained reinforcement learning with linear function approximation. Our {\bf goal} is to design provable RL algorithms that cope with practical constraints and large state spaces. Our contributions are two-folds.
\begin{itemize}
    \item A novel diversity model is studied. Based on this model, agent is able to explore the environment safely. We also demonstrate two examples of action spaces which satisfy the diversity model, which covering both the continuous and discrete case.
    \item A conservative algorithm is proposed inspired by UCB algorithm and the random exploration. While the random exploration is conservative, it also provides enough information to do UCB type of exploitation and exploration. Theoretical analysis shows that the algorithm achieves $\tilde{O}(\sqrt{d^3H^3 T})$  regret, which matches the regret bound of unconstrained RL.
\end{itemize}
\fi

\section{Related Works}
In this section, we briefly discuss existing works that are most relevant to our work. A detailed literature review is deferred to \Cref{appx:ref}.

\textbf{Unconstrained Episodic Tabular MDPs.} Unconstrained tabular MDPs have been well studied in the literature. For an episodic MDP with $S$ states, $A$ actions and horizon $H$, the minimax regret lower bound scales in $\Omega(\sqrt{H^3SAK})$ \citep{domingues2021episodic}, where $K$ denotes the number of episodes. Several algorithms have been proposed and shown to achieve the minimax lower bound (and thus order-optimal), including \citet{azar2017minimax, zanette2019tighter,menard2021fast}.

\textbf{Conservative Exploration.} Conservative exploration corresponds to the setting where a good baseline policy that may not be optimal is available, and the agent is required to perform not much worse than the baseline policy during the learning process. Such conservative scenario has been studied in bandits \citep{wu2016conservative,kazerouni2017conservative,garcelon2020improved} and tabular MDPs \citep{garcelon2020conservative}.
\citet{garcelon2020conservative} investigate both the average reward setting and the finite horizon setting. 
\citet{yang2021reduction} propose a reduction-based framework for conservative bandits and RL, which translates a minimax lower bound of the non-conservative setting to a valid lower bound for the conservative case. It also proposes a Budget-Exporation algorithm and shows that its regret scales in $\tilde{O}\left(\sqrt{H^3SAK} + \frac{H^3SA\Delta_0}{\kappa(\kappa+\Delta_0)}\right)$ for tabular MDPs, where $\Delta_0$ is the suboptimality gap of the baseline policy, and $\kappa$ is the tolerable performance loss from the baseline. However, all these works assume {\it cumulative} conservative constraint. 
As discussed in \cref{sec:intro}, our episodic-wise constraint is more stringent, and correspondingly the algorithms and the regret analysis are also different from the prior works. 

\textbf{Constrained MDP with Baseline Policies.} Conservative exploration studied in this paper can be viewed as a specific case of the Constrained Markov Decision Process (CMDP) \citep{Altman:CMDP:1999,liu2021learning,efroni2020exploration,wei2022triple}, where the goal is to maximize the expected total reward subject to constraints on the expected total costs in each episode. Assuming a known safe baseline policy that satisfies the corresponding constraints, OptPess-LP \citep{liu2021learning} is shown to achieve an regret of $\tilde{O}(\frac{1}{\kappa}\sqrt{H^6S^3AK})$ without any constraint violation with high probability, while DOPE \citep{bura2022dope} improves the regret to $\tilde{O}(\frac{1}{\kappa}\sqrt{H^6S^2AK})$, where $\kappa$ denotes the Slater parameter. We note that both algorithms do not achieve the optimal regret in the unconstrained counterpart.

\section{Problem Formulation}

We consider an episodic MDP $\Mc=(\Sc, \Ac, H, P, r, s_1)$, where $\Sc$ and $\Ac$ are the sets of states and actions, respectively, $H \in \Zb_+$ is the length of each episode, $P = \{P_h\}^H_{h=1}$ and $r = \{r_h\}^H_{h=1}$ are respectively the state transition probability measures and the reward functions, and $s_1$ is a given initial state.  We assume that $\Sc$ and $\Ac$ are finite sets with cardinality $S$ and $A$ respectively. Moreover, for each $h \in [H]$, $P_h(\cdot|s,a)$ denotes the transition kernel over the next states if action $a$ is taken for state $s$ at step $h\in[H]$, and $r_h : \Sc \times \Ac \rightarrow [0,1]$ is the deterministic reward function at step $h$ which is assumed be known for simplicity. Our result can be easily generalized to random and unknown reward functions. {We consider the learning problem where $\Sc$ and $\Ac$ are known while $P$ are unknown a priori.}

A policy $\pi$ is a set of mappings $\{\pi_h : \Sc \rightarrow \Delta(\Ac)\}_{h\in[H]}$, where $\Delta(\Ac)$ is the set of all probability distributions over the action space $\Ac$. In particular, $\pi_h(a|s)$ denotes the probability of selecting action $a$ in state $s$ at time step $h$. 

An agent interacts with this episodic MDP as follows. In each episode, {the environment begins with a fixed} initial state $s_1$. Then, at each step $h \in [H]$, the agent observes the state $s_h \in \Sc$, picks an action $a_h\in \Ac$, and receives a reward $r_h (s_h, a_h)\in [0,1]$. The MDP then evolves to a new state $s_{h+1}$ that is drawn from the probability measure $P_h(\cdot|s_h,a_h)$. The episode terminates after $H$ steps.


For each $h \in [H]$, we define the state-value function $V_h^\pi: \Sc\rightarrow \Rb$ as the expected total reward received under policy $\pi$ when starting from an arbitrary state at the $h$-th step until the end of the episode. Specifically, $\forall s \in \Sc,h \in[H]$, 
\begin{align} \textstyle
V_h^\pi(s) &:= \Eb_{\pi}\bigg[\sum_{h'=h}^H r_{h'}(s_{h'},a_{h'})\bigg| s_h = s \bigg],
\end{align}
where we use $\Eb_{\pi} [\cdot]$ to denote the expectation over states and actions that are governed by $\pi$ and $P$. Since the MDP begins with the same initial state $s_1$, to simplify the notation, we use $V_1^\pi$ to denote $V^\pi_1(s_1)$ without causing ambiguity. Correspondingly, we define the action-value function $Q_h^{\pi}: \Sc\times\Ac\rightarrow \Rb$ at step $h$ as the expected total reward under policy $\pi$ after taking action $a$ at state $s$ in step $h$, that is:
\begin{align*}
Q_h^{\pi}(s,a) :=& \Eb_\pi\bigg[\sum_{h'=h}^H r_{h'}(s_{h'},a_{h'})\bigg| s_h=s,a_h = a \bigg]\\
=& r_h(s,a)+ [P_h V^{\pi}_{h+1}](s,a),
\end{align*}
where $[P_h V^{\pi}_{h+1}](s,a):=\Eb_{s'\sim P_h(\cdot|s,a)}[V^{\pi}_{h+1}(s')]$.
\if{0}
Therefore, we have the Bellman equation
\begin{align} 
&Q_h^{\pi}(x,a) := r_h(x,a) + \Eb_{x'\sim\Pb_h(\cdot|x,a)}[V_{h+1}^{\pi}(x')]\label{eqn: Bellman1}\\
&V_h^{\pi}(x) = Q_h^{\pi}(x,\pi_h(x))\label{eqn: Bellman2}
\end{align}
It is worth noting that $V_{H+1}(x) = 0$.
\fi
Since the action space and the episode length are both finite, there always exists an optimal policy $\pi^\star$ that gives the optimal value $V_h^\star(s) = \sup_\pi V_h^\pi(s)$ for all $s \in \Sc$ and $h\in [H]$.

\if{0}
\textbf{Linear MDP.} We assume the MDP $(\Sc, \Ac, H, P, r,x_1)$ is a linear MDP~\citep{Jin:2020:COLT} with a (known) feature map $\phi$, i.e., for any $h\in [H]$, there exist $d$ unknown measures $\mu_h = (\mu_h^{(1)},\ldots,\mu_h^{(d)})$ over $\Sc$ and an
unknown vector $\theta_h\in \Rb^d$, such that for any $(x, a)\in \Sc \times \Ac$, we have
$P_h(x'|x, a) = \langle\phi(x, a)$ and $\mu_h(x')\rangle, r_h(x, a) = \langle \phi(x, a), \theta_h\rangle$.
Without loss of generality, we assume $\|\phi(x, a)\| \leq  1$ for all $(x, a)\in \Sc \times \Ac$, and $\max\{\|\mu_h(\Sc)\|, \|\theta_h\|\}\leq \sqrt{d}$ for all $h \in[H]$.
\fi

\textbf{Conservative Constraint.} While there could be various forms of constraints imposed on the RL algorithms, in this work, we focus on a baseline policy-based constraint~\citep{garcelon:AISTATS:2020,yang2021reduction}. In many applications, it is common to have a known and reliable baseline policy that is potentially suboptimal but satisfactory to some degree. Therefore, for applications of RL algorithms, it is important that they are guaranteed to perform not much worse than the existing baseline throughout the learning process. Denote the baseline policy as $\pi^b$ and the corresponding expected total reward obtained under $\pi^b$ in an episode as $V_1^{\pi^b}$. 
Then, throughout the entire learning process, we require that the expected total reward for each episode $k$ is at least $\mathbf{\gamma}$ 
with high probability, where $\mathbf{\kappa}:= V_1^{\pi^b} - \gamma>0$ characterizes how much risk the algorithm can take during the learning process.
A policy $\pi$ that achieves expected total reward at least $\mathbf{\gamma}$ is considered to be ``safe'', and we emphasize that our proposed algorithms do not require the knowledge of $V_1^{\pi^b}$. Let $\pi^k$ be the policy adopted by the agent during episode $k \in[K]$. 
Mathematically, we formulate the conservative constraint as 
\begin{align}\label{eqn:constraint}
\Pb\left[V_1^{\pi^k}\geq \mathbf{\gamma}, \forall k\in[K] \right ]\geq 1-\delta, \text{ where } \delta\in(0,1).
\end{align}

\textbf{Comparison with Previous Conservative Constraints.} The conservative constraint in \Cref{eqn:constraint} is more restrictive compared with \citet{garcelon:AISTATS:2020,yang2021reduction}, where the constraint is imposed on the {cumulative} expected reward over all experienced episodes instead of on each episode. 
We note that this stringent constraint has a profound impact on the algorithm design. 
While the previous cumulative conservative constraint enables the idea of saving the conservative budget early on and spending it later to play explorative actions, it cannot guarantee that in each episode, the expected total reward is above a certain threshold. 
Our constraint in \Cref{eqn:constraint}, in contrast, requires the expected total reward to be above a threshold in each episode. Hence, the idea of saving budget from early episodes for exploration in future episodes cannot be adopted, and it requires a more sophisticated algorithm design to control the budget spending {\it within} each episode and ensure the safety of all executed policies. 

{In addition, the per-episode conservative constraint in our work is more practical than the cumulative reward-based constraints. This is because each episode in the episodic MDP setting corresponds to the learning agent interacting with the environment from the beginning to the end, e.g., a robot walks from a starting point to the end point. Guaranteeing the performance in every episode has physical meanings, e.g., making sure that the robot does not suffer any damage while learning how to walk. This cannot be captured by the long-term constraint that spans many episodes.}

\textbf{Learning Objective.} Under the given episodic MDP setting, the agent aims to learn the optimal policy by interacting with the environment during a set of episodes, subject to the conservative constraint. The difference between $V_1^{\pi^k}$ and $V_1^\star$ serves as the expected regret or the suboptimality of the agent in the $k$-th episode. Thus, after playing for $K$ episodes, the total expected regret is
\begin{align}
\label{eqn:defReg}
\mbox{Reg}(K):=KV_1^\star-\sum_{k=1}^K V_1^{\pi^k} .
\end{align}
Our objective is to minimize $\mbox{Reg}(K)$ while satisfying \Cref{eqn:constraint} for any given $\delta\in (0,1) $. 

\section{The StepMix Algorithm}
In this section, we aim to design a novel safe exploration algorithm to satisfy the episodic conservative constraint and achieve the optimal learning regret. 

\subsection{Challenges}
\label{sec:challenge}
For unconstrained episodic MDPs with finite states and actions, in order to achieve the minimax regret lower bound  $\Omega(\sqrt{H^3SAK})$, the core design principle \citep{azar2017minimax, zanette2019tighter,menard2021fast} is to construct a Bernstein inequality-based Upper Confidence Bound (UCB) for the action-state value function under the optimal policy (i.e., $Q^\star_h$), and then to execute an optimistic policy that maximizes the UCB in each step. Such a UCB takes the variance of the corresponding estimated value function into consideration, leading to a more efficient exploration policy.

Intuitively, in order to achieve the same learning regret, the safe exploration policy should follow a similar Bernstein-inequality based design principle. However, this may lead to several technical challenges, as elaborated below.

First, we note that in conventional CMDP problems under {\it episodic} cost constraints~\citep{bura2022dope,liu2021learning}, the exploration policy in each episode $k$ is usually obtained by solving a constrained optimization problem in the form of $\pi^k = \argmax_{\pi\in \Pi_k} {V}_1^\pi(P^k)$, where $P^k$ is the estimated model and $\Pi_k$ is the set of estimated safe policies. For given $P^k$, both the objective function and the constraint set can be expressed as a linear function of $\pi$ or of occupancy measures, and thus can be solved efficiently. However, if Bernstein inequality is adopted to construct a tighter confidence set of the value functions (and hence $\Pi_k$), it can no longer be formulated as a linear programming problem, resulting in unfavorable computational complexity in each iteration.


Second, in order to keep track of the estimation error under the adopted exploration policy $\pi^k$, it is necessary to bound $(\hP^k_h-P_h)V^{\pi^k}_{h+1}$ for each $h\in[H]$. In order to achieve the optimal dependence on $S$, a common technique is to decompose it into two terms, $(\hP^k_h-P_h)V^{\pi^\circ}_{h+1}$ and {$(\hP^k_h-P_h)(V^{\pi^k}_{h+1}-V^{\pi^\circ}_{h+1})$}, where $\pi^\circ$ is a fixed policy that is independent with the historical data, and then bound them separately. Intuitively, $\pi^\circ$ should be a policy ``close'' to $\pi^k$, so that as $\hP^k_h$ converges to $P_h$, both terms converge to zero and the overall learning regret can thus be bounded. In the unconstrained case, $\pi^\circ$ is naturally set to be the optimal policy {$\pi^\star$}. However, under the episodic conservative constraint in our setting, the selection of $\pi^\circ$ is more delicate. This is because $\pi^k$ may be very different from {$\pi^\star$}, especially at the beginning of the learning stage when little information of the underlying MDP is known. Therefore, how to construct a good ``anchor'' policy $\pi^\circ$ that stays close to $\pi^k$ {\it throughout the learning process} becomes challenging.


Finally, in order to ensure the safety of the exploration policy $\pi_k$ in each episode, it is necessary to obtain a {\it pessimistic} estimation of the corresponding value function and to make sure it is above the threshold $\gamma$. While the Lower Confidence Bound (LCB) under the optimal policy {$\pi^\star$} can be constructed in a symmetric manner as UCB, it is not immediately clear how to construct a Bernstein-type LCB for $V^{\pi_k}$, as $\pi_k$ {is not fixed but dependent on history, and it may deviate from the optimistic policy significantly due to the episodic conservative constraint.}


\if{0}
To solve these two problems, we borrow the idea of ``budget'' from \citet{yang2021reduction}. BudgetExploration or LCBCE \cite{yang2021reduction} can avoid the previous two problems at the cost of sacrificing episodic constraint to become cumulative constraint. With cumulative constraint, they can play a known baseline policy for ``budget'' and use the ``budget'' in the later episode on the existing non-conservative algorithms. We focus on the episodic constraint, which requires us to use the ``budget'' to explore within the episode, otherwise we will lose the exploration opportunities. For algorithm BudgetExploration or LCBCE \cite{yang2021reduction}, they can only decide to use the baseline policy or an optimistic greedy policy given by a non-conservative algorithm. To explore right in the episode, our algorithm must find a policy \textbf{between} the baseline policy and the optimistic greedy policy given by a non-conservative algorithm rather than one of them.

In summary, we want to:
\begin{enumerate}
    \item design the algorithm that can find a safe $\pi^k$ between baseline policy $\pi^b$ and optimistic policy $\hat{\pi}^\star$ given by a non-conservative reinforcement learning algorithm.
    \item give the UCB and LCB for $\pi^k$ armed with $|(\hP-p)V^\star|$-type inequalities in analysis.
    \item prove that the constraint will affect comparably finite episodes
\end{enumerate}
To deal with these problems, StepMix is proposed. We mainly follow BPI-UCBVI \cite{menard2021fast} algorithm as the non-conservative reinforcement learning algorithm. It is designed for best policy identification, but it also has the state-of-the-art order of regret and is easier for our analysis. 
\fi

\subsection{Algorithm Design}\label{sec:design}
In this subsection, we explicitly address the aforementioned challenges and present a novel algorithm termed as StepMix. 
Before we proceed to elaborate the design of StepMix, we first introduce the definition of step mixture policies. 

\begin{definition}[Step Mixture Policies]
\label{def:stepmix}
The step mixture policy of two Markov policies $\pi^1 $ and $\pi^2$ with parameter $\rho$, denoted by $\rho\pi^1 + (1-\rho)\pi^2$, is a Markov policy such that the probability of choosing an action $a_h$ given a state $s_h$ under the step mixture policy is 
$\rho\pi_h^1(a_h|{s_h})+(1-\rho)\pi^2_h(a_h|{s_h})$.
\end{definition}

\begin{algorithm}[t]
    \caption{The StepMix Algorithm}
    \label{alg:step}
\begin{algorithmic}
\STATE {\bf Input:} $\pi^b$, $\gamma$, $\beta$, $\beta^\star$, $\Dc_0=\emptyset$.
\FOR{$k$ = $1$ to $K$}
    \STATE Update model estimate $\hP$ according to \Cref{eqn:p}.
\STATE \textcolor{blue}{\it \# Optimistic policy identification}

$\tilde{V}_{H+1}^{k}(s)=\utilde{V}_{H+1}^{k}(s)=0,\forall s\in \Sc$.
    \FOR{$h$ = $H$ to $1$}
        \STATE Update $\tilde{Q}_{h}^{k}(s,a)$, $\utilde{Q}_{h}^{k}(s,a), \forall (s,a)\in \Sc\times \Ac$ according to \Cref{eqn:q}.
        \STATE $\gp_h^{k}(s)\gets\argmax_a \tilde{Q}_{h}^{k}(s,a)$, $\tilde{V}^{k}_h(s)\gets\tilde{Q}^{k}_h(s,\gp_h^k(s))$, $\utilde{V}^{k}_h(s)\gets \utilde{Q}^{k}_h(s,\gp_h^k(s)),\forall s\in\Sc$.
    \ENDFOR
       \STATE \textcolor{blue}{\it \# Candidate policy construction and evaluation}
    \FOR{$h_0$ = $0$ to $H$}
        \STATE $\pi^{k,h_0} = \{\pi_1^b,\pi_2^b,\cdots,\pi_{h_0}^b,\gp_{h_0+1}^{k},\cdots, \gp_{H-1}^{k}, \gp_{H}^{k}\}$.
       \STATE $\utilde{V}^{k,h_0}=\mbox{PolicyEva}(\hP^k,\pi^{k,h_0})$.
    \ENDFOR
         \STATE  \textcolor{blue}{\it  \# Safe exploration policy selection}
    \IF{$\{h\,|\,\utilde{V}_1^{k,h} \ge \gamma, h=0,1,\ldots,H\}=\emptyset$}
        \STATE $\pi^k = \pi^b$. 
    \ELSE
        \STATE $h^k = \min \{h\,|\,\utilde{V}_1^{k,h} \ge \gamma, h=0,1,\ldots,H\}$.
        \IF{$h^k=0$}
            \STATE $\pi^k = \gp^k$. 
            \ELSE
            \STATE Set $\pi^k$ according to \Cref{eqn:stepmix2}.
        \ENDIF
    \ENDIF
    \STATE Execute $\pi^k$ and collect $\{(s_h^k,a_h^k,s_{h+1}^k)\}_{h=1}^H$.
\STATE $\Dc_n\gets \Dc_{n-1}\cup\{(s_h^k,a_h^k,s_{h+1}^k)\}_{h=1}^H$.
\ENDFOR
\end{algorithmic}
\end{algorithm}

StepMix is a model-based algorithm that features a unique design of the candidate policies and safe exploration policies. In the following, we elaborate its major components.


\textbf{Model Estimation.} At each episode $k$, the agent uses the available dataset to obtain an estimate of the transition kernel. Specifically, let $n_h^{k}(s,a) = \sum_{\tau=1}^{k-1}\mathds{1}{\{s_h^{\tau} = s,~ a_h^{\tau} = a\}}$ and $n_h^{k}(s,a,s') = \sum_{\tau=1}^{k-1} \mathds{1}{\{s_h^{\tau} = s,~ a_h^{\tau} = a,~ s_{h+1}^{\tau} = s' \}}$ be the visitation counters. The agent estimates $\hP_h^{k}(s'|s,a)$ as 
\begin{align}\label{eqn:p}
    \hP_h^{k}(s'|s,a)&=\left\{\begin{array}{cc}
       \frac{n_h^{k}(s,a,s')}{n_h^{k}(s,a)},  &  \mbox{if } n_h^{k}(s,a)>1,\\
      \frac{1}{S},   & \mbox{otherwise}.
    \end{array}
    \right.
\end{align}

\textbf{Bernstein-type Optimistic Policy Identification.} With the updated model estimates $\hP^k$, the agent then tries to construct an optimistic policy. We note that this optimistic policy may not be identical to the exploration policy selected afterwards. However, it provides important information regarding the model estimate accuracy and will be leveraged to construct an efficient yet safe exploration policy. Specifically, we first denote
\[
\Var_{\hP^k_h}(\tilde{V}_{h+1}^k)(s,a)=\mathbb{E}_{s'\sim \hP^k_h(\cdot|s,a)}[(\tilde{V}_{h+1}^k(s')-\mathbb{E}_{s'\sim \hP^k_h(\cdot|s,a)}[\tilde{V}_{h+1}^k(s')])^2],\]
which captures the variance of $\tilde{V}_{h+1}^k$ under transition kernel $\hP^k_h$ given $(s^k_h,a^k_h)=(s,a)$. 

Then, 
with $\tilde{V}_{H+1}^{k}(s)=\utilde{V}_{H+1}^{k}(s)=0$, $\forall s\in \Sc$, for each $h\in[H], (s,a)\in \Sc\times\Ac$, we recursively define
\begin{align}
   & \scriptstyle  \tilde{Q}^{k}_h(s,a)\triangleq  \min\big(H, r_h(s,a)+3\sqrt{{\Var}_{\hP_h^k}(\tilde{V}_{h+1}^{k})(s,a)\frac{\beta^\star}{n^k_h(s,a)}} + 14 H^2\frac{\beta}{n^k_h(s,a)}+\frac{1}{H}\hP_h^k(\tilde{V}^{k}_{h+1}-\utilde{V}^{k}_{h+1})(s,a)+\hP_h^k\tilde{V}^{k}_{h+1}(s,a)\big)\nonumber\\
  & \scriptstyle \utilde{Q}^{k}_h(s,a)\triangleq  \max\big(0, r_h(s,a)-3\sqrt{\Var_{\hP_h^k}(\tilde{V}_{h+1}^{k})(s,a)\frac{\beta^\star}{n^k_h(s,a)}}  - 22 H^2\frac{\beta}{n^k_h(s,a)}-\frac{2}{H}\hP_h^k(\tilde{V}^{k}_{h+1}-\utilde{V}^{k}_{h+1})(s,a)+\hP_h^k\utilde{V}^{k}_{h+1}(s,a)\big),   \label{eqn:q}  
\end{align}
and obtain an optimistic policy $\gp^k$ by setting {$\gp_h^k(s) =\arg\max_{a\in\Ac}\tilde{Q}^{k}_h(s,a)$.}
After that, we set  $\tilde{V}^{k}_h(s)=\tilde{Q}^{k}_h(s,\gp_h^k(s))$, and $\utilde{V}^{k}_h(s)= \utilde{Q}^{k}_h(s,\gp_h^k(s))$.


Intuitively speaking, $\tilde{Q}^{k}_h(s,a)$ serves as a Bernstein-type UCB for the true value function under the optimal policy, i.e., {$Q^\star(s,a)$}, while $\utilde{Q}^{k}_h(s,a)$ serves as the corresponding LCB. We note that the designs of $\tilde{Q}^{k}_h(s,a)$ and $\utilde{Q}^{k}_h(s,a)$ are not symmetric, i.e., the coefficients associated with the Bernstein-type bonus terms are not exactly opposite. Actually, this unique selection of the bonus terms is critical for us to obtain a valid LCB not just for the optimal value function, but also for those value functions under the exploration policy $\pi^k$, as elaborated below.

\textbf{Candidate Policy Construction and Evaluation.} Once the agent obtains the optimistic policy $\gp^k$, it will proceed to construct a set of candidate policies denoted as $\{\pi^{k,h_0}\}_{h_0=0}^H$, where $\pi^{k,h_0} = \{\pi_1^b,\cdots\pi_{h_0}^b,\gp_{h_0+1}^{k},\cdots,\gp^{k}\}$. 
We note that $\pi^{k,h_0}$ follows the safe baseline $\pi^b$ for the first $h_0$ steps, after which it switches to the optimistic policy $\gp^k$. Besides, $\pi^{k,h_0}$ and $\pi^{k,h_0+1}$ only differ at step $h_0+1$. As $h_0$ sweeps from $H$ to $0$, $\pi^{k,h_0}$ essentially forms a smooth interpolation between the safe baseline $\pi^b$ and the optimistic policy $\gp^k$.

For each candidate policy $\pi^{k,h_0}$, we obtain UCB and LCB on the two corresponding true value functions {denoted as $Q^{k,h_0}$ and $V^{k,h_0}$} respectively, by invoking the {PolicyEva} subroutine (See \cref{alg:policyEva} in \cref{appx:proof_step}). Specifically, PolicyEva recursively updates $\tilde{Q}^{k,h_0}_h(s,a)$ and $\utilde{Q}^{k,h_0}_h(s,a)$ in the same form as in \Cref{eqn:q}, while $\tilde{V}^{k,h_0}_h(s)\triangleq\langle \pi_h^{k,h_0}(\cdot|s) ,\tilde{Q}^{k,h_0}_h(s,\cdot) \rangle $, $\utilde{V}^{k,h_0}_h(s)\triangleq\langle \pi_h^{k,h_0}(\cdot|s) ,\utilde{Q}^{k,h_0}_h(s,\cdot) \rangle $.

We design the set of candidate policies in order to explicitly address the second challenge in \Cref{sec:challenge}, i.e., it is desirable to obtain a fixed ``anchor'' policy that stays close to $\pi^k$ throughout the learning process. Intuitively, in order to satisfy the conservative constraint in each episode, $\pi^k$ would stay at $\pi^b$ when it has not collected enough information of the environment; As $k$ proceeds, it is desirable to have $\pi^k$ evolve to the optimal policy {$\pi^\star$}, in order to achieve the optimal learning regret. Thus, it may not be reasonable to expect that a single fixed anchor policy would stay close to $\pi^k$ in every episode. Instead, we construct {\it a set} of anchor policies denoted as $\{\pi^{\star,h_0}\}_{h_0=0}^H$, where $\pi^{\star,h_0} = \{\pi_1^b,\cdots\pi_{h_0}^b,\pi_{h_0+1}^{\star},\cdots,\pi_{H}^{\star}\}$.  
Essentially, $\pi^{k,h_0}$ is the optimistic version of  $\pi^{\star,h_0}$. Thus, the estimation error in $(\hP^k_h-P_h)V^{\pi^{k,h_0}}_{h+1}$ can be decomposed with respect to $V^{\pi^{\star,h_0}}_{h+1}$ and then be bounded separately. As $\pi^k$ dynamically evolves in between $\pi^b$ and $\gp$, we expect that it stays close to $\pi^{k,h_0}$ for certain $h_0$, and thus the estimation error in $(\hP^k_h-P_h)V^{\pi^{k}}_{h+1}$ can be effectively bounded as well.



\if{0}
With the $\pi^{\star,h_0}$, we extend the Bernstein-style inequality from $|(\hP-p)V^\star|$ to all $|(\hP-p)V^{\star,h_0}|$. With the $\pi^{k,h_0}$, we then shrink $\{\pi | \bunderline{V}^\pi>\gamma\}$ to $\Pi$:
\begin{align*}
    A &= \{\pi | \bunderline{V}^\pi>\gamma\} \\
    B &= \{\rho\pi^{k,h_0-1}+(1-\rho)\pi^{k,h_0}|\rho\in[0,1],h_0\in[H]\} \\
    \Pi &= A \cap B
\end{align*}
$A$ is the feasible set. We call $B$ the neighbor-wise step mixture policy set. The neighbor-wise step mixture policy set is our design of policies between $\pi^b$ and $\hat{\pi}^\star$.

To use $|(\hP-p)V^{\star,h_0}|$-style Berstein inequality, we set good events to bound $|(\hP-p)V^{\star,h_0}|$ for all $h_0$ with Bernstein inequalities. With them, upper bound $\tilde{Q}^{k,h_0}$ and lower bound $\utilde{Q}^{k,h_0}$ can be given to bound $Q^{\star,h_0}$ and $Q^{k,h_0}$ as \Cref{eqn:ucblcb}, where $\beta$ and $\beta^\star$ are logarithm terms. We will prove that $\utilde{Q}^{k,h_0}\le Q^{k,h_0}\le Q^{\star,h_0}\le \tilde{Q}^{k,h_0}$. Then, all the policies within $B$ will have reasonable LCB and UCB. By neighbor-wise step mixture, there is $Q^{\text{mix}}=\rho Q^{k,h_0-1}+(1-\rho) Q^{k,h_0}$ where $\pi^\text{mix}=\rho\pi^{k,h_0-1}+(1-\rho)\pi^{k,h_0}$, so the UCB and LCB can also enjoy the linear combination property.
\fi

\textbf{Safe Exploration Policy Selection.} After constructing and evaluating the set of candidate policies, we then design a safe exploration policy by mixing two neighboring candidate policies. 

Specifically, the learner will compare $\utilde{V}_1^{k,h_0}$ with the threshold {$\mathbf{\gamma}$} for $h_0=0,1,\ldots,H$. If it is above the threshold, it indicates that with high probability the candidate policy $\pi^{k,h_0}$ will satisfy the conservative constraint. 
Let $h^k$ be the smallest $h_0$ such that $\utilde{V}_1^{k,h_0} \ge \gamma$. Then, we have the following cases:
\begin{itemize}[leftmargin=*]
\vspace{-0.1in}
    \item If $h^k=0$, it indicates that the LCB of the optimistic policy $\gp$ is above the threshold. Thus the learner executes $\gp^k$. 
    \item If $h^k\in[1:H]$, it indicates that $\pi^{k,h^k}$ is safe but $\pi^{k,h^k-1}$ may be not. More importantly, they only differ in a single step $h^k$. Then, the learner would construct a mixture of $\pi^{k,h^k}$ and $\pi^{k,h^k-1}$ as follows: 
    \begin{align}
                \rho &= \frac{\utilde{V}_{1}^{k,h^k}({s_1})-\gamma}{\utilde{V}_1^{k,h^k}({s_1}) - \utilde{V}_1^{k,h^k-1}({s_1})} ,\label{eqn:stepmix1} \\
                \pi^k &= (1-\rho)\pi^{k,h^k}+\rho \pi^{k,h^k-1}. \label{eqn:stepmix2}
            \end{align}
    \item If none of $\utilde{V}^{k,h_0}$ is above the threshold, it indicates that the LCB of $V^{\pi^b}$ is below the threshold, which occurs when the estimation has high uncertainty. The learner will then resort to $\pi^b$ for conservative exploration. 
\end{itemize}

\if{0}
Therefore, at each episode $k$, StepMix finds a safe policy $\pi^k$ chosen from either $\pi^b, \bar{\pi}^n$, or a step mixture policy in \Cref{eqn:stepmix2} . Once the policy $\pi^n$ is executed and a trajectory is collected, the learner moves on to the next episode.
Therefore, at each episode $k$, StepMix tries to find the most explorative and exploitative policy from $\{\pi|\pi=\rho\pi^{k,h}+(1-\rho)\pi^{k,h-1}, h\in[H],\rho\in[0,1]\}$ that satisfies the constraint $V_1^{\pi}\ge \gamma$. It should be noted that $\pi^{k,H}$ is always the baseline and $\pi^{k,0}$ is always the overall constraint-free greedy policy. 
\fi
Once policy $\pi^k$ is executed and a trajectory is collected, the learner moves on to the next episode.

\subsection{Theoretical Analysis}\label{sec:theoretical_analysis}
The performance of StepMix is stated in the following theorem. 
\if{0}
\begin{theorem}\label{thm:step}
There exist absolute constants $c'$, $c_{\beta}$, $c_1$ and $c_2$ such that, for any $\delta\in(0,1)$, if we choose $\lambda=c'd\log(dNH/\delta)$ and $\beta = c_{\beta}dH\sqrt{\iota}$ in \Cref{alg:step} with $\iota = 2\log(4dHN/\delta)$, then with probability at least $1-\delta$,  StepMix-LSVI (\Cref{alg:step}) simultaneously (i) satisfies the conservative constraint in \Cref{eqn:constraint}, and (ii) achieves a total regret that is at most 
\begin{equation}
\label{eqn:stepreg}
c_1\sqrt{d^3H^4 N\iota^2} + \frac{c_2d^3H^4\Delta_0\iota^2}{\mathbf{\kappa}^2}, 
\end{equation}
where {\color{red}$\Delta_0 := V^\star - V^{\pi^b}$} is the suboptimality gap of the baseline policy and $\mathbf{\kappa} := V^{\pi^b} - \mathbf{\gamma}$ is the tolerable value loss from the baseline policy.
\end{theorem}
\fi

\begin{theorem}[Informal]
    \label{thm:step}
    {With probability {at least} $1-\delta$}, StepMix (\Cref{alg:step}) simultaneously (i) satisfies the conservative constraint in \Cref{eqn:constraint}, and (ii) achieves a total regret that is at most 
        \begin{align*} \textstyle   \tilde{O}\big(\sqrt{H^3SAK}+H^3S^2A+H^3SA\Delta_0\big(\frac{1}{\kappa^2}+\frac{S}{\kappa}\big)\big),
        \end{align*}
    where {$\Delta_0 := V_1^\star - V_1^{\pi^b}$} is the suboptimality gap of the baseline policy and $\mathbf{\kappa} := V_1^{\pi^b} - \mathbf{\gamma}$ is the tolerable value loss from the baseline policy.
\end{theorem}

\begin{remark}
\Cref{thm:step} indicates that StepMix achieves a near-optimal regret in the order of $\tilde{O}(\sqrt{H^3SAK})$, while ensuring zero constraint violation with high probability. Compared with {BPI-UCBVI \citep{menard2021fast}}, the conservative exploration only leads to an additive constant term $\tilde{O}(H^3SA\Delta_0(\frac{1}{\kappa^2}+\frac{S}{\kappa}))$ in the learning regret bound. The additive term matches with that in the lower bound under the weaker cumulative conservative constraint in \citet{yang2021reduction} up to a constant, indicating our result is near-optimal. For the special case when $\gamma= 0$, the LCBs estimated in StepMix will always be greater than $\gamma$; thus the optimistic policy is always safe. Therefore, the algorithm reduces to an optimistic algorithm and the additive term becomes zero. Further discussion on this can be found in \Cref{re:more}. 
\end{remark}


The proof of \Cref{thm:step} is provided in Appendix~\ref{appx:proof_step}. We outline the major steps of the proof as follows.  

\if{0}
First, \jing{we note that the mixture policy $\pi^k$ is stochastic in general, as opposed to the deterministic greedy policy under LSVI-UCB. To cope with the policy randomness and temporal dependency, we develop a new uniform concentration lemma for value functions under policy $\tilde{\pi}^{n,h_0}$ for any $h_0\in[0:H+1]$, as elaborated in \Cref{lemma:uniform_step}. Thus, the uniform concentration can be established for any mixture of $\tilde{\pi}^{n,h_0}$ and $\tilde{\pi}^{n,h_0-1}$.} Such uniform concentration ensures that with high probability, the true value functions are bounded by the constructed UCB and LCB in Algorithm~\ref{alg:step} (see \Cref{lemma:LCB_UCB_Guarantee}). Thus, when the LCB of a policy is above the threshold $\gamma$, it ensures its safety with high probability. Moreover, the gap between them is controlled by the total expected bonus within an episode, i.e. $\sum_{h=1}^H\Eb_{\pi^n}\left[\|\phi(x_h,a_h)\|_{(\Lambda_h^n)^{-1}}\right]$, where $\pi^n$ is the actual policy executed in episode $n$, chosen from $\pi^b$, $\bar{\pi}^n$, and $\tilde{\pi}^{n,h_n}$ \jing{(see \Cref{lemma:Difference_LCB_UCB_True_Value}).  
The next step is thus to bound the total expected bonus under $\pi^n$.} 
However, the various forms policies $\pi^n$ may choose from make our analysis significantly harder than the original analysis in \citet{Jin:2020:COLT}. We highlight several major challenges in the following. 
\fi

As discussed in \Cref{sec:design}, one pivotal component in StepMix is the construction of the candidate policies. As a result, in our proof, we first extend the good event related to $(\hP_h^k-P_h)V^\star_{h+1}$ to $H+1$ good events related to $(\hP_h^k-P_h)V^{\star,h_0}_{h+1}$, $h_0=0,1,\ldots, H$. Since $\pi^{*,h_0}$ is a fixed anchor policy, $(\hP_h^k-P_h)V^{\star,h_0}_{h+1}$ can be bounded for all $h_0$. 

Next, we show that $\tilde{Q}^{k,h_0}$ and $\utilde{Q}^{k,h_0}$ are valid UCB and LCB of $Q^{k,h_0}$ and $Q^{\star,h_0}$ respectively, in the sense that $\utilde{Q}^{k,h_0}\le Q^{k,h_0}\le Q^{\star,h_0}\le \tilde{Q}^{k,h_0}$. Meanwhile, we show that $\tilde{Q}^{k,h_0}$ and $\utilde{Q}^{k,h_0}$ are sufficiently tight, as $\tilde{Q}^{k,h_0}-\utilde{Q}^{k,h_0}$ is bounded and will converge to zero sufficiently fast. Furthermore, with the properties of our constructed step mixture policy, we can obtain tight UCB and LCB for the step mixture policies as well.

Finally, we show that $\pi_k$ only stays at $\pi^b$ or the mixture policy for finite number of episodes. This is due to 
the fact that $V_1^{\star,h_0}\geq V_1^{\pi_b}$ for any $h_0=0,1,\ldots,H$. Thus, with high probability, $\utilde{V}_1^{k,h_0}\geq\gamma$ when $k$ is sufficiently large. As a result, the agent will then select the optimistic policy $\gp^{k}$ in most of the episodes. Thus, the regret of StepMix has the same leading term as that under the optimistic policy, which will then be bounded efficiently.

\if{0}
The target of our algorithm StepMix is to find a good policy to explore and exploit under the performance constraint. Conventionally, UCB is used to find the policy and LCB is used to check the constraint. The challenge is that bounding an arbitrary stochastic policy is hard. To cope with the challenge, we restrict the candidate policies to the step-wise greedy policies of the baseline policy and the greedy policy at different steps, namely $\pi^{k,h}$, and their neighbor-wise mixtures, namely $(1-\rho)\pi^{k,h}+\rho\pi^{k,h-1},h\in[H]$.

For the step-wise greedy policies $\pi^{k,h}$, behaviorally, they will converge to the step-wise optimal policies $\pi^{\star,h}$ when the number of samples increases. Moreover, for any $h$, $\pi^{k,h}$ converges to $\pi^{\star,h}$ in the same way that a greedy policy under UCB converges to the optimal policy in BPI-UCBVI. With our design of $\tilde{Q}^{k,h}$ and $\utilde{Q}^{k,h}$, the convergence can be showed in the following lemma.

\begin{lemma}
    \label{lem:ringgap}
    \begin{align*}
        \tilde{Q}_h^{k,h_0}(s,a)-\utilde{Q}_h^{k,h_0}(s,a) &\le G_h^{k,h_0}(s, a) \\
        \tilde{V}_h^{k,h_0}(s)-\utilde{V}_h^{k,h_0}(s,a) &\le \pi_h^{k,h_0}G_h^{k,h_0}(s)
    \end{align*}
    where
    \begin{align*}
        &G_h^{k,h_0}(s,a)\\
        =&\min\Biggr(H, 6\sqrt{\Var_{\hP_h^k}(\tilde{V}^{k,h_0}_{h+1})(s,a)\frac{\beta^\star(n^k_h(s,a),\delta')}{n^k_h(s,a)}}\\
        +&36H^2\frac{\beta^\star(n^k_h(s,a),\delta')}{n^k_h(s,a)}+(1+\frac{3}{H})\hP_h^k \pi_{h+1}^{k,h_0}G_{h+1}^{k,h_0}(s)\Biggr)
    \end{align*}
\end{lemma}

We should note it again, that $\pi^{k,0}$ is the overall greedy policy and $\pi^{k,H}$ is the baseline policy. Intuitively, $\pi^{k,h}$ is more conservative if $h$ is bigger. Since the baseline policy is assumed to be safe, by adjusting the index $h$, there is always a step-wise greedy policy conservative enough to be safe. In StepMix, the smallest valid $h$ is selected to enable better exploration and exploitation.

However, only the set of $\pi^{k,h}$ is not explorative enough under constraint and the neighbor-wise mixtures $(1-\rho)\pi^{k,h}+\rho\pi^{k,h-1}$ are necessary. For example, imagine a one-step MDP or bandit, that has only two actions. The first action renders reward $1$ and the second action renders $0.5$. If the baseline is deterministic to take the second action and the constraint $\gamma$ is $0.1$, then the algorithm will stay on the baseline forever, since it has no information about the first action and be afraid to make a bad decision. By introducing the mixture policies, the problem can be solved that StepMix can always choose a small $\rho$ to explore and meet the constraint requirement in expectation. In addition, neighbor-wise mixtures has good properties such as \ref{lem:stepcomb} which make the analysis of UCB and LCB much easier. With the neighbor-wise mixture policies, we can prove that StepMix will stay on mixture policies for comparably finite episodes(\ref{lem:finite}) and only cause the constant additional term in the final regret.
\fi

\if{0}
First, when $\pi^n$ is a step mixture policy, we do not have a direct estimation on the corresponding value function. Thanks to the {\it one-step-difference} of $\tilde{\pi}^{n,h_n}$ and $\tilde{\pi}^{n,h_n-1}$, the true value function under the step mixture policy $\pi^n$ is a {\it linear combination} of $V^{\tilde{\pi}^{n,h_n}}$ and $V^{\tilde{\pi}^{n,h_n-1}}$ (see \Cref{lemma:OneStepDiffer}). The linearity also holds for the LCB, \jing{which ensures the {\bf safety} of the $\pi^n$ for each episode $n$. Besides,} the difference between the LCB and $V^{\pi^n}$ is controlled by $ \sum_{h=1}^H\Eb_{\pi^n}\left[\|\phi(x_h,a_h)\|_{(\Lambda_h^n)^{-1}}\right] $, {similar to that between the UCB and true value function under the optimistic policy $\bar{\pi}^n$} (see \Cref{corollary: LCB_pi^n}).

Another challenge is due to the randomness of the step mixture policy. Since the actions taken under the same step mixture policy may be different, the information matrix $\Lambda_h^n$ may have different realizations. To cope with such randomness, we relate the information matrix with its expectation, i.e., $\bar{\Lambda}_h^n:=\lambda I + \sum_{\tau=1}^{n-1}\Eb_{\pi^{\tau}}\left[\phi(x_h,a_h)\phi(x_h,a_h)^{\top}\right]$, and show that the elliptical potential $\|\phi(x_h,a_h)\|_{({\Lambda}_h^n)^{-1}}$ is upper bounded by $\|\phi(x_h,a_h)\|_{(\bar{\Lambda}_h^n)^{-1}}$ up to a constant factor (see \Cref{lemma:random_info_matrix}).

\jing{In order to bound the {\bf regret}, we first relate the step mixture policy with the baseline policy.} Intuitively, since $\bar{\pi}^n$ is an optimistic policy, switching from $\pi^b$ to $\bar{\pi}^n$ after step $h_n$ increases the value function in general. Thus, $V^{\pi^n}$, after padding the bonus terms, should be larger than $V^{\pi^b}$. Combining with the gap between LCB and its true value, we have the following corollary.


We also note that the LCB of $V^{\pi^{n}}$ when $\pi^n$ is a step mixture policy is exactly equal to the threshold $\mathbf{\gamma}$. Therefore,
\begin{align}\label{eqn:bonus_control}
\textstyle
V^{\pi^b} - \mathbf{\gamma} \leq 4\beta\sum_{h'= 1}^H\Eb_{\pi^{n}}\left[ \left\|\phi(x_{h'},a_{h'})\right\|_{(\Lambda_h^n)^{-1}}\right].
\end{align}
Let $\Nc$ be the subset of episodes when a step mixture policy is adopted. Summing \Cref{eqn:bonus_control} over $\Nc$, we have the right hand side bounded by $\tilde{O}(\sqrt{|\Nc|})$ due to a revised elliptical potential lemma (see \Cref{corollary: revised elliptical potential lemma}). However, the left hand side equals $\mathbf{\kappa}|\Nc|$, which implies $\Nc$ is a finite set. Similar arguments can be applied to $\pi^b$. Therefore, StepMix-LSVI always plays $\bar{\pi}^n$ except for a finite number of episodes, which results in adding only a constant term to the original regret $\tilde{O}(\sqrt{d^3H^4N})$.



\fi

\section{The EpsMix Algorithm}
In this section, we briefly introduce another algorithm named EpsMix and defer the detailed design and analysis to \cref{appx:eps_proof}. Different from StepMix in Algorithm~\ref{alg:step}, EpsMix does not construct step mixture policies during the learning process. Rather, it adopts a randomization mechanism at the beginning of each episode, and designs episodic mixture policies \citep{wiering2008ensemble,baram2021maximum} defined as follows.

\begin{definition}[Episodic Mixture Policy]
\label{def:epsmix}
Given two policies $\pi^1$ and $\pi^2$ with parameter $\rho\in (0,1)$, the episodic mixture policy, denoted by $\rho\pi^1 \oplus (1-\rho)\pi^2$, randomly picks $\pi^1$ with probability $\rho$ and $\pi^2$ with probability $1-\rho$ at the beginning of an episode and plays it for the entire episode.
\end{definition}

The EpsMix algorithm is presented in \Cref{alg:eps} in ~\cref{appx:eps_proof}, and it proceeds as follows. Similar to StepMix, at the beginning of each episode $k$, it first constructs an optimistic policy, denoted as $\gp^{k}$. It then evaluates the LCB of the expected total rewards under both $\gp^{k}$ and ${\pi}^b$, denoted as $\utilde{V}_1^{k}$ and $\utilde{V}_1^{k,b}$ respectively.  
If $\utilde{V}_1^{k}$ is above the threshold $\mathbf{\gamma}$, it indicates that the optimistic policy $\gp^{k}$ satisfies the conservative constraint with high probability. The learner thus executes $\gp^{k}$ in the following episode $k$. Otherwise, if $\utilde{V}_1^{k,b}$ is above the threshold while $\utilde{V}_1^{k}$ is not, it constructs an episodic mixture policy $\rho_k \gp^{k} \oplus (1-\rho_k)\pi^b$ so that $\rho_k \utilde{V}_1^{k}+(1-\rho_k) \utilde{V}_1^{k,b}=\mathbf{\gamma}$. It implies that the episodic policy satisfies the conservative constraint in expectation with high probability. If neither $\utilde{V}_1^{k}$ nor $\utilde{V}^{k,b}$ is above the threshold, EpsMix will resort to the baseline policy to collect more information.

Our theoretical analysis shows that EpsMix has the same performance guarantees as StepMix.
At the same time, we note that EpsMix is less conservative than StepMix in the sense that, the expected return under a {\it selected} policy in an episode may be below the threshold when $\utilde{V}_1^{k}< {\mathbf{\gamma}}$. However, when taking the randomness in the policy mixture procedure into consideration, we can still guarantee that the expected total return under an episodic mixture policy is above the threshold with probability at least $1-\delta$.

\section{From Baseline Policy to Offline Dataset}

Both EpsMix and StepMix critically depend on the baseline policy $\pi^b$ to achieve the desired conservative guarantee. In reality, however, a baseline policy that provably satisfies the conservative constraint may not always be explicitly given to the algorithm. Instead, the learning agent may have access to an offline dataset that is collected from the target environment by executing an unknown behavior policy $\mu$, and the goal is to design a conservative exploration algorithm that satisfies \Cref{eqn:constraint} only using the offline dataset. 

A natural approach to solve this problem is to first learn a baseline policy from the dataset, and then use it as an input to EpsMix or StepMix. The challenge, however, is that instead of having full confidence in the conservative guarantee of $\pi^b$, we must deal with the \emph{safety uncertainty} of the learned baseline policy, that is introduced by using the offline dataset as well as the offline learning algorithm that produces the baseline policy. Fortunately, we prove that for StepMix, the uncertainty of learning a safe baseline policy from the offline dateset does not affect the conservative constraint violation or the regret order if the offline dataset is sufficiently large. 


\if{0}
\begin{theorem}
\label{thm:offStep}
Let $\pi^{\text{off}}$ be the output of the PEVI algorithm \citep{jin2021pessimism} (see \Cref{alg:pevi} in \Cref{appx:offline}) with $N_1 = \tilde{\Theta}(\frac{d^3H^4}{\kappa^2})\footnote{We hide the logarithmic factor for simplicity.}$ offline trajectories and parameters chosen properly. If we replace the baseline policy $\pi^b$ used in \Cref{alg:step} by $\pi^{\text{off}}$, then there exist two constants $c_1,c_2$ such that with probability at least $1-2\delta$, we can simultaneously (i) satisfy the conservative constraint in \Cref{eqn:constraint}, and (ii) achieve a total regret that is at most 
\[c_1\sqrt{d^3H^4 N\iota^2} + \frac{4c_2d^3H^4(\Delta_{0}+\kappa/2)\iota^2}{\kappa^2}.\]
\end{theorem}
\fi

\if{0}
\begin{theorem}[StepMix with an offline dataset]
\label{thm:offStep}
Assume that a behavior policy $\pi^b$ satisfying $V^{\pi^b} \geq \gamma$ produces a dataset with $N_1$ trajectories, and an offline policy $\pi^\text{off}$ is trained with these $N_1$ trajectories using the PEVI algorithm \citep{jin2021pessimism} (see \Cref{alg:pevi} in \Cref{appx:alg}). Denote $\kappa = V^{\pi^b} - \gamma$ and $\Delta_0=V^\star-V^{\pi^b}$. With probability at least $1-\delta$, invoking the StepMix algorithm with $\pi^\text{off}$ as the baseline policy can simultaneously (i) satisfy the conservative constraint in \Cref{eqn:constraint}, and (ii) achieve a total regret that is at most $c_1\sqrt{d^3H^4 N\iota^2} + \frac{c_2d^3H^4\Delta_{0,\text{off}}\iota^2}{\mathbf{\kappa}_\text{off}^2}$,
if $N_1 \ge 24c_\beta^2d^3H^4\log^2(192c_{\beta}^2d^4H^5/\kappa^2\delta)/\kappa^2$, where $\beta=c_{\beta}dH\sqrt{\iota}$, $\iota=2\log(10dHN/\delta)$, $\kappa_\text{off}=\kappa-2\sqrt{3} \beta_1 H \sqrt{\frac{d}{N_1}}$, $\Delta_{0,\text{off}}=\Delta_0+2\sqrt{3} \beta_1 H \sqrt{\frac{d}{N_1}}$, $\beta_1 = c_{\beta}dH\sqrt{\iota_1}$, $\iota_1 = 2\log(20dHN_1/3\delta)$, $\lambda=c'd\log(5dNH/2\delta)$, and $\lambda_1=c'd\log(5dN_1H/3\delta)$.
\end{theorem}
\fi

\begin{theorem}
    \label{thm:offstep}
    Let $\hat{\pi}$ be the output of the offline VI-LCB algorithm \citep{xie2021policy} (see \Cref{alg:offline} in \Cref{appx:offline}) with $n = \tilde{\Theta}(\frac{H^5SA}{\bar{\kappa}^2})\footnote{We hide the logarithm factor for simplicity.}$ offline trajectories. If we replace the baseline policy $\pi^b$ used in \Cref{alg:step} by $\hat{\pi}$, then with probability at least $1-\delta$, StepMix can simultaneously (i) satisfy the conservative constraint in \Cref{eqn:constraint}, and (ii) achieve a total regret that is at most
\begin{align*}\textstyle
\tilde{O}\left(\sqrt{H^3SAK}+H^3S^2A+H^3SA\bar{\Delta}_0\left(\frac{1}{\bar{\kappa}^2}+\frac{S}{\bar{\kappa}}\right)\right),
\end{align*}
    where $\bar{\kappa}=(V_1^\mu-\gamma)/2 > 0$ and $\bar{\Delta}_0=V^\star_1-V^\mu_1+\bar{\kappa}$.
\end{theorem}

A similar result for EpsMix can be established, and is given as \Cref{thm:offeps} in \Cref{appx:offline}. We see that $n$ scales inversely proportional to $\kappa^2$, suggesting that a good behavior policy would require small amount of data and vice versa. Besides, the additive term in the regret becomes larger compared with that in \Cref{thm:step}. In general, a large $n$ serves two purposes: First, it reduces the safety uncertainty due to offline learning, such that the impact on the safety constraint violation is negligible compared with that caused by the (online) StepMix policy. Second, it ensures that the regret bound is dominated by the number of online episodes $K$. We also note that although both \Cref{thm:offstep} and \Cref{thm:offeps} depend on using VI-LCB as the offline learning algorithm, the conclusion can be extended to general offline algorithms as long as they can produce an approximately safe policy from the pre-collected data with high probability.

\section{Experimental Results}


\subsection{Performance Evaluation of StepMix and EpsMix} \label{sec:exp1}

\textbf{Synthetic Environment.} 
We generate a synthetic environment to evaluate the proposed algorithms. We set the number of states $S$ to be 5, the number of actions $A$ for each state to be 5, and the episode length $H$ to be 3. The reward $r_h(s,a)$ for each state-action pair and each step is generated independently and uniformly at random from $[0, 1]$. We also generate the transition kernel $P_h(\cdot|s,a)$ from an $S$-dimensional simplex independently and uniformly at random. Such procedure guarantees that the synthetic environment is a proper tabular MDP.

\textbf{Baseline Policy.} We adopt the Boltzmann policy \citep{Thrun:1992} as the baseline policy in our algorithms. Under the Boltzmann policy, actions are taken randomly according to 
$\pi_h(a|s)=\frac{\exp\{\eta Q^\star_h(s, a)\}}{\sum_{a\in\Ac} \exp\{\eta Q_h^\star(s, a)\}},$
where a larger $\eta$ leads to a more deterministic policy and higher expected value.

\textbf{Results.}  We first evaluate the proposed StepMix and EpsMix, and compare with {BPI-UCBVI \citep{menard2021fast}}. For each algorithm, we run {10} trials and plot the average expected return per episode. 

\begin{figure*}[hpbt]
    \centering
    \subfigure[$\eta=5,\gamma=2.0$]{ \includegraphics[width=0.23\textwidth]{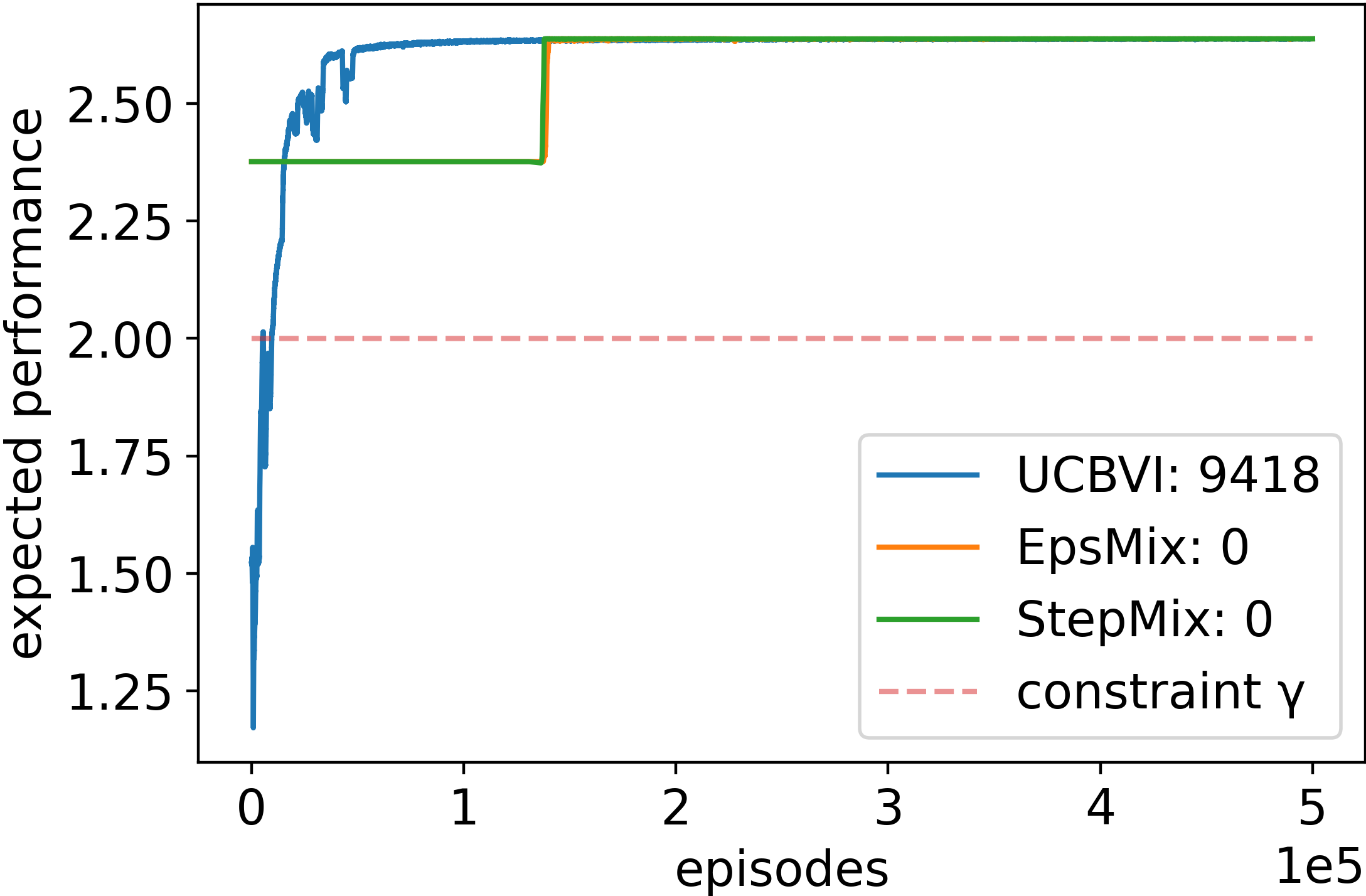}\label{fig:1}}
    \subfigure[$\eta=5,\gamma=2.2$]{ \includegraphics[width=0.23\textwidth]{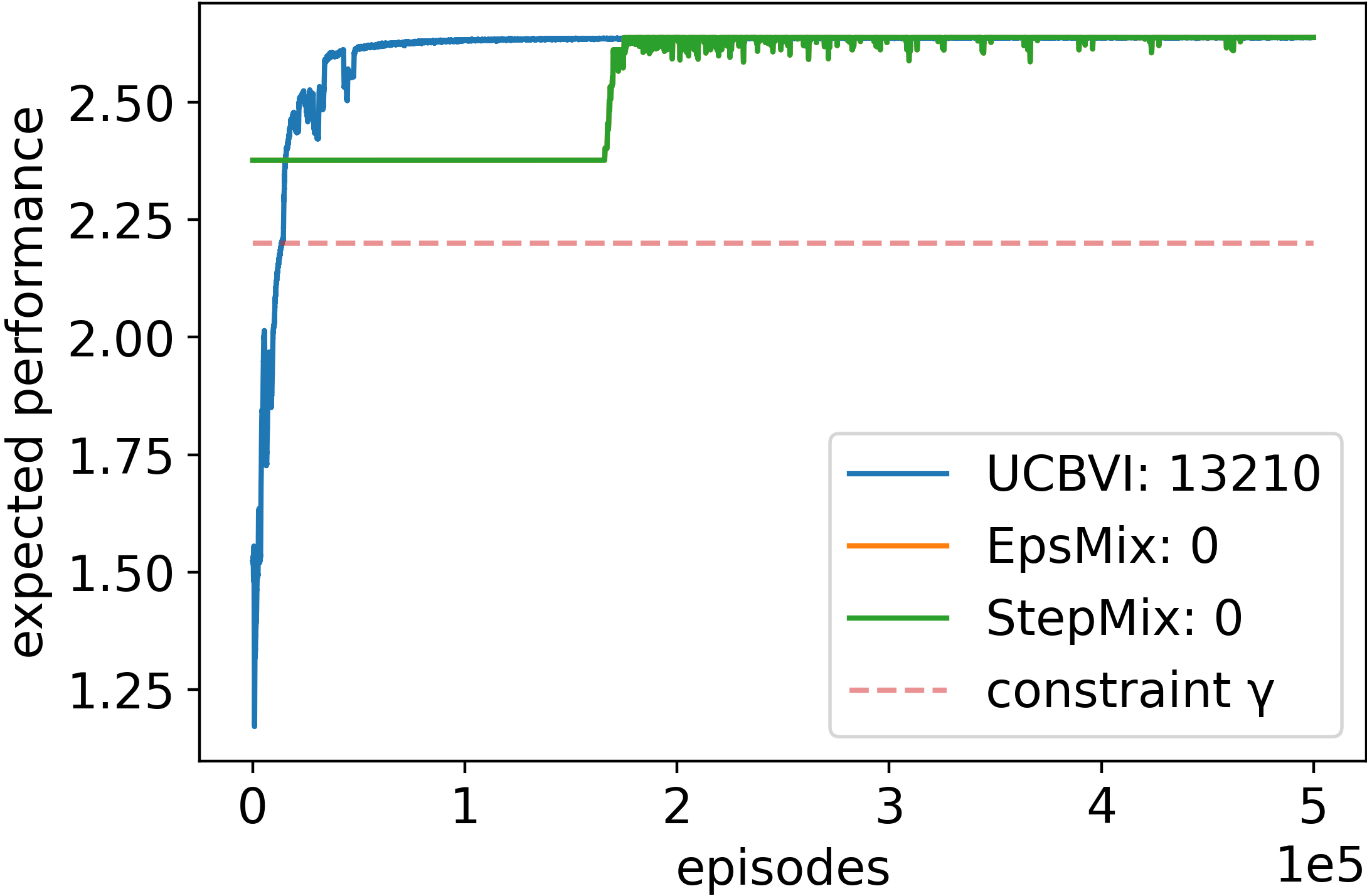}\label{fig:2}}
    \subfigure[$\eta=10,\gamma=2.0$]{ \includegraphics[width=0.23\textwidth]{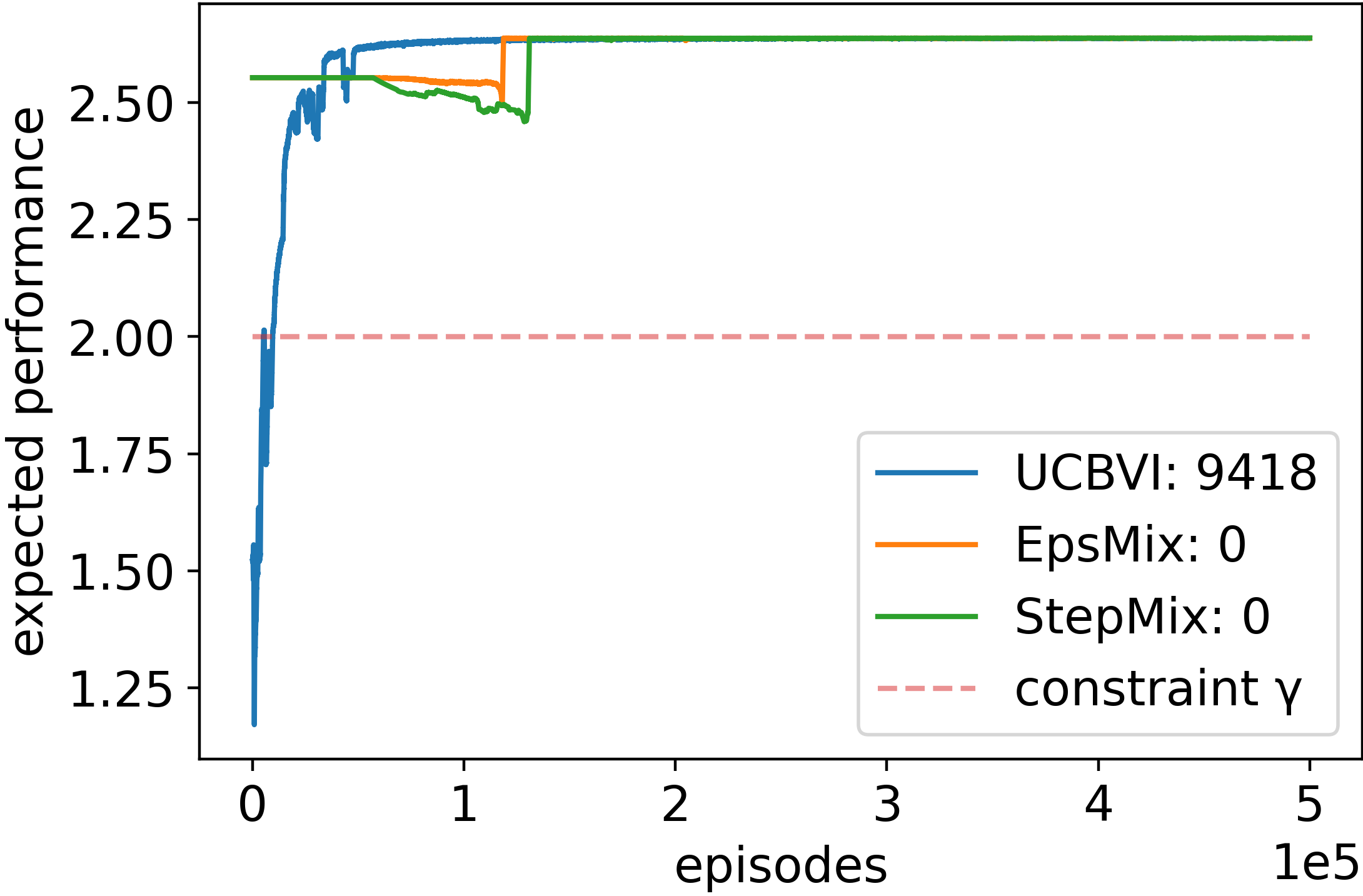}\label{fig:3}}
    \subfigure[$\eta=10,\gamma=2.2$]{ \includegraphics[width=0.23\textwidth]{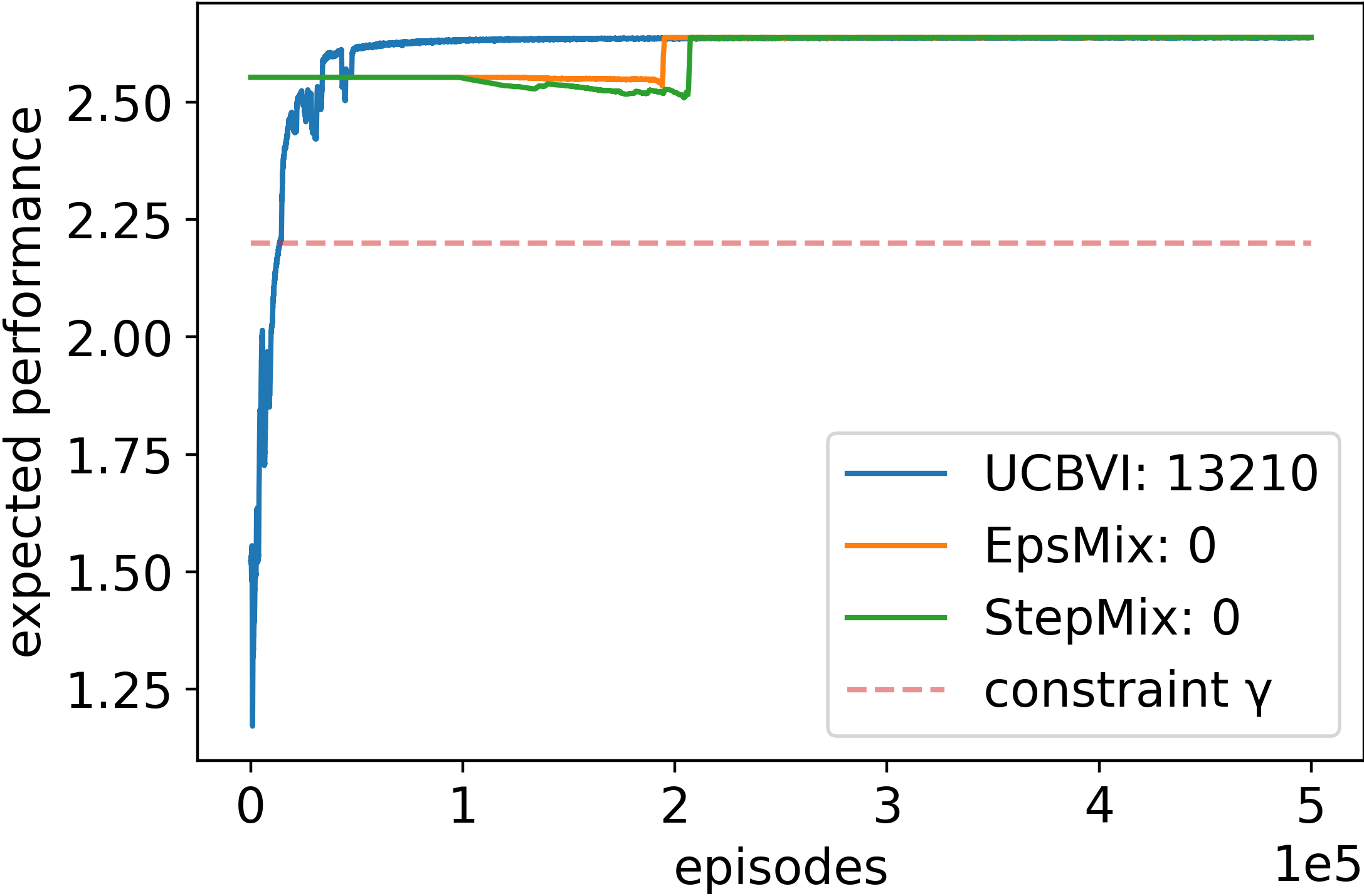}\label{fig:4}}
    \vspace{-0.1in}
    \caption{\small Average expected return of each episode under StepMix, EpsMix, and BPI-UCBVI with different constraint $\gamma$ and baseline parameter $\eta$. Numbers of violations are stated in the legend.}
    \vspace{-0.1in}
    \label{fig:performance}
\end{figure*}

\begin{figure*}[hpbt]
    \centering
    \subfigure[$\eta=10,n=5000$]{ \includegraphics[width=0.23\textwidth]{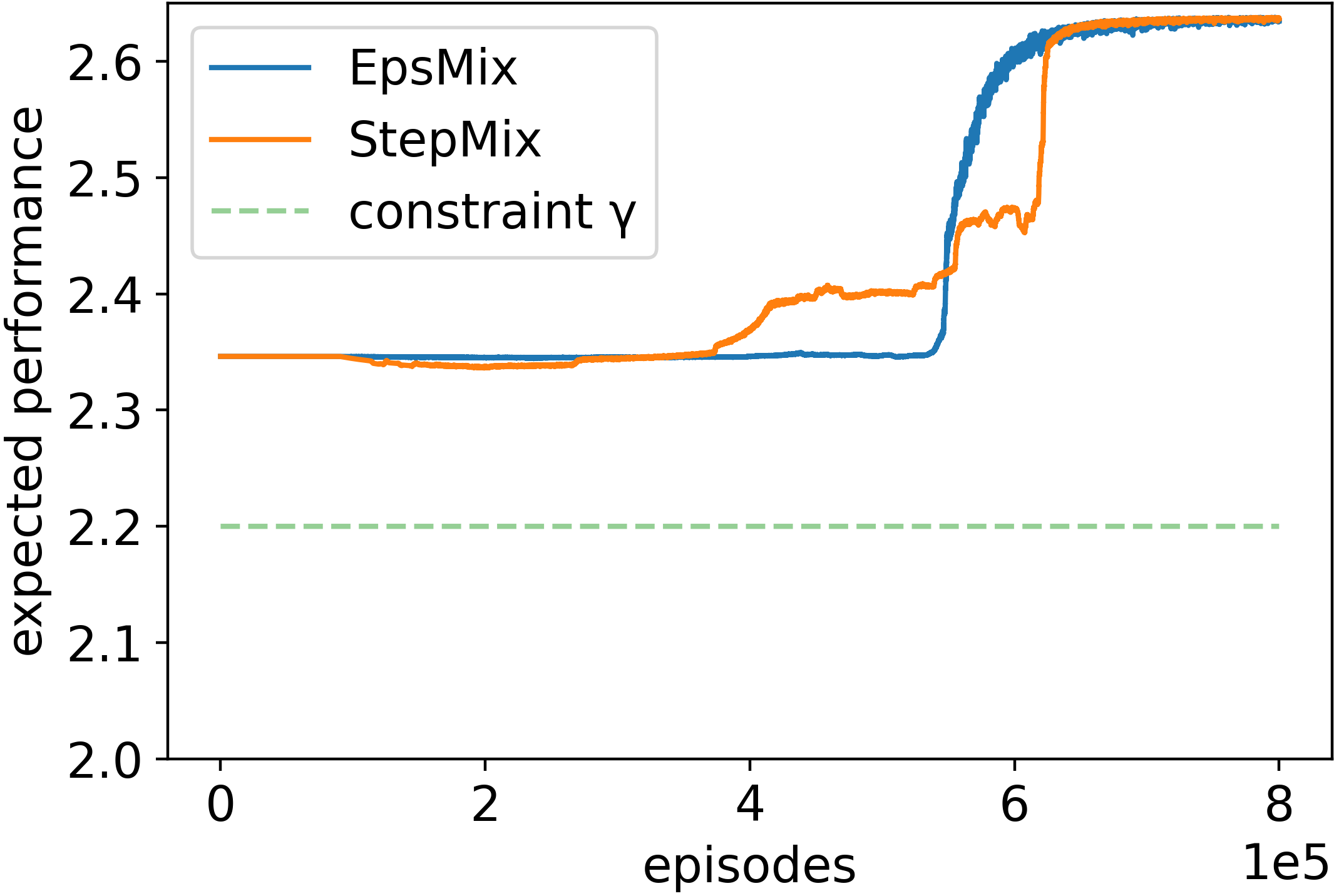}\label{fig:5}}
    \subfigure[$\eta=10,n=8000$]{ \includegraphics[width=0.23\textwidth]{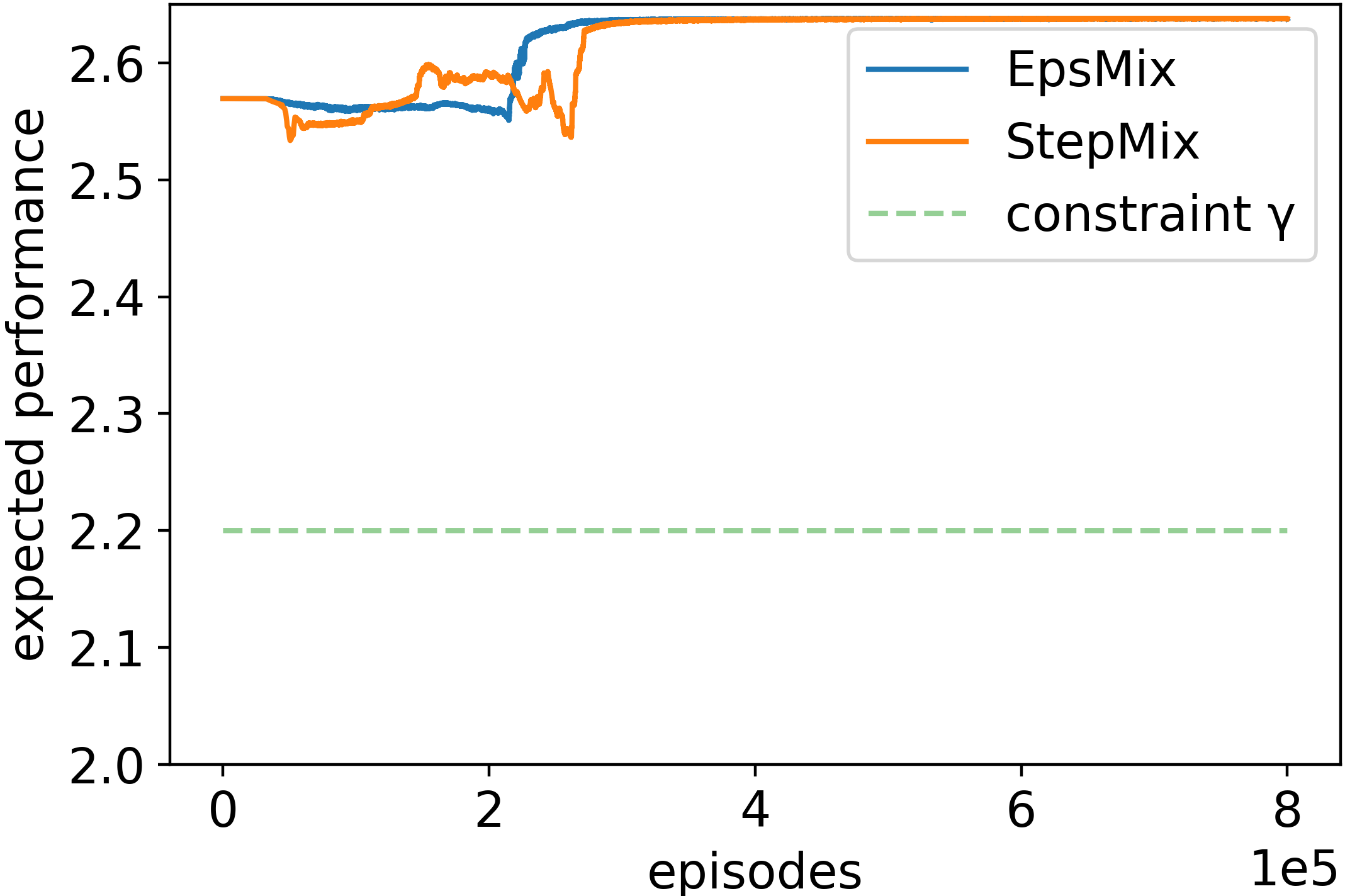}\label{fig:6}}
    \subfigure[$\eta=15,n=5000$]{ \includegraphics[width=0.23\textwidth]{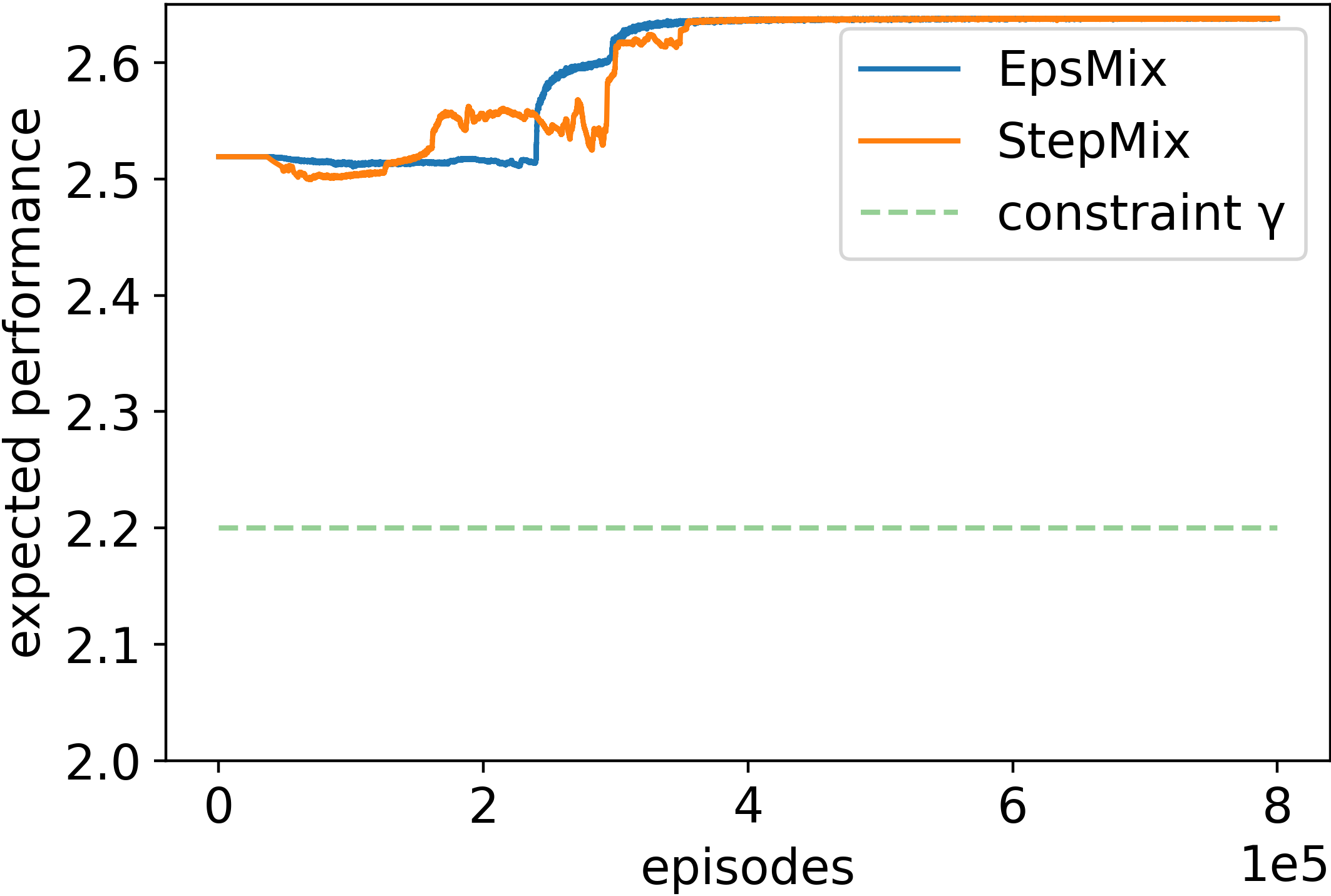}\label{fig:7}}
    \subfigure[$\eta=15,n=8000$]{ \includegraphics[width=0.23\textwidth]{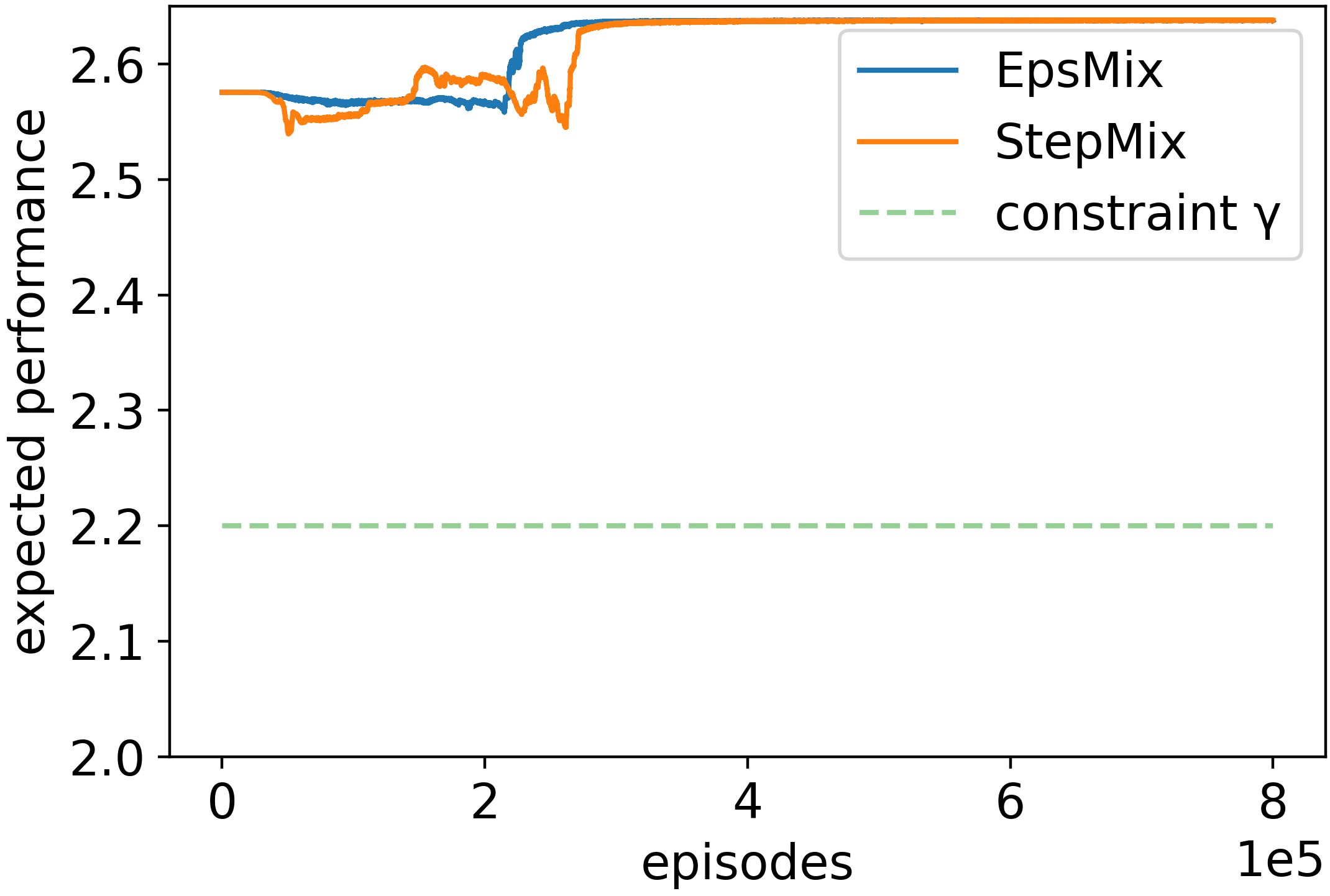}\label{fig:8}}
    \vspace{-0.1in}
    \caption{\small Average expected return of each episode under StepMix, EpsMix, and BPI-UCBVI with offline dataset.}
    \label{fig:offline}
\end{figure*}

In \Cref{fig:performance}, we track the expected return obtained in each episode with different baseline parameter $\eta$ and conservative constraint $\gamma$. We have the following observations. First, both StepMix and EpsMix converge to the optimal policy with no constraint violation in all settings. Between StepMix and EpsMix, the latter exhibits slightly faster convergence. They both tend to stay on the baseline policy when the information is not sufficient, implied by the constant expected return at the beginning of the learning process. When more information is collected, these two algorithms will deviate from the baseline policy and converge to the optimal policy. In contrast, BPI-UCBVI converges to the optimal policy as well, but violates the conservative constraints in earlier episodes. Besides, more stringent constraint $\gamma$ makes StepMix and EpsMix more conservative. Both algorithms experience delayed convergence when $\gamma$ increases. Meanwhile, a better baseline policy also leads to better learning performance throughout the learning process. 

We report the performance of learning with an offline dataset in \Cref{fig:offline}. We use the baseline Boltzmann policy with $\eta=10$ and $\eta=15$ to collect the offline dataset. The numbers of offline trajectories are set to be $5000$ and $8000$, respectively. The conservative constraint $\gamma$ is set to be 2.2.  \Cref{fig:offline} shows that learning a baseline policy from the offline dataset and using it as an input to StepMix and EspMix does not affect their performances significantly. With more offline trajectories collected, the algorithms start from a better baseline and converge to the optimal policy faster.

\subsection{Empirical Comparison with DOPE and OptPess-LP}
In this subsection, we empirically compare the learning performances of StepMix, EpsMix, DOPE~\citep{bura2022dope} and OptPess-LP~\citep{liu2021learning}.

\textbf{Synthetic Homogeneous Environment.} In this experiment, we set $S$ to be 4, $A$ to be 2, and $H$ to be 3. In order to match the homogeneous environment assumption under DOPE and OptPess-LP, we set $P_h=P$ and $r_h=r$ for any $h\in[H]$, and randomly generate $P$ an $r$ as in \Cref{sec:exp1}. As DOPE and OptPess-LP are both developed to solve CMDP problems with general cost functions, to match the conservative constraint considered in this work, we set the corresponding cost function as $c(s,a) = 1 - r(s,a)$ and set the constraint as $\mathbb{E}_{\pi}[\sum_{h=1}^H c(s_h,a_h)]\le H - \gamma$.



\textbf{Results.}
We adopt the Boltzmann policy from \Cref{sec:exp1} as the baseline policy and set $\eta$ to be 5. We run each algorithm for 10 trials and plot the average regrets in \Cref{fig:9} and the average expected return of each episode in \Cref{fig:10}.

\begin{figure}[ht]
    \centering
    \if{0}
    \begin{minipage}[t]{0.3\textwidth}
        \centering
        \includegraphics[width=0.85\textwidth]{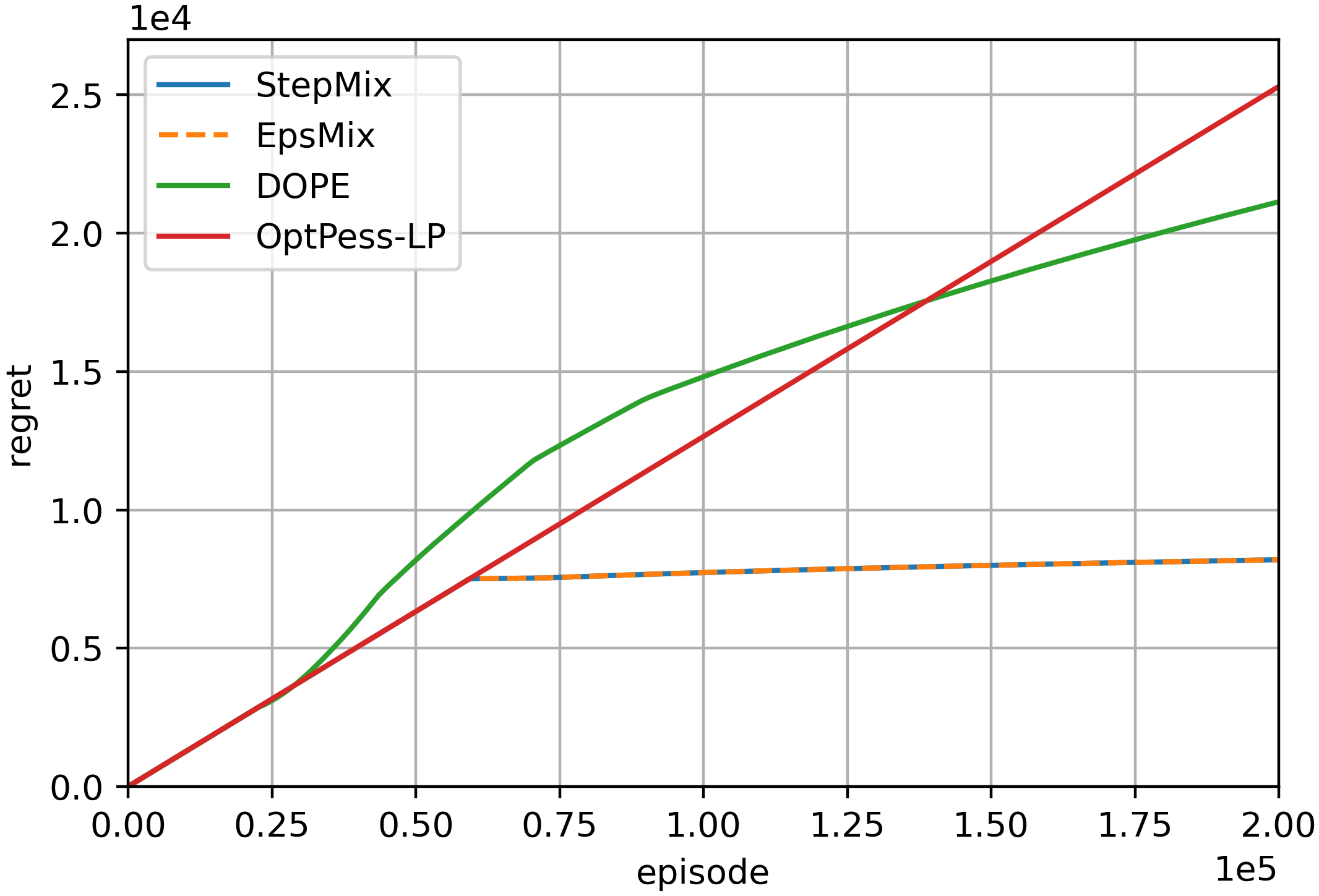}
        \caption{The average regrets of each episode under StepMix, EpsMix, DOPE and OptPess-LP algorithms in a homogeneous environment with $H=3$, $|\Sc|=4$ and $|\Ac|=2$.}
    \label{fig:9}
    \end{minipage}
    \begin{minipage}[t]{0.3\textwidth}
       \includegraphics[width=0.85\textwidth]{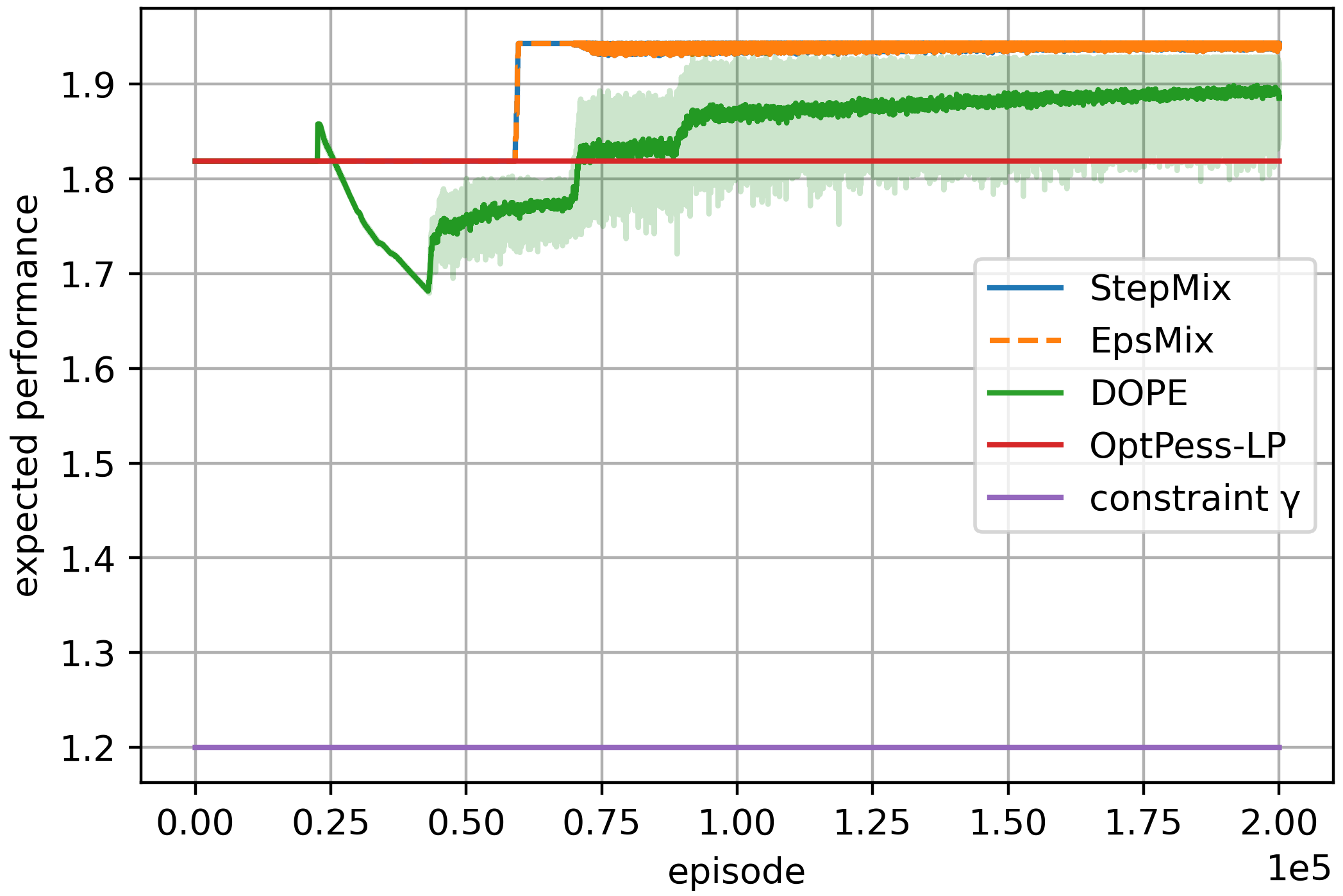}
    \end{minipage}
    \fi
    \subfigure[]{ \includegraphics[width=0.23\textwidth]{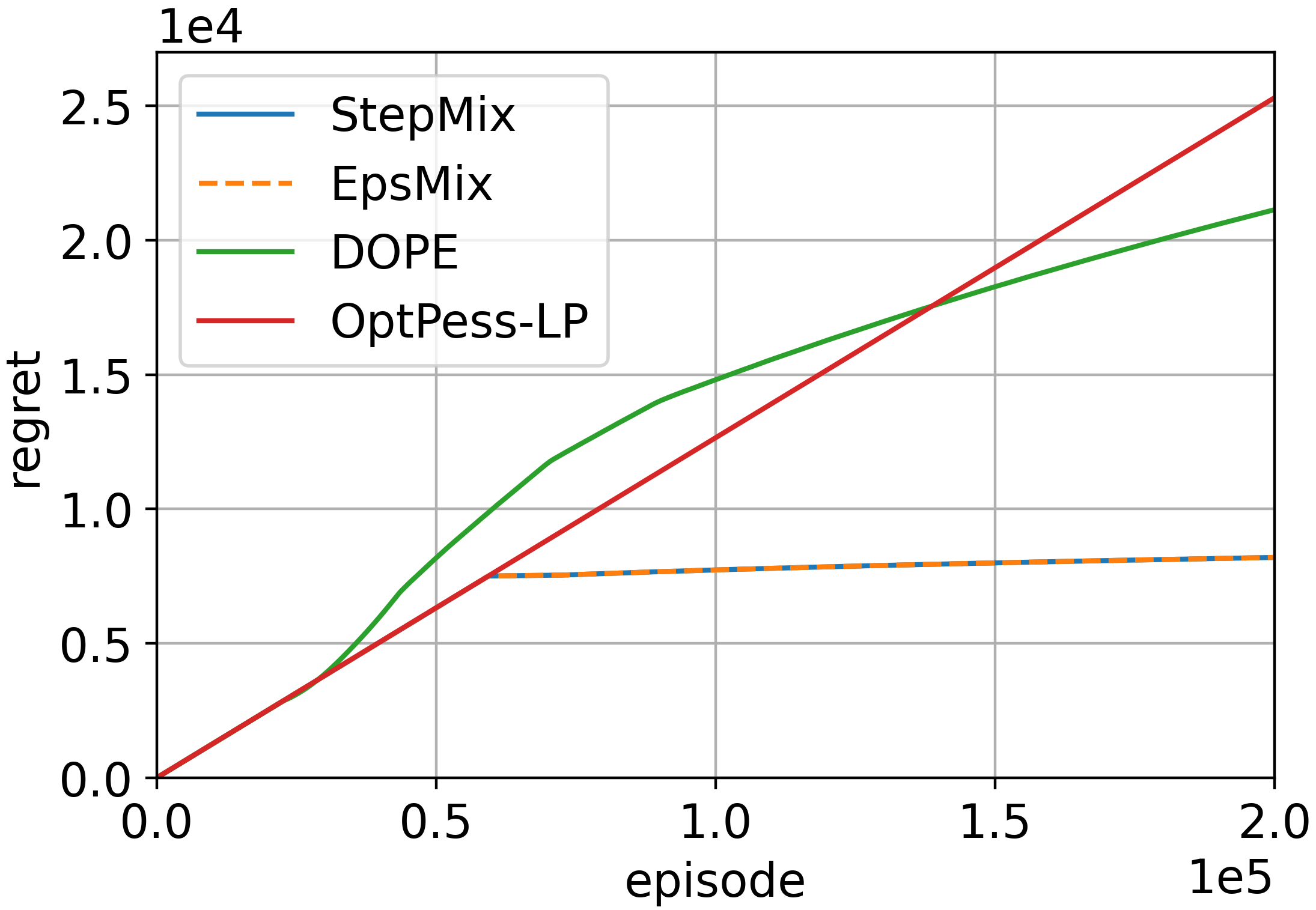}\label{fig:9}}
    \subfigure[]{ \includegraphics[width=0.23\textwidth]{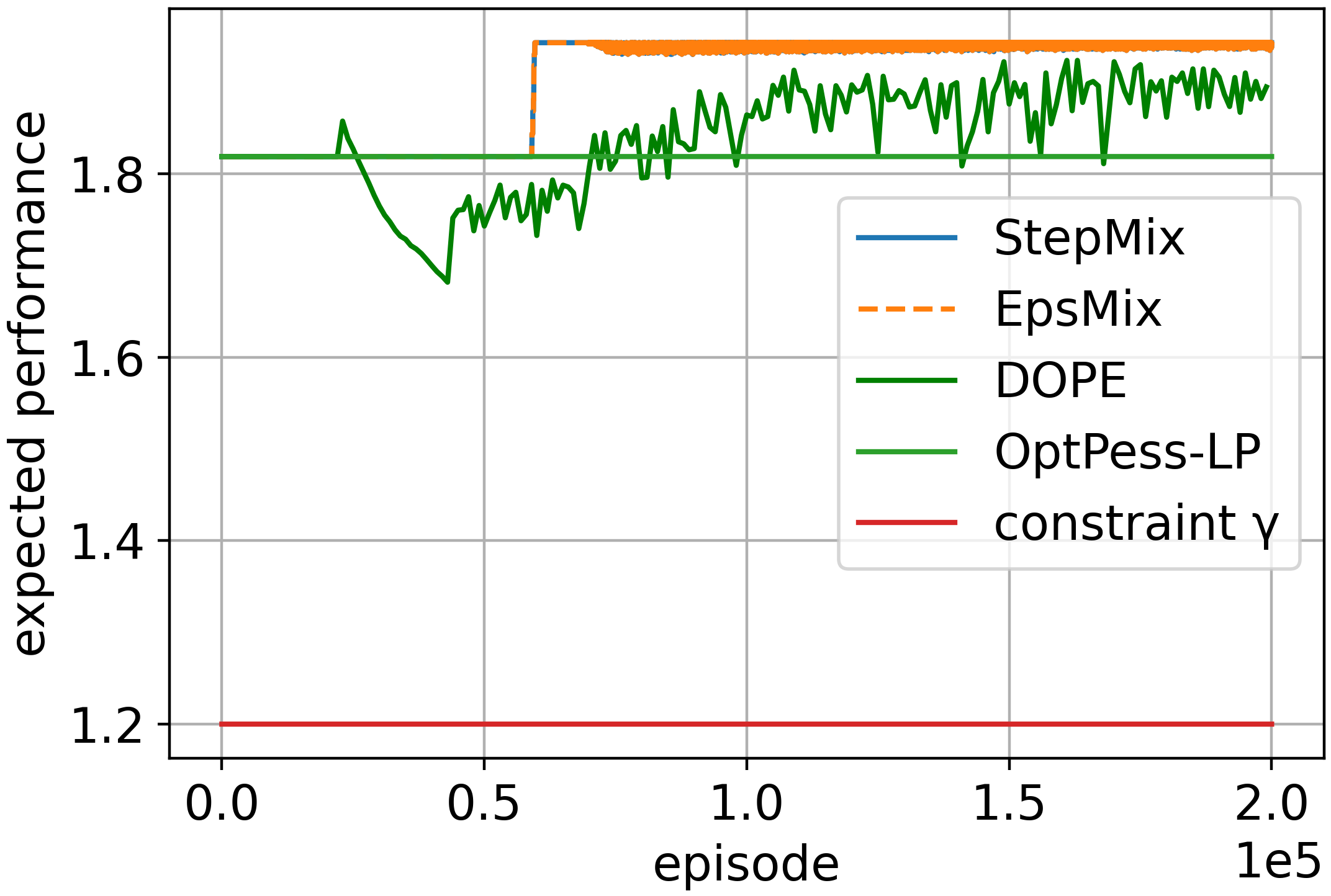}\label{fig:10}}
    \caption{Performance comparison between StepMix, EpsMix, DOPE, and OptPess-LP. (a) Average regret. (b) Average expected return per episode. }
    \label{fig:compare_dope}
\end{figure}

\if{0}
\begin{figure*}[hpbt]
    \centering
    \subfigure[$\eta=10,n=5000$]{ \includegraphics[width=0.23\textwidth]{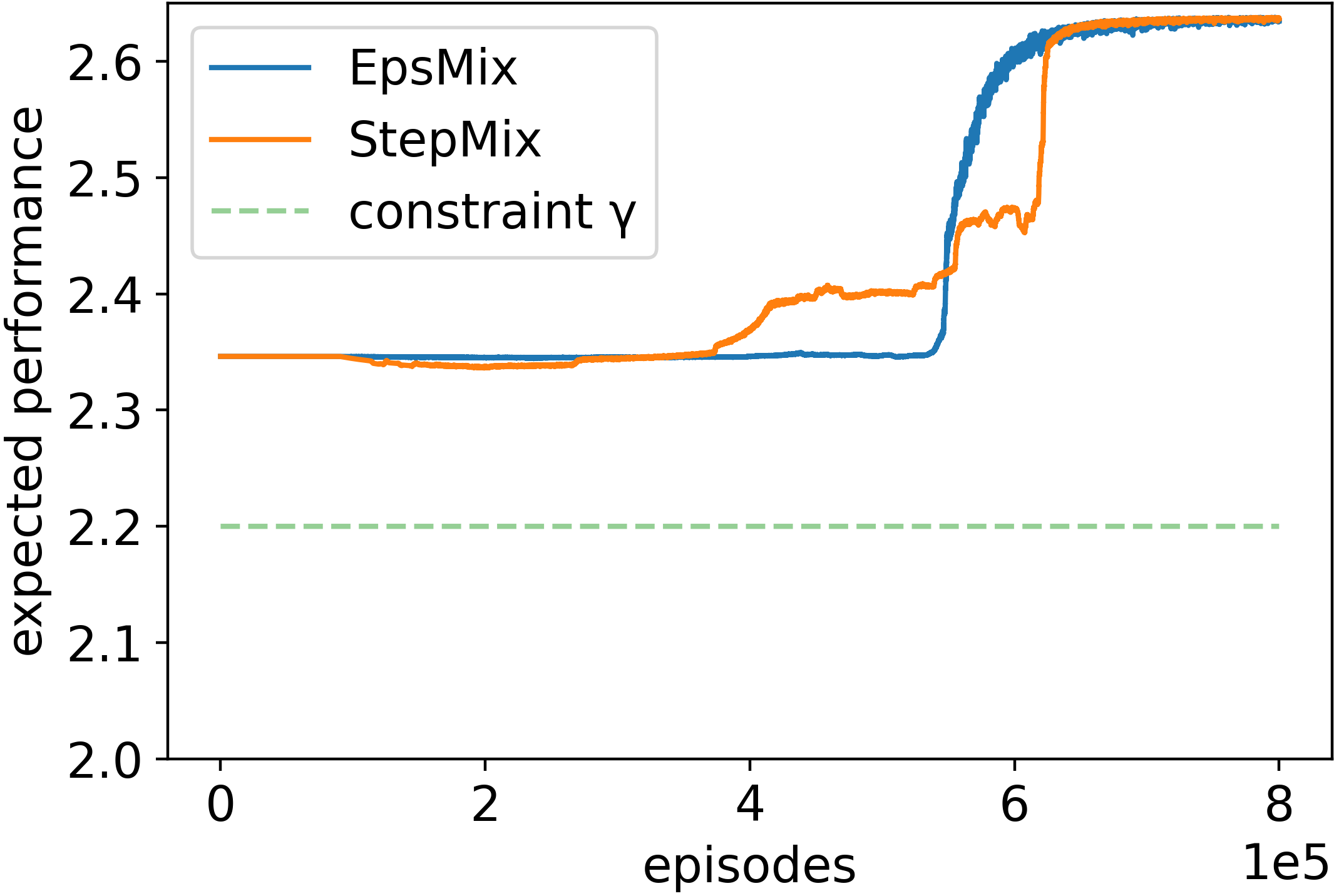}\label{fig:5}}
    \subfigure[$\eta=10,n=8000$]{ \includegraphics[width=0.23\textwidth]{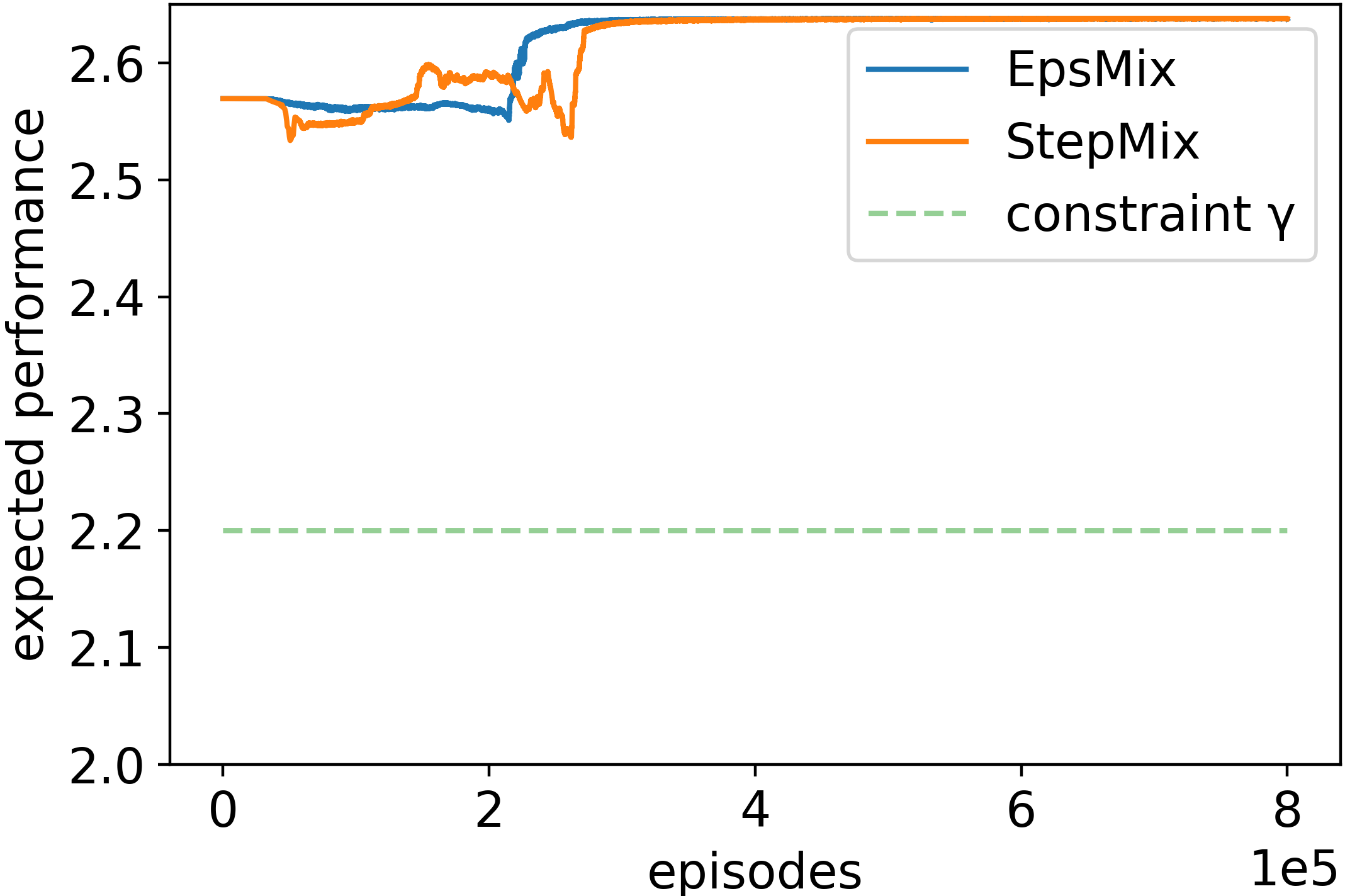}\label{fig:6}}
    \subfigure[$\eta=15,n=5000$]{ \includegraphics[width=0.23\textwidth]{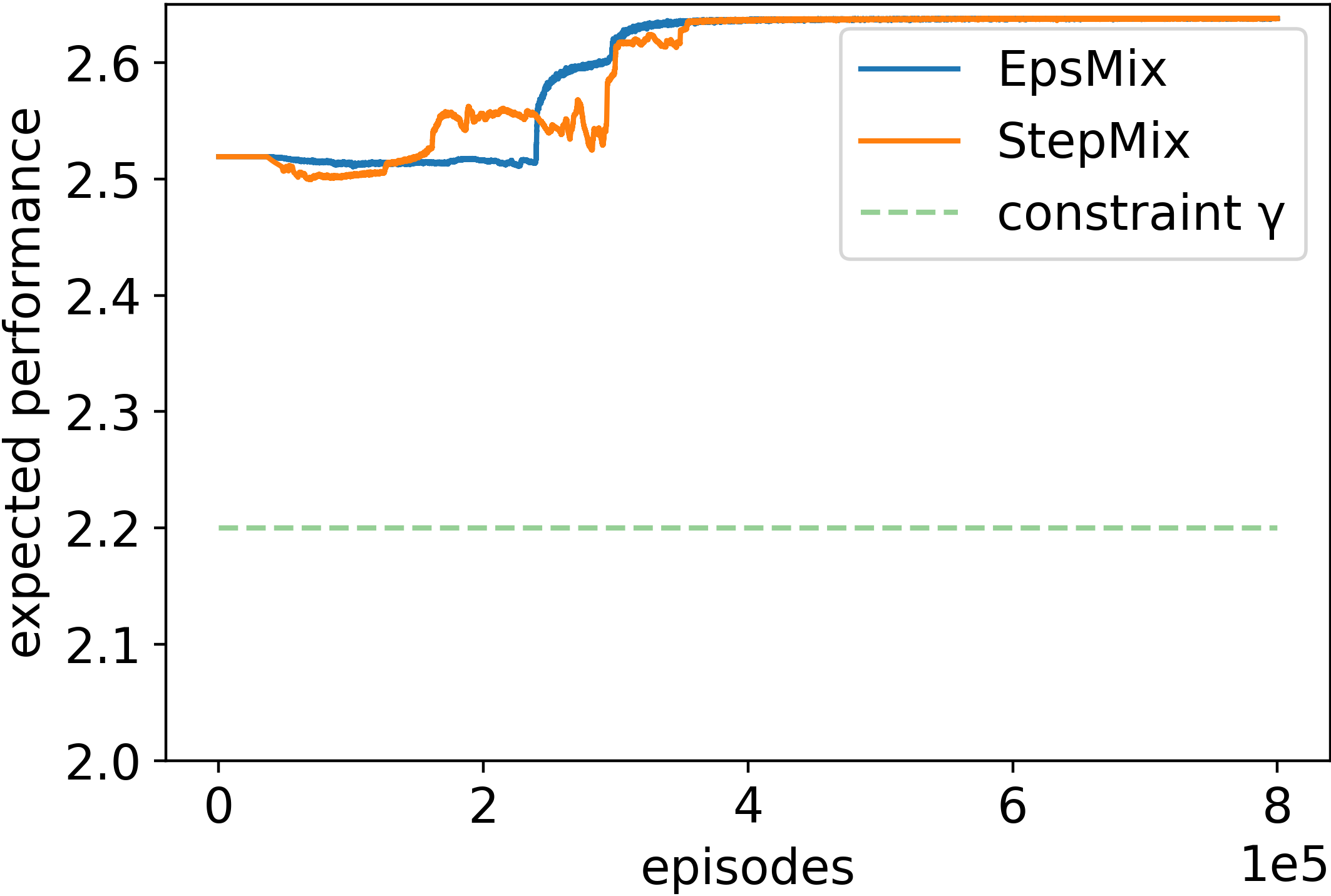}\label{fig:7}}
    \subfigure[$\eta=15,n=8000$]{ \includegraphics[width=0.23\textwidth]{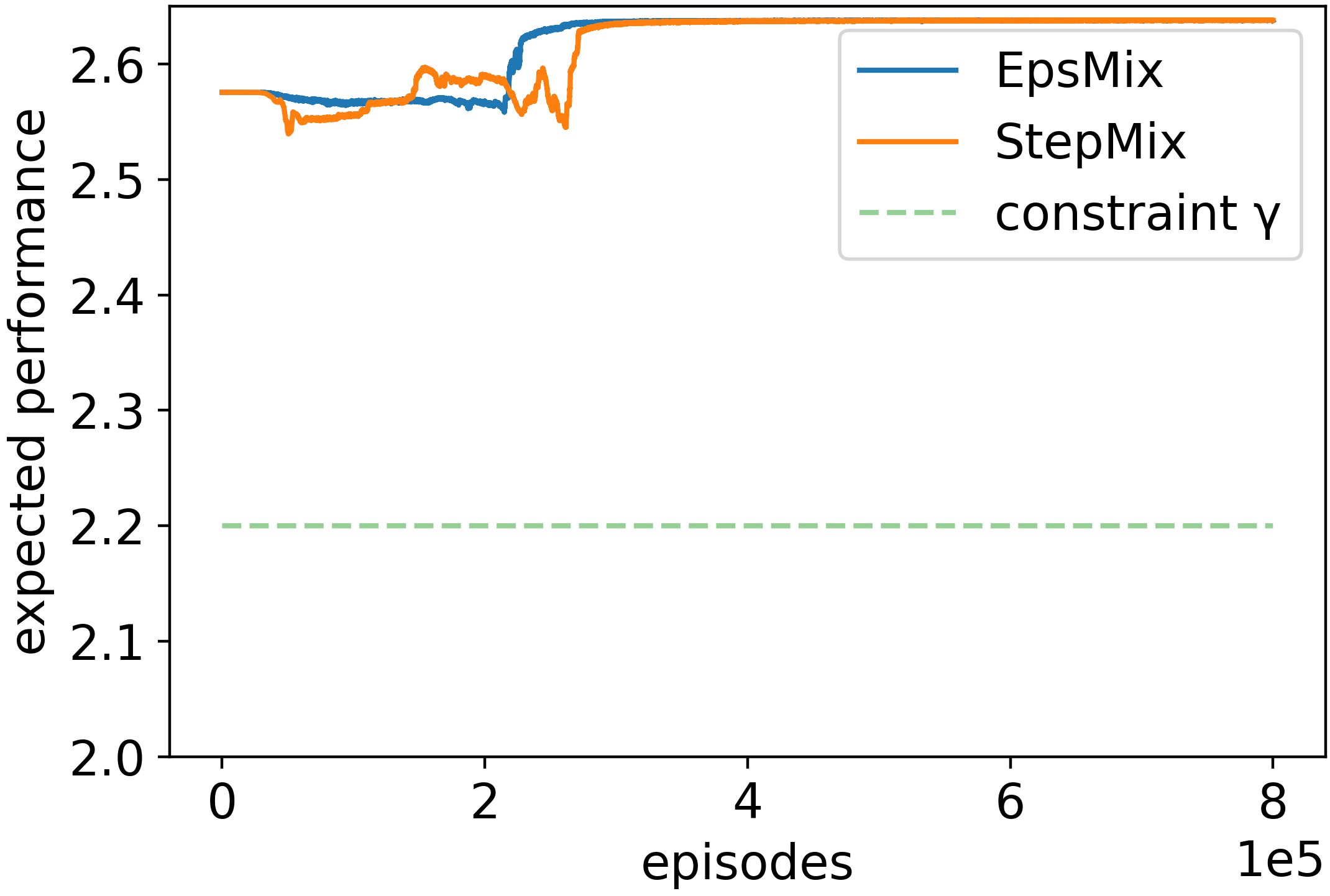}\label{fig:8}}
    \vspace{-0.1in}
    \caption{\small Average total reward of each episode under StepMix, EpsMix, and BPI-UCBVI with offline dataset.}
    \label{fig:offline}
\end{figure*}
\fi

We observe that all four algorithms achieve the same performance at the beginning of the learning process, implying that they all adopt the baseline policy initially. After that, DOPE is the first algorithm to deviate from the baseline and explore other safe policies, followed by StepMix and EpsMix. Although DOPE starts the exploration earlier, it actually renders much higher regret than StepMix and EpsMix. This implies that the exploration under DOPE is not as efficient as the near-optimal exploration strategies adopted by StepMix and EpsMix. On the other hand, OptPess-LP stays on the baseline throughout the learning horizon, leading to a linearly increasing regret. This is because OptPess-LP does not explore sufficiently, and thus is unable to identify a safe exploration policy other than the baseline in this scenario. Similar phenomenon has been observed in~\citet{bura2022dope}. 
\Cref{fig:10} also shows that the constraint violation is zero throughout the learning horizon under all four algorithms.

\if{0}
In \Cref{fig:7}, we fix the StepMix algorithm and baseline policy with $k=20$, but change the constraint as $\gamma=(1-\alpha)V^{\pi^b}$. Generally, if the $\alpha$ is smaller, the constraint is tighter. It is clear in the figure that, by changing the corresponding $\alpha$ from $0.2$ to $0.05$, the StepMix algorithm becomes more conservative and converge to optimal policy slower. 

We also want to show \Cref{fig:9}, which is corresponding to \Cref{lem:finite} very closely. \Cref{lem:finite} is very critical to our analysis result, which claims that there are comparably finite episodes on the step mixture policies. In \Cref{fig:9}, we can see that the number of episodes spent on the step mixture policies of each step is comparably finite to $h_0=0$ when the number of total episode $K$ increases. And, $h_0=0$ means that the algorithm is reduced to the pure optimistic algorithm , of which the number of episodes increases linearly. Generally, the pure UCBVI episodes will dominate and our algorithm will have the same leading order with it.

Summarizing these experiments results, our algorithms works complying with our analysis results. They converge to the optimal policy, do not violate the constraint under different settings and even performs better than UCBVI with a good baseline policy and favorable constraints.
\fi

\if{0}
\begin{figure}[hpbt]
    \centering
    \subfigure[Regret with $\beta=1$]{ \includegraphics[width=0.24\textwidth]{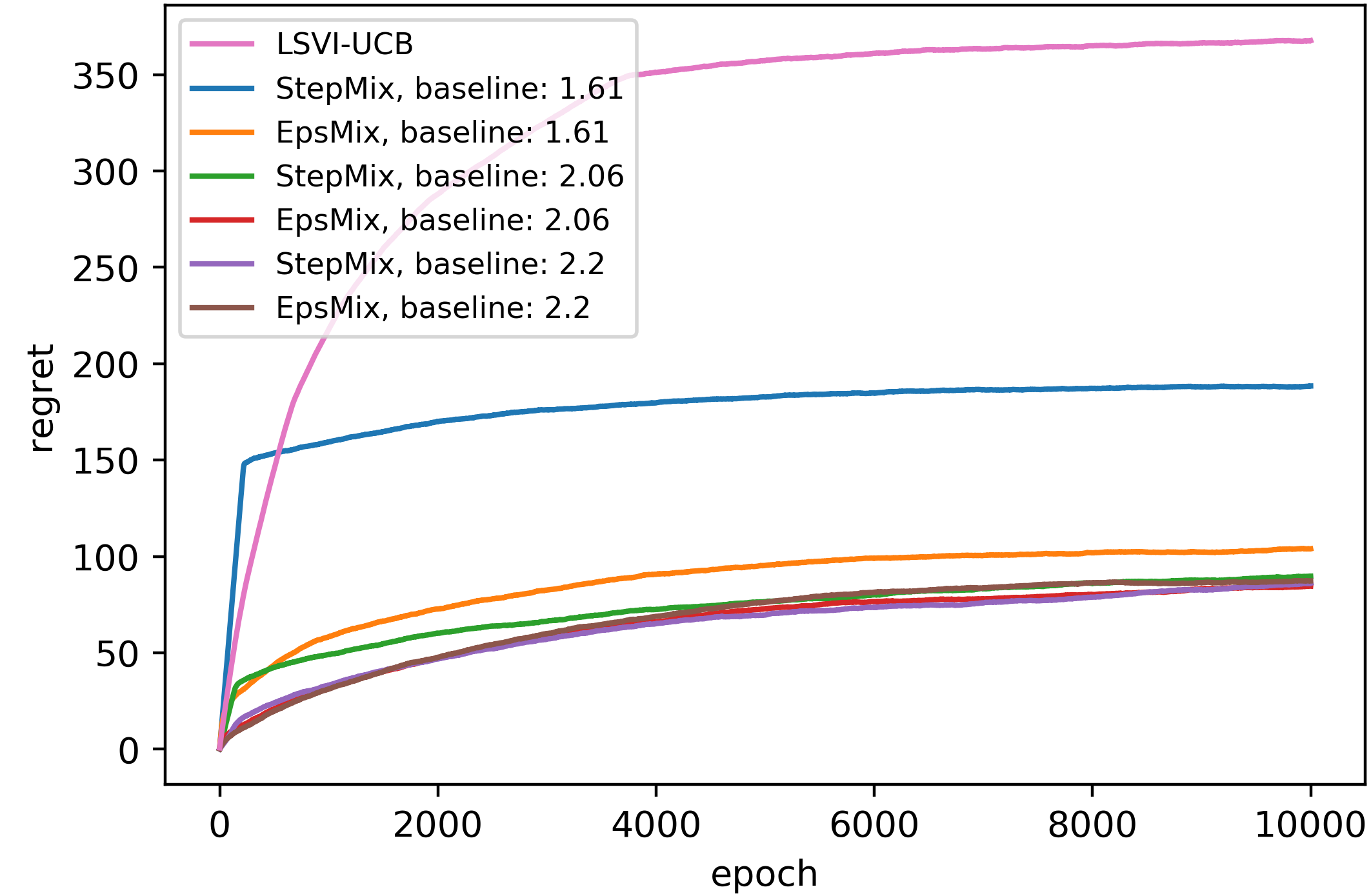}\label{fig:regret}} 
    \subfigure[$k=20$, $\beta=1$]{ \includegraphics[width=0.24\textwidth]{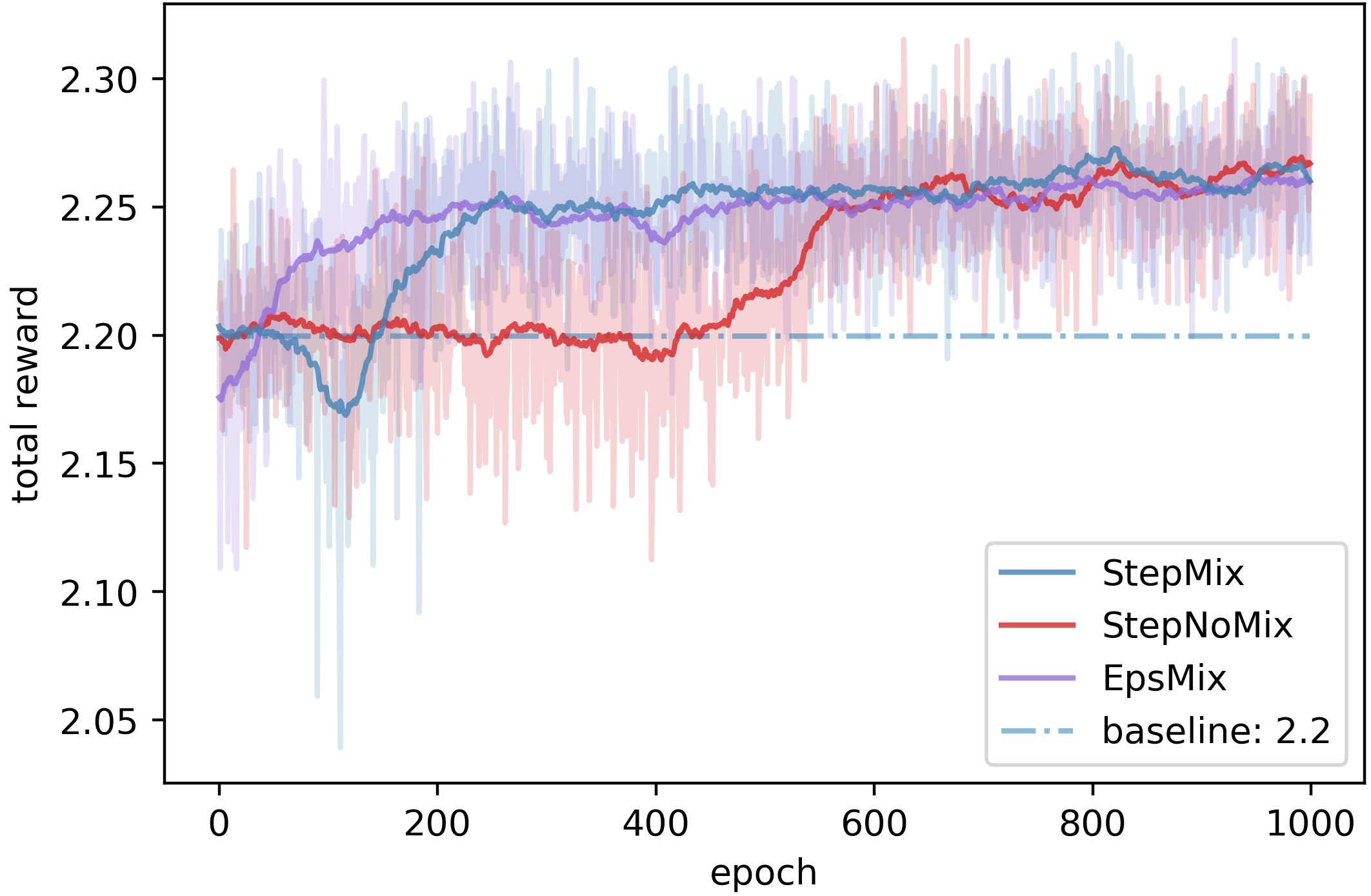}\label{fig:nomix}}
    \subfigure[$k=20$, $\beta=2$, $30$ offline trajectories]{\includegraphics[width=0.24\textwidth]{images/offline/k20.0_beta2.0_30offline_in_main.png}\label{fig:offline}}
    \subfigure[violations vs. number of offline trajectories]{ \includegraphics[width=0.24\textwidth]{images/offline/5to20offlineViolations_log.png}\label{fig:logvio}}
    \vspace{-0.1in}
    \caption{\small Online: (a) regret, (b) compare with StepNoMix. Offline: (c) large $N_1$, (d) varying  $N_1$.}
    \label{fig:performance2}
\end{figure}

We compare the regret performances in \Cref{fig:regret}. All algorithms achieve sub-linear regret, and the regret under StepMix-LSVI and EpsMix-LSVI are much lower than that under LSVI-UCB. In general, a better baseline policy leads to lower regret under our algorithms, which is consistent with the theory. 
In \Cref{fig:nomix}, we compare with an additional algorithm termed StepNoMix, a modified version of StepMix-LSVI. Instead of constructing a step mixture policy, StepNoMix will execute the first concatenated policy whose LCB is above the threshold. Compared with StepMix-LSVI, StepNoMix stays at the baseline for more episodes, indicating that it is less effective in exploring unknown dimensions. 
\fi

\if{0}
\section{compare with existing LCB techniques}

Here we compare our LCB with two existing works \citep{xie2021policy, xie2022armor}.

First, the most important difference is that our LCB requires no coverage assumption for behavior policy, while both \citet{xie2021policy, xie2022armor} require the coverage assumption on the behavior policy in order to bound the estimation error of the pessimistic policy constructed from the offline dataset. This is because the LCB constructed in our work is the lower bound of value function under the exploration policy $\pi^k$
, while the LCB in \citet{xie2021policy, xie2022armor} is for the pessimistic policy constructed from the offline dataset. Essentially, our algorithm relies on the online algorithm to approach the optimal policy thus does not rely on the coverage of the behavior policy, while \citet{xie2021policy} decouples the offline learning and online fine-tuning process and still requires the coverage assumption for the offline learning.

Second, compare to \citet{xie2021policy}, our reference value functions are different. In \citet{xie2021policy}, in order to derive the pessimistic policy and utilize Bernstein's inequality, it constructs a reference function based on an LCB algorithm. In our work, the reference functions are the true value functions of candidate policies, which are chosen from a series of policies mixed by the baseline policy and optimistic policy produced from a Bernstein-style UCB algorithm. This is essentially the reason why we are able to remove the coverage assumption for the behavior policy while \citet{xie2021policy} still relies on it to first bound their reference function.

Third, compare to \citet{xie2022armor}, our LCB expression is computationally efficient. In \citet{xie2022armor}, the authors use an MLE-style estimation of models to achieve the confidence set of models, which is implicit compared with our explicit construction of the Bernstein-style bonus term. Our approach is more computationally efficient.

\fi

\section{Conclusions}
\label{sec:conc}
We investigated conservative exploration in episodic tabular MDPs. Different than the majority of existing literature, we considered a stringent episodic conservative constraint, which motivated us to incorporate mixture policies in conservative exploration. We proposed two model-based algorithms, one with step mixture policies and the other with episodic randomization. Both algorithms were proved to achieve near-optimal regret order as that under the constraint-free setting, while never violating the conservative constraint in the learning process. We also investigated a practical case where the baseline policy is not explicitly given to the algorithm, but must be learned from an offline dataset. We showed that as long as the dataset is sufficiently large, the offline learning step does not affect the conservative constraint or the regret of our proposed algorithms. Experimental results in a synthetic environment corroborated the theoretical analysis and shed some interesting light on the behavior of our algorithms. 


\section*{Acknowledgements}
The work of DL, RH and JY was supported in part by the US National Science Foundation (NSF) under awards 2030026, 2003131, and 1956276. The work of CS was supported in part by the US NSF under awards 2143559, 2029978, 2002902, and 2132700.

\bibliography{conservative}
\bibliographystyle{apalike}

\newpage
\appendix
\onecolumn

\section{Related Works}
\label{appx:ref}

\textbf{Constrained RL with Baseline Policies.} Conservative exploration studied in this paper can be viewed as a specific case of the Constrained Markov Decision Process (CMDP) \citep{Altman:CMDP:1999}, which has been investigated in both offline and online settings. 
In the {\it offline} setting, a given baseline policy produces a set of trajectories for the agent to learn a policy that is guaranteed to perform at least as good as the baseline with high probability without actually interacting with the MDP~\citep{Bottou:2013:JMLR,Thomas:AAAI:2015,Thomas:2015:ICML,swaminathan:2015ICML,Petrik:2016:NIPS,laroche:ICML19,Simao:2019:AAAI}. 
It can also be extended to the {\it semi-batch} setting~\citep{pirotta2013safe}. 
In the {\it online} setting, the agent has to trade off exploration and exploitation while interacting with the MDP. Several algorithms have been proposed in the literature~\citep{garcelon:AISTATS:2020,yang2021reduction}. 
\citet{garcelon:AISTATS:2020} introduce a Conservative Upper-Confidence Bound for Reinforcement Learning (CUCRL2) algorithm for both finite horizon and average reward problems with $O(\sqrt{T})$ regret. 
\citet{yang2021reduction} propose a reduction-based framework for conservative bandits and RL, which translates a minimax lower bound of the non-conservative setting to a valid lower bound for the conservative case. They also propose a Budget-Exploration algorithm and show that its regret scales in $\tilde{O}\left(\sqrt{H^3SAK} + \frac{H^3SA\Delta_0}{\kappa(\kappa+\Delta_0)}\right)$ for tabular MDPs, where $\Delta_0$ is the suboptimality gap of the baseline policy, and $\kappa$ is the tolerable performance loss from the baseline. However, all these works assume {\it cumulative} conservative constraint.

\textbf{Other Forms of Constraints.} Beside the constraint imposed by a baseline policy, which is generally ``aligned'' with the learning goal, CMDP also studies the case where the algorithm must satisfy a set of constraints that potentially are not aligned with the reward. In general, both cumulative cost constraints \citep{efroni2020exploration,turchetta2020safe,zheng2020constrained,qiu2020upper,ding2020natural,kalagarla2020sample,liu2021learning,wei2022triple,ghosh2022provably} and episodic cost constraints \citep{liu2021learning,bura2022dope,huang2022safe} have been investigated. Assuming a known safe baseline policy that satisfies the corresponding constraints, OptPess-LP \citep{liu2021learning} is shown to achieve a regret of $\tilde{O}(\frac{1}{\kappa}\sqrt{H^6S^3AK})$ without any constraint violation with high probability, while DOPE \citep{bura2022dope} improves the regret to $\tilde{O}(\frac{1}{\kappa}\sqrt{H^6S^2AK})$, where $\kappa$ denotes the Slater parameter. We note that both algorithms do not achieve the optimal regret in the unconstrained counterpart, due to the adopted linear programming-based approaches. 
Beyond tabular setting, CMDP has also been discussed in linear~\citep{ding2021provably,ghosh2022provably,amani2021safe,yang2021reduction} or low-rank models~\citep{huang2022safe}.  
Other formulations different from conservative exploration or CMDP, such as minimizing the variance of expected return \citep{10.5555/3042573.3042784} or generally, maximizing some utility function of state-action pairs \citep{ding2021provably}, have also been investigated.  
Lastly, \citet{yang2021accelerating} study constrained reinforcement learning with a baseline policy that may not satisfy the given set of constraints.




\if{0}
\citet{Yang:ICML:2020}  proposes an online RL algorithm, namely the MatrixRL, that leverages ideas from linear bandit to learn a low-dimensional representation of the probability transition model while carefully balancing the exploitation-exploration tradeoff. It shows that MatrixRL achieves a regret bound ${O}\big(H^2d\log T\sqrt{T}\big)$. 
\citet{wang2021optimism} proposes the USVI-UCB algorithm under a weaker optimistic closure assumption, which achieves an $\tilde{\mathcal{O}}(\sqrt{d^3H^3T})$ regret. 
This result is improved to $\tilde{\mathcal{O}}(d\sqrt{H^3T})$ in \citet{zanette2020learning}, which proposes another weaker assumption called low inherent Bellman error. Instance-dependent logarithmic regret bounds are established for linear MDPs in \citet{he2021logarithmic}. 
In addition, there is another related line of works focusing on linear mixture MDPs \citep{ayoub2020model,cai2020provably,zhou2021provably}. 
We note that those algorithms do not consider any conservative constraints into their formulation. 
\fi

\textbf{Safe Bandits.} Bandits problem is a standard RL problem where it interacts with a stationary environment, which reduces the difficulties of learning. Several constraints are considered in the bandits setting. The first is that the cumulative expected reward of an agent should exceed a certain threshold. This setting is originally studied in \citet{wu2016conservative}, which adopts an UCB type of exploration and checks whether the policy satisfies the conservative constraint. \citet{kazerouni2017conservative,garcelon2020improved,pacchiano2021stochastic} then extend the conservative setting to contextual linear bandits. The second constraint is much  stronger, as it requires that each arm played by the learning agent be safe given the baseline or the threshold. \citet{NEURIPS2019_09a8a897} and \citet{khezeli2020safe} both use an LCB type of algorithm to ensure the arms selected by the algorithms are safe under linear bandits setting. \citet{du2021one} consider conservative exploration with a sample-path constraint on the actual observed rewards rather than in expectation. 

\textbf{Policy Optimization.} This is another research direction in RL that utilizes baseline policies \citep{schulman2015trust}. However, the focus and assumptions of these papers are very different from this work. For example, \citet{zhong2021optimistic} and \citet{luo2021policy} focus on the non-stationary and adversary environments, respectively. While policy optimization can achieve sublinear regret under certain MDP models \citep{shani2020optimistic}, it usually lacks performance guarantees during the learning process, which is in stark contrast to our results.

\textbf{Other LCB Techniques.} 
We highlight the differences between the LCBs used in our online algorithms StepMix and EpsMix and two LCB techniques used in \citet{xie2021policy, xie2022armor}. First and foremost, {\citet{xie2021policy,xie2022armor} study offline RL problems.} They both impose a coverage assumption on the behavior policy in order to bound the estimation error of the pessimistic policy constructed from the offline dataset. On the other hand, our LCB has no such coverage assumption for the behavior policy, since the LCB constructed in our work is a lower bound of the value function under the online exploration policy $\pi^k$. 
Second, the reference value functions in this work are different than those in \citet{xie2021policy}, which constructs a reference function based on an LCB algorithm so that it can derive the pessimistic policy and utilize Bernstein's inequality. In our work, however, the reference functions are the true value functions of candidate policies, which are chosen from a series of policies mixed by the baseline policy and optimistic policy produced from a Bernstein-style UCB algorithm. 
Last but not the least, our LCB expression is more computationally efficient than that in \citet{xie2022armor}. 

\section{Notations}\label{appx:notation}
We list the notations of common quantities as follows.
\begin{table*}[!hbt]
    \label{table:notation}
    \begin{center}
    \adjustbox{max width= \linewidth}{
    \begin{tabular}{ccc}
    \Xhline{3\arrayrulewidth}
    Notation & Meaning & Definition \\
    \Xhline{2\arrayrulewidth}
    $r_h(s,a)$ & reward & - \\
    \midrule
    $P_h(s'|s,a)$ & transition probability & - \\
    \midrule
    $n_h^k(s,a)$ & visitation count & $\sum_{\tau=1}^{k-1} \mathds{1}\{s_h^\tau=s,a_h^\tau=a\}$ \\
    \midrule
    $n_h^k(s,a,s')$ & visitation-transition count & $\sum_{\tau=1}^{k-1} \mathds{1}\{s_h^\tau=s,a_h^\tau=a,s_{h+1}^\tau=s'\}$ \\
    \midrule
    $\hP_h^k(s'|s,a)$ & empirical estimate of transition probability & $\frac{n_h^k(s,a,s')}{n_h^k(s,a)}$ if $n_h^k(s,a)\ge 1$; $\frac{1}{S}$, otherwise \\
    \midrule
    $d_h^\pi(s,a)$ & occupancy measure under policy $\pi$ & $\mathbb{E}_{\pi}[\mathds{1}\{s_h=s,a_h=a\}]$\\
    \midrule
    $d_h^k(s,a)$ & occupancy measure under policy $\pi^k$ & $d^{\pi^k}_h(s,a)$\\
    \midrule
    $\bar{n}_h^k(s,a)$ & expected visitation count & $\sum_{\tau=1}^{k-1}d^\tau_h(s,a)$\\
    \midrule
    
    $Q_h^\pi(s,a)$ & true Q function & $\mathbb{E}_\pi[\sum_{i=h}^{H}r_i(s_i,a_i)|s_h=s,a_h=a]$\\
    \midrule
    $V_h^\pi(s)$ & true V function & $\mathbb{E}_\pi[\sum_{i=h}^{H}r_i(s_i,a_i)|s_h=s]$\\
    \midrule
    $\pi^\star$ & optimal policy & $\argmax_\pi V_1^\pi(s_1)$ \\
    \midrule
    $\pi^b$ & baseline policy & - \\
    \midrule
    $\pi^{\star,h_0}$ & step-wise optimal policy & $\{\pi_1^b,\cdots,\pi_{h_0}^b,\pi_{h_0+1}^\star,\cdots,\pi_H^\star\}$ \\
    \midrule
    $\bar{\pi}^k$ & global optimistic policy & constructed from \Cref{alg:step_remake} \\
    \midrule
    $\pi^{k,h_0}$ & step-wise optimistic policy & $\{\pi_1^b,\cdots,\pi_{h_0}^b,\bar{\pi}_{h_0+1}^k,\cdots,\bar{\pi}_{H}^k\}$ \\
    \midrule
    $Q^{k,h_0},V^{k,h_0},Q^{\star,h_0},V^{\star,h_0}$ & corresponding true value functions & $Q^{\pi^{k,h_0}},V^{\pi^{k,h_0}},Q^{\pi^{\star,h_0}},V^{\pi^{\star,h_0}}$ \\
    \midrule
    $\tilde{Q}^{k,h_0}, \tilde{V}^{k,h_0}$ & Upper Confidence Bounds & defined in \Cref{eqn:ucbQ,eqn:ucbV} \\
    \midrule
    $\utilde{Q}^{k,h_0}, \utilde{V}^{k,h_0}$ & Lower Confidence Bounds & defined in \Cref{eqn:lcbQ,eqn:lcbV} \\
    \midrule
    $G^{k,h_0}$ & G function & defined in \Cref{eqn:G} \\
    \midrule
    $\beta(n,\delta)$ & logarithm term involved in $\mathcal{E}$ & $\log(SAH/\delta) + S\log(8e(n+1))$\\
    \midrule
    $\beta^{\text{cnt}}(\delta)$ & logarithm term involved in $\mathcal{E}^\text{cnt}$ & $\log(SAH/\delta)$\\
    \midrule
    $\beta^\star(n,\delta)$ & logarithm term involved in $\mathcal{E}^\star$ & $\log(SAH/\delta) + \log(8e(n+1))$\\
    \Xhline{3\arrayrulewidth}
    \end{tabular}
    }
    \end{center}
\end{table*}
 
We also adopt the following $\min,\max$ notations:
    \[a \wedge b = \min(a,b), a \vee b = \max(a,b).\]
For any given policy $\pi$ and $Q$ function, we denote
    \[\pi_h Q_h(s)=\langle \pi_h(\cdot|s), Q_h(s,\cdot)\rangle=\sum_{a\in\mathcal{A}}\pi_h(a|s)Q_h(s,a).\]
For any given transition kernel $P_h$ and value function $V_{h+1}$, we define the variance of $P_h V_{h+1}(s,a)$ as follows.
    \[\Var_{P_h}(V_{h+1})(s,a)=\mathbb{E}_{s'\sim P_h(\cdot|s,a)}[(V_{h+1}(s')-\mathbb{E}_{s'\sim P_h(\cdot|s,a)}[V_{h+1}(s')])^2].\]

At last, we introduce three types of good events and their notations that will be intensively used in the following proofs.

The first type of good events characterizes the connection between the true visitation counts and the expected visitation counts:
\begin{align*}
    \mathcal{E}^{\text{cnt}}(\delta)  \triangleq \left\{\forall k\in [K], \forall h\in[H], \forall(s,a)\in\mathcal{S}\times\mathcal{A}: n_h^k(s,a)\ge \frac{1}{2}\bar{n}^k_h(s,a)-\beta^{\text{cnt}}(\delta) \right\},
\end{align*}
where $\beta^{\text{cnt}} = \log(SAH/\delta)$, $n_h^k(s,a)$ denotes the number of visitations of state-action pair $(s,a)$ and $\bar{n}_h^k(s,a)$ denotes the expected visitation count.

The second type of good events, defined as follows, upper bounds the KL divergence between the estimated transition distribution and the true transition distribution.
\begin{align*}
    \mathcal{E}(\delta) \triangleq \left\{\forall k\in [K], \forall h\in[H], \forall(s,a)\in\mathcal{S}\times\mathcal{A}: \mbox{KL}(\hP_h^k(\cdot|s,a),P_h(\cdot|s,a))\le \frac{\beta(n^k_h(s,a),\delta)}{n^k_h(s,a)} \right\},
\end{align*}
where $\beta(n,\delta) = \log(SAH/\delta) + S\log(8e(n+1))$.

The third type of good events provides a Bernstein-style concentration guarantee, defined as follows.
\begin{align*}
    \mathcal{E}^\star(V,\delta)& \triangleq \biggr\{\forall k\in [K], \forall h\in[H], \forall(s,a)\in\mathcal{S}\times\mathcal{A}, \in [H]\cup\{0\}: \left|(\hP_h^k-P_h)V_{h+1}(s,a)\right|\le\\ 
    & \quad \quad \min\bigg\{H, \sqrt{2\Var_{P_h}(V_{h+1})(s,a) \frac{\beta^\star(n^k_h(s,a), \delta)} {n^k_h(s,a)}} + 3H\frac{\beta^\star(n^k_h(s,a), \delta)}{n^k_h(s,a)} \bigg\} \biggr\}, 
\end{align*}
where $\beta^\star(n,\delta) = \log(SAH/\delta) + \log(8e(n+1))$ and $V$ is a value function independent with $\hP_h^k$ and bounded by $H$. Later, $V$ will be chosen separately in StepMix or EpsMix.

\if{0}
\begin{equation*}
\begin{aligned}
    \text{true visitation count: } & n_h^k(s,a) = \sum_{\tau=1}^{k-1} \mathbf{1}(s_h^\tau=s,a_h^\tau=a) \\
    \text{expected visitation count: } & \bar{n}_h^k(s,a) = \sum_{\tau=1}^{k-1} d_h^{\pi^\tau}(s,a) \\
    \text{estimated P: } & \hP_h^k(s'|s,a) = \frac{\sum_{\tau=1}^{k-1}  \mathbf{1}(s_h^\tau=s,a_h^\tau=a,s_{h+1}^\tau=s')}{n_h^k(s,a)} \\
    & \beta(n,\delta) = \log(3SAH/\delta) + S\log(8e(n+1)) \\
    & \beta^\star(n,\delta) = \log(3SAH/\delta) + \log(8e(n+1)) \\
    \text{true Q function: } & Q^\pi_h(s,a) = r_h(s,a) + (P_h V_{h+1}^\pi)(s,a) \\
    \text{true V function: } & V^\pi_h(s,a) = (\pi_h Q_{h}^\pi)(s) \\
    \text{optimal policy and optimal Q-function: } & Q^\star = Q^{\pi^\star} \ge Q^{\pi}, \forall \pi \\
    h_0 \text{ step-wise optimal policy: } &\pi^{\star, h_0}=\{\pi^b_1,\pi^b_2,\cdots,\pi^b_{h_0},\pi^\star_{h_0+1},\pi^\star_{h_0+2},\cdots,\pi^\star_{H}\} \\
    \text{short note: } &\pi^{\star,0} = \pi^\star,\vspace{5pt} \pi^{\star,H} = \pi^b\\
    \text{step-wise optimal Q-function: } & Q^{\star,h_0} = Q^{\pi^{\star,h_0}} \\
\end{aligned}
\end{equation*}

Other notations will be listed, but will not be explained in detail here.
\begin{align*}
    \text{upper bound of step-wise optimal functions: } & \tilde{Q},\tilde{V} \\
    \text{lower bound of step-wise optimal functions: } & \utilde{Q},\utilde{V} \\
    \text{estimated step-wise optimal policy: } & \pi^{k,h_0}_h \\
    \text{lower bound of Q/V-functions of estimated step-wise optimal policy: } & \mathring{Q},\mathring{V}
\end{align*}

\fi

\section{Algorithm Design and Analysis of StepMix}
\label{appx:proof_step}
We first recall the StepMix algorithm (\cref{alg:step_remake}) and provide the PolicyEva subroutine in \cref{alg:policyEva}.

\begin{algorithm}[!htb]
    \caption{The StepMix Algorithm}
    \label{alg:step_remake}
\begin{algorithmic}
\STATE {\bf Input:} $\pi^b$, $\gamma$, $\beta$, $\beta^\star$, $\Dc_0=\emptyset$.
\FOR{$k$ = $1$ to $K$}
    \STATE Update model estimate $\hP$ according to \Cref{eqn:p}.
\STATE \textcolor{blue}{\it \# Optimistic policy identification}

$\tilde{V}_{H+1}^{k}(s)=\utilde{V}_{H+1}^{k}(s)=0,\forall s\in \Sc$.
    \FOR{$h$ = $H$ to $1$}
        \STATE Update $\tilde{Q}_{h}^{k}(s,a)$, $\utilde{Q}_{h}^{k}(s,a), \forall (s,a)\in \Sc\times \Ac$ according to \Cref{eqn:q}.
        \STATE $\gp_h^{k}(s)\gets\argmax_a \tilde{Q}_{h}^{k}(s,a)$, $\tilde{V}^{k}_h(s)\gets{\tilde{Q}^{k}_h(s,\gp_h^k(s))}$, $\utilde{V}^{k}_h(s)\gets \utilde{Q}^{k}_h(s,\gp_h^k(s)),\forall s\in\Sc$.
    \ENDFOR
       \STATE \textcolor{blue}{\it \# Candidate policy construction and evaluation}
    \FOR{$h_0$ = $0$ to $H$}
        \STATE $\pi^{k,h_0} = \{\pi_1^b,\pi_2^b,\cdots,\pi_{h_0}^b,\gp_{h_0+1}^{k},\cdots, \gp_{H-1}^{k}, \gp_{H}^{k}\}$.
       \STATE $\utilde{V}^{k,h_0}=\mbox{PolicyEva}(\hP^k,\pi^{k,h_0})$.
    \ENDFOR
         \STATE  \textcolor{blue}{\it  \# Safe exploration policy selection}
    \IF{$\{h\,|\,\utilde{V}_1^{k,h} \ge \gamma, h=0,1,\ldots,H\}=\emptyset$}
        \STATE $\pi^k = \pi^b$. 
    \ELSE
        \STATE $h^k = \min \{h\,|\,\utilde{V}_1^{k,h} \ge \gamma, h=0,1,\ldots,H\}$.
        \IF{$h^k=0$}
            \STATE $\pi^k = \gp^k$. 
            \ELSE
            \STATE Set $\pi^k$ according to \Cref{eqn:stepmix2}.
        \ENDIF
    \ENDIF
    \STATE Execute $\pi^k$ and collect $\{(s_h^k,a_h^k,s_{h+1}^k)\}_{h=1}^H$.
\STATE $\Dc_n\gets \Dc_{n-1}\cup\{(s_h^k,a_h^k,s_{h+1}^k)\}_{h=1}^H$.
\ENDFOR
\end{algorithmic}
\end{algorithm}

\begin{algorithm}[!htb]
    \caption{PolicyEva Subroutine}
    \label{alg:policyEva}
\begin{algorithmic}
    \STATE {\bf Input:} $\hP^k$, $\pi$
    \STATE {\bf Initialization:} Set $\tilde{V}_{H+1}^k(s)$ and $\utilde{V}_{H+1}^k(s)$ to be $0$ for any $s\in \Sc$.
    \FOR{$h$ = $H$ to $1$}
        \STATE Update $\tilde{Q}_{h}^{k}(s,a)$, $\utilde{Q}_{h}^k(s,a), \forall (s,a)\in \Sc\times \Ac$:
        \begin{align*}
   & \scriptstyle \tilde{Q}^{k}_h(s,a)\triangleq  \min\bigg(H, r_h(s,a)+3\sqrt{\Var_{\hP_h^k}(\tilde{V}_{h+1}^{k})(s,a)\frac{\beta^\star}{n^k_h(s,a)}} + 14 H^2\frac{\beta}{n^k_h(s,a)}+\frac{1}{H}\hP_h^k(\tilde{V}^{k}_{h+1}-\utilde{V}^{k}_{h+1})(s,a)+\hP_h^k\tilde{V}^{k}_{h+1}(s,a)\bigg)\\
  & \scriptstyle \utilde{Q}^{k}_h(s,a)\triangleq  \max\bigg(0, r_h(s,a)-3\sqrt{\Var_{\hP_h^k}(\tilde{V}_{h+1}^{k})(s,a)\frac{\beta^\star}{n^k_h(s,a)}} - 22 H^2\frac{\beta}{n^k_h(s,a)}-\frac{2}{H}\hP_h^k(\tilde{V}^{k}_{h+1}-\utilde{V}^{k}_{h+1})(s,a)+\hP_h^k\utilde{V}^{k}_{h+1}(s,a)\bigg)   
\end{align*}
        \STATE $\tilde{V}^{k}_h(s)\gets\langle\tilde{Q}^{k}_h(s,\cdot),\pi_h(\cdot|s)\rangle$, $\utilde{V}^{k}_h(s)\gets \langle\utilde{Q}^{k}_h(s,\cdot),\pi_h(\cdot|s)\rangle,\forall s\in\Sc$.
    \ENDFOR
    \STATE {\bf Output:} $\utilde{V}_1$
\end{algorithmic}
\end{algorithm}

Based on the construction of $\pi^{k,h_0}$ in \Cref{alg:step_remake}, we define the following Q-value functions and value functions.
\begin{align}
    &\begin{aligned}
        \tilde{Q}^{k,h_0}_h(s,a)\triangleq & \min\Biggr(H, r_h(s,a)+3\sqrt{\Var_{\hP_h^k}(\tilde{V}_{h+1}^{k,h_0})(s,a)\frac{\beta^\star(n_h^k(s,a),\delta')}{n^k_h(s,a)}} + 14 H^2\frac{\beta(n_h^k(s,a),\delta')}{n^k_h(s,a)}\\
        &+\frac{1}{H}\hP_h^k(\tilde{V}^{k,h_0}_{h+1}-\utilde{V}^{k,h_0}_{h+1})(s,a)+\hP_h^k\tilde{V}^{k,h}_{h+1}(s,a)\Biggr)
    \end{aligned}\label{eqn:ucbQ}\\
    &\begin{aligned}
        \utilde{Q}^{k,h_0}_h(s,a)\triangleq & \max\Biggr(0, r_h(s,a)-3\sqrt{\Var_{\hP_h^k}(\tilde{V}_{h+1}^{k,h_0})(s,a)\frac{\beta^\star(n_h^k(s,a),\delta')}{n^k_h(s,a)}} - 22 H^2\frac{\beta(n_h^k(s,a),\delta')}{n^k_h(s,a)} \\
        &-\frac{2}{H}\hP_h^k(\tilde{V}^{k,h_0}_{h+1}-\utilde{V}^{k,h_0}_{h+1})(s,a)+\hP_h^k\utilde{V}^{k,h_0}_{h+1}(s,a)\Biggr)
    \end{aligned}\label{eqn:lcbQ}\\
    &\tilde{V}^{k,h_0}_h(s)\triangleq  \langle \pi_h^{k,h_0}(\cdot|s) ,\tilde{Q}^{k,h_0}_h(s,\cdot) \rangle \label{eqn:ucbV}\\
    &\utilde{V}^{k,h_0}_h(s) \triangleq  \langle \pi_h^{k,h_0}(\cdot|s) ,\utilde{Q}^{k,h_0}_h(s,\cdot) \rangle .\label{eqn:lcbV}
\end{align}
We point out that, since the definitions of \Cref{eqn:ucbQ,eqn:lcbQ} are the same as \Cref{eqn:q} used in \Cref{alg:step_remake},  $\utilde{V}^{k,h_0}$ defined in \Cref{eqn:lcbV} is consistent with the same quantity used in \Cref{alg:step_remake}. Moreover, when $h>h_0$,  we have $\pi_h^{k,h_0}=\gp_h^k$, which implies that $\tilde{Q}_h^{k,h_0}=\tilde{Q}_h^{k}$ and $\pi_h^{k,h_0}(s)=\argmax_a Q^{k,h_0}_h(s,a)$.

Before proceeding to the formal proof, we outline its major steps as follows.

{\bf Step one:} In \Cref{appx:steplemmas}, we verify that the good events happen with high probability and introduce the linearity of the ocupancy measure and the value function of the step-mix policies.

{\bf Step two:} In \Cref{appx:stepConfidence}, we prove that  $\tilde{Q}_h^{k,h_0}$ and  $\utilde{Q}_h^{k,h_0}$ are valid UCB and LCB of the true Q-value functions of both the step-wise optimal policies $\pi^{\star,h_0}$ and the step-wise optimistic policies $\pi^{k,h_0}$, respectively, and provide a bound for the gap between $\tilde{Q}_h^{k,h_0}$ and $\utilde{Q}_h^{k,h_0}$.

{\bf Step three:} In \Cref{appx:stepfinite}, we leverage the bound for the gap between $\tilde{Q}_h^{k,h_0}$ and $\utilde{Q}_h^{k,h_0}$ to  show a sublinear ``weak'' regret of the online policies $\pi^k$, where the regret is defined in terms of the performance difference $V^{\pi^b} - V^{\pi^k}$.  Hence, we can prove that there are only finite episodes in which the executed policy is not equal to the optimistic policy (finite non-optimistic policy lemma; cf. \Cref{lem:finite}).

{\bf Step four:} In \Cref{appx:stepthm}, based on the bound of the gap between $\tilde{Q}_h^{k,h_0}$ and $\utilde{Q}_h^{k,h_0}$ and the finite non-optimistic policy lemma, we prove the regret stated in \Cref{thm:step}.

\subsection{Step One: Good Events and Basic Properties of Step Mixture Policies}
\label{appx:steplemmas}
We first prove the following lemma which shows that the good events  defined in \cref{appx:notation} occur with high probability.
\begin{lemma}[Good Events]\label{lem:stepgood} Let $\mathcal{E},\mathcal{E}^{\text{cnt}}$, and $\mathcal{E}^{\star}$ be the events defined in \Cref{appx:notation}. Then, under \Cref{alg:step}, with probability at least $1 - \delta$, the following good events occur simultaneously:
    \[\mathcal{E}\bigg(\frac{\delta}{3}\bigg), \mathcal{E}^{\text{cnt}}\bigg(\frac{\delta}{3} \bigg),\mathcal{E}^\star\bigg(V^{\star,h_0},\frac{\delta}{3(H+1)} \bigg),\forall h_0\in[H]\cup\{0\}.\]
\end{lemma}
\begin{proof}
From \Cref{thm:goodevent1}, \Cref{thm:goodevent2} and \Cref{thm:goodevent3}, we have $\mathcal{E}(\frac{\delta}{3})$ , $\mathcal{E}^\text{cnt}(\frac{\delta}{3})$ and  $\underset{\small h_0\in [H]\cup \{0\}}{\cap} \mathcal{E}^\star(V^{\star,h_0},\frac{\delta}{3(H+1)})$ occur with probability at least $1 - \delta/3$, respectively. 
    Then, by taking a union bound, all those good events occur simultaneously with probability at least $1 - \delta$.
\end{proof}

Then, we provide several useful lemmas that capture the favorable properties of the step mixture policies $\pi^{k,h_0}$ and step-wise optimal policies $\pi^{*,h_0}$. 

The following lemma establishes the optimality of policy $\pi^{\star,h_0}$.
\begin{lemma}[Optimality of Step-wise Optimal Policies]
    \label{lem:stepoptimal}
    Define $\Pi_{h_0}:=\{\pi|\pi_h=\pi_h^b,\forall h\le h_0\}$. Then, 
    for any $\pi\in\Pi_{h_0}$, we must have
    \[Q_h^{\star, h_0}(s,a)\ge Q_h^{\pi}(s,a),\]
    \[V_h^{\star, h_0}(s)\ge V_h^{\pi}(s).\]
\end{lemma}
\begin{proof}
    Using the performance difference lemma  \citep{kakade2002approximately}, for any $\pi\in\Pi_{h_0}$ we have
    \[V_h^{\pi}(s) - V_h^{\star,h_0}(s) = \sum_{m=h}^H\mathbb{E}_{\pi}[ Q^{\star,h_0}_m(s_m,a_m)-V^{\star,h_0}_m(s_m)|s_h=s].\]
    
  For the case where $h>h_0$, $\pi^{\star,h_0}_{h}=\pi^\star_{h}$, and thus   $Q_h^{\star,h_0}(s,a)=Q^\star_h(s,a),\forall h\ge h_0,\forall s,a$. In addition, $\pi^\star(s)=\arg\max_a Q^\star(s,a)$. Thus, $\mathbb{E}_{\pi}[Q^{\star,h_0}_h(s_h,a_h)-V^{\star,h_0}(s_h)]\le 0, \forall h\ge h_0$. 
    
    For the case where $h\le h_0$, we have $\pi_h=\pi^{\star,h_0}_h=\pi_h^b$, $\mathbb{E}_{\pi}[Q^{\star,h_0}_h(s_h,a_h)-V^{\star,h_0}(s_h)]= 0, \forall h < h_0$.
    
    Combining the results for both cases, we have
    \[V_h^{\pi}(s) - V_h^{\star,h_0}(s) = \sum_{m=h}^H\mathbb{E}_{\pi}[ Q^{\star,h_0}_m(s_m,a_m)-V^{\star,h_0}_m(s_m)|s_h=s]\le 0.\]
    
    Following the same argument, we can prove that $Q_h^{\star, h_0}(s,a)\ge Q_h^{\pi}(s,a)$.
\end{proof}

The next lemma characterizes the property of the step mixture policies obtained by mixing two policies that are one-step different.
\begin{lemma}
    \label{lem:stepvisitratio}
    If $\pi=\rho\pi^1+\rho\pi^2$, where $\pi^1$ and $\pi^2$ differ only at step $h_0$,
    and let $d_h^1(s,a)$ and $d_h^2(s,a)$ be the {occupancy measure} of $\pi^1$ and $\pi^2$. Then we have \[d_h^{\pi}(s,a)=\rho d_h^1(s,a) + (1-\rho) d_h^2(s,a),\]
    where $d_h^{\pi}(s,a)=\mathbb{E}_{\pi}[\mathds{1}\{s_h=s,a_h=a\}]$ is the {occupancy measure} under policy $\pi$.
\end{lemma}
\begin{proof}
    Based on the definition of $\pi$, we consider the following possible cases. 
    
    When $h< h_0$, we have $\pi_h^1=\pi_h^2=\pi_h$. Thus the corresponding occupancy measures should also be the same, i.e.,
    \[d_h^1(s,a)=d_h^2(s,a)=d_h^\pi(s,a)=\rho d_h^1(s,a)+(1-\rho)d_h^2(s,a).\]
    
    When $h=h_0$, we have $\pi_{h_0}=\rho\pi_{h_0}^1 + (1 - \rho)\pi_{h_0}^2$. Using the fact that $d_{h_0-1}^1(s,a)=d_{h_0-1}^2(s,a)=d_{h_0-1}^\pi(s,a)$, we have
    \begin{align*}
        d_{h_0}^\pi(s,a) = & \sum_{s',a'}\pi_{h_0}(a|s) P_{h_0-1}(s|s',a') d_{h_0-1}^\pi(s',a') \\
        = & \sum_{s',a'}(\rho\pi_{h_0}^1(a|s)+(1-\rho)\pi_{h_0}^2(a|s)) P_{h_0-1}(s|s',a') d_{h_0-1}^\pi(s',a') \\
        = & \rho\sum_{s',a'}\pi_{h_0}^1(a|s) P_{h_0-1}(s|s',a') d_{h_0-1}^{1}(s',a')+(1-\rho)\sum_{s',a'}\pi_{h_0}^2(a|s) P_{h_0-1}(s|s',a') d_{h_0-1}^{2}(s',a') \\
        = & \rho d_{h_0}^1(s,a) + (1-\rho) d_{h_0}^2(s,a).
    \end{align*}
    
    When $h>h_0$, we again have $\pi_h^1=\pi_h^2=\pi_h$. We then prove the equality through induction. Assume $d_{h-1}^\pi(s,a)=\rho d_{h-1}^1(s,a)+(1-\rho) d_{h-1}^2(s,a),\forall h-1\ge h_0$, which holds when $h-1=h_0$ based on the analysis above. Then, 
    \begin{align*}
        d_{h}(s,a) = & \sum_{s',a'}\pi_{h}(a|s) P_{h-1}(s|s',a') d_{h-1}^\pi(s',a') \\
        = & \sum_{s',a'}\pi_{h}(a|s) P_{h-1}(s|s',a') (\rho d_{h-1}^1(s,a)+(1-\rho) d_{h-1}^2(s,a)) \\
        = & \rho\sum_{s',a'}\pi_{h}^1(a|s) P_{h-1}(s|s',a')d_{h-1}^1(s,a) + (1-\rho)\sum_{s',a'}\pi_{h}^2(a|s) P_{h-1}(s|s',a')d_{h-1}^2(s,a) \\
        = & \rho d_{h}^1(s,a)+(1-\rho) d_{h}^2(s,a),
    \end{align*}
    which completes the proof.
\end{proof}
The above lemma shows that the {occupancy measure} of a step mixture policy obtained by mixing two policies that are one-step different is a linear combination of the {occupancy measures} of the corresponding policies. Such linearity also holds for the corresponding value functions, as shown in the following proposition.
\begin{proposition}
    \label{lem:stepcomb}
    With the same condition as in \Cref{lem:stepvisitratio}, the following equality holds:
    \[V_1^{\pi}=\rho V_1^{\pi^1} + (1-\rho) V_1^{\pi^2}.\]
\end{proposition}

\begin{proof}
    By the definition of $V_1^\pi$ and $d_h^\pi(s,a)$, we have $V_1^\pi=\sum_{h=1}^H\sum_{s,a}d_h^\pi(s,a) r_h(s,a)$. Hence,
    \begin{align*}
        V_1^{\pi} =& \sum_{h=1}^H \sum_{s,a} d_h^{\pi}(s,a) r_h(s,a) \\
        =& \sum_{h=1}^H \sum_{s,a} \left(\rho d_h^{1}(s,a) + (1-\rho) d_h^{2}(s,a)\right) r_h(s,a) \\
        =& \rho V_1^{\pi^1} + (1-\rho) V_1^{\pi^2}.
    \end{align*}
\end{proof}

\subsection{Step Two: Confidence Bounds}
\label{appx:stepConfidence}

In this step, we validate the UCBs and LCBs constructed in \Cref{eqn:ucbQ,eqn:lcbQ,eqn:ucbV,eqn:lcbV}, and provide bounds for the estimation error induced by the pairs of UCBs and LCBs.

Recall that the upper confidence bounds of value functions of each policy $\pi^{k,h_0}$ are
\begin{equation}
\left\{
\begin{aligned}
        &\tilde{Q}^{k,h_0}_h(s,a)\triangleq  \min\Biggr(H, r_h(s,a)+3\sqrt{\Var_{\hP_h^k}(\tilde{V}_{h+1}^{k,h_0})(s,a)\frac{\beta^\star(n^k_h(s,a),\delta')}{n^k_h(s,a)}}+ 14 H^2\frac{\beta(n^k_h(s,a),\delta')}{n^k_h(s,a)}\\
        &\quad \quad +\frac{1}{H}\hP_h^k(\tilde{V}^{k,h_0}_{h+1}-\utilde{V}^{k,h_0}_{h+1})(s,a)+\hP_h^k\tilde{V}^{k,h_0}_{h+1}(s,a) \Biggr),\\
        &\tilde{V}^{k,h_0}_h(s)\triangleq  \langle \pi_h^{k,h_0}(\cdot|s), \tilde{Q}^{k,h_0}_h(s,\cdot) \rangle ,\label{Xb}
\end{aligned}
\right.
\end{equation}
where $\delta'=\frac{\delta}{3(H+1)}$ and 
\begin{align}
    &\pi_h^{k,h_0}(s)=  \left\{
    \begin{array}{lc}
     \pi^b_h(s), &\text{ if }h\le h_0,\\ 
 \arg\max_{a\in\mathcal{A}}\tilde{Q}_h^{k,h_0}(s,a), &\text{ if } h > h_0.
     \end{array}
    \right.\label{Xe}
\end{align}
Meanwhile, the lower confidence bounds of the the same value functions are
\begin{equation}
    \left\{
    \begin{aligned}
            &\utilde{Q}^{k,h_0}_h(s,a)\triangleq \max\Biggr(0, r_h(s,a)-3\sqrt{\Var_{\hP_h^k}(\tilde{V}_{h+1}^{k,h_0})(s,a)\frac{\beta^\star(n^k_h(s,a),\delta')}{n^k_h(s,a)}}- 22 H^2\frac{\beta(n^k_h(s,a),\delta')}{n^k_h(s,a)}\\
            &\quad \quad -\frac{2}{H}\hP_h^k(\tilde{V}^{k,h_0}_{h+1}-\utilde{V}^{k,h_0}_{h+1})(s,a)+\hP_h^k\utilde{V}^{k,h_0}_{h+1}(s,a) \Biggr),\\
            &\utilde{V}^{k,h_0}_h(s)\triangleq  \langle \pi_h^{k,h_0}(\cdot|s), \tilde{Q}^{k,h_0}_h(s,\cdot) \rangle\label{Xd}.
    \end{aligned}
    \right.
\end{equation}

The following lemma  shows that the above construction are valid UCBs and LCBs.


\begin{lemma}[UCB and LCB]
    \label{lem:stepucblcb}
    With $\tilde{Q}^{k,h_0}$, $\utilde{Q}^{k,h_0}$, $\tilde{V}^{k,h_0}$, $\utilde{V}^{k,h_0}$ defined in \Cref{Xb,Xd}, the true value functions $Q^{k,h_0},V^{k,h_0},Q^{\star,h_0},V^{\star,h_0}$ can be bounded as:
    \begin{align}
        \utilde{Q}_h^{k,h_0}(s,a) \overset{(\romannumeral1)}\le Q_h^{k,h_0}(s,a)&\overset{(\romannumeral2)}\le Q_h^{\star,h_0}(s,a) \overset{(\romannumeral3)}\le \tilde{Q}_h^{k,h_0}(s,a)  \label{eqn: UCB-LCB-Q},\\
        \utilde{V}_h^{k,h_0}(s) \overset{(\romannumeral4)}\le V_h^{k,h_0}(s)&\overset{(\romannumeral5)}\le V_h^{\star,h_0}(s) \overset{(\romannumeral6)}\le \tilde{V}_h^{k,h_0}(s) \label{eqn: UCB-LCB-V}.
    \end{align}
\end{lemma}

\if{0}
\begin{proof}
    We use induction to prove the lemma. More specifically, we prove that if \Cref{eqn: UCB-LCB-V} holds for $h+1$, $\forall h\in[1:H]$, it must hold for $h$ as well. For the base case $h=H+1$, all value functions are zeros. Thus, \Cref{eqn: UCB-LCB-V} holds. 
    
    \textbf{First}, we prove that the third inequality of \Cref{eqn: UCB-LCB-Q} holds, i.e. $Q_h^{\star,h_0}(s,a)\le \tilde{Q}_h^{k,h_0}(s,a)$. It suffices to consider the case when $\tilde{Q}_h^{k,h_0}<H$. We have
    \begin{align}
         \tilde{Q}_h^{k,h_0}&-Q^{\star,h_0}_h \notag\\
        = & 3\sqrt{\Var_{\hP_h^k}(\tilde{V}_{h+1}^{k,h_0})(s,a)\frac{\beta^\star(n^k_h(s,a),\delta')}{n^k_h(s,a)}}+ 14 H^2\frac{\beta(n^k_h(s,a),\delta')}{n^k_h(s,a)}\notag\\
        &+\frac{1}{H}\hP_h^k(\tilde{V}^{k,h_0}_{h+1}-\utilde{V}^{k,h_0}_{h+1})(s,a)+\hP_h^k\tilde{V}^{k,h_0}_{h+1}(s,a)-P_h V^{\star,h_0}_{h+1}(s,a)\notag\\
        = & 3\sqrt{\Var_{\hP_h^k}(\tilde{V}_{h+1}^{k,h_0})(s,a)\frac{\beta^\star(n^k_h(s,a),\delta')}{n^k_h(s,a)}}+ 14 H^2\frac{\beta(n^k_h(s,a),\delta')}{n^k_h(s,a)}\notag\\
        &+\frac{1}{H}\hP_h^k(\tilde{V}^{k,h_0}_{h+1}-\utilde{V}^{k,h_0}_{h+1})(s,a)+\hP_h^k(\tilde{V}^{k,h_0}_{h+1}-V^{\star,h_0}_{h+1})(s,a)+(\hP_h^k-P_h) V^{\star,h_0}_{h+1}(s,a)\label{eqn:UCB1}.
    \end{align}
    Under good event $\mathcal{E}^\star(V^{\star,h_0},\delta')$, we can bound the last term $(\hP_h^k-P_h) V^{\star,h_0}_{h+1}(s,a)$ as follows:
    \begin{equation}
        \label{eqn:UCB2}
        |(\hP_h^k-P_h) V^{\star,h_0}_{h+1}(s,a)|\\
        \le \sqrt{2\Var_{P_h}(V^{\star,h_0}_{h+1})(s,a)\frac{\beta^\star(n^k_h(s,a), \delta')}{n^k_h(s,a)}}+3H\frac{\beta^\star(n^k_h(s,a), \delta')}{n^k_h(s,a)}.
    \end{equation}
  We further bound the true variance $\Var_{P_h}(V^{\star,h_0}_{h+1})(s,a)$ with the empirical variance $\Var_{\hP_h}(\tilde{V}^{k,h_0}_{h+1})(s,a)$ as follows.
    \begin{align*}
        \Var_{P_h}(V^{\star,h_0}_{h+1})&(s,a)\\
        \overset{(a)}{\le} & 2\Var_{\hP_h}(V^{\star,h_0}_{h+1})(s,a)+4H^2\frac{\beta(n^k_h(s,a), \delta')}{n^k_h(s,a)} \\
        \overset{(b)}{\le} & 4 \Var_{\hP_h}(\tilde{V}^{\star,h_0}_{h+1})(s,a) + 4H\hP_h^k|\tilde{V}_{h+1}^{k,h_0}-V^{\star,h_0}_{h+1}|(s,a)+4H^2\frac{\beta(n^k_h(s,a), \delta')}{n^k_h(s,a)}\\
        \overset{(c)}{\le} & 4 \Var_{\hP_h}(\tilde{V}^{\star,h_0}_{h+1})(s,a) + 4H\hP_h^k(\tilde{V}_{h+1}^{k,h_0}-\utilde{V}^{k,h_0}_{h+1})(s,a)+4H^2\frac{\beta(n^k_h(s,a), \delta')}{n^k_h(s,a)},\\
    \end{align*}
    where $(a)$ follows from \Cref{lem:klvar} and the definition of good event $\mathcal{E}(\delta/3)$, $(b)$ is due to \Cref{lem:varerr}, and $(c)$ is due to the induction hypothesis. \jingc{which induction? if you want to use induction to prove anything, you should clearly state what you want to prove and how to prove it.}\ldh{Add a part at the beginning of the proof.}
    
    Plugging the bound of variance to \Cref{eqn:UCB2} and applying the facts that $\sqrt{x+y}\le \sqrt{x}+\sqrt{y}$ , $\sqrt{xy}\le x+y$ and $\beta^\star<\beta$, we have
    \begin{equation}
        \label{eqn:varbound1}
        |(\hP_h^k-P_h) V^{\star,h_0}_{h+1}(s,a)|\\
        \le 3\sqrt{\Var_{\hP_h}(\tilde{V}^{k,h_0}_{h+1})(s,a)\frac{\beta^\star(n^k_h(s,a), \delta')}{n^k_h(s,a)}}+\frac{1}{H}\hP_h^k(\tilde{V}^{k,h_0}_{h+1}-\utilde{V}^{k,h_0}_{h+1})(s,a)+14H^2\frac{\beta(n^k_h(s,a), \delta')}{n^k_h(s,a)}.
    \end{equation}
    
    Now, plugging \Cref{eqn:varbound1} back to \Cref{eqn:UCB1}, we have
    \[\tilde{Q}_h^{k,h_0}-Q^{\star,h_0}_h \ge \hP_h^k(\tilde{V}^{k,h_0}_{h+1}-V^{\star,h_0}_{h+1})(s,a)\ge 0,\]
    where $\tilde{V}^{k,h_0}_{h+1}-V^{\star,h_0}_{h+1}\ge 0$ comes from the induction hypothesis.\jingc{again, not defined}\hrq{added in the first paragraph}
    
    \textbf{Second}, for the first inequality of \Cref{eqn: UCB-LCB-Q}, i.e. $\utilde{Q}_h^{k,h_0}(s,a)\le Q_h^{k,h_0}(s,a)$, it suffices to consider the case when $\utilde{Q}^{k,h_0}_h>0$. We have
    \begin{equation}
    \label{eqn:LCB0}
    \begin{aligned}
         Q^{k,h_0}_h&(s,a) - \utilde{Q}_h^{k,h_0}(s,a) \\
        = & (P_h-\hP_h^k)V_{h+1}^{\star,h_0}(s,a) + (P_h-\hP_h^k)(\utilde{V}_{h+1}^{k,h_0}-V_{h+1}^{\star,h_0})(s,a) + P_h(V^{k,h_0}_{h+1}-\utilde{V}^{k,h_0}_{h+1})(s,a) \\
        & +
        3\sqrt{\Var_{\hP_h^k}(\tilde{V}_{h+1}^{k,h_0})(s,a)\frac{\beta^\star(n^k_h(s,a),\delta')}{n^k_h(s,a)}} 
        + 22 H^2\frac{\beta(n^k_h(s,a),\delta')}{n^k_h(s,a)} \\ 
        & +\frac{2}{H}\hP_h^k(\tilde{V}^{k,h_0}_{h+1}-\utilde{V}^{k,h_0}_{h+1})(s,a).
    \end{aligned}
    \end{equation}
    We have established a bound for $|(P_h-\hP_h^k)V_{h+1}^{\star,h_0}(s,a)|$ in \Cref{eqn:varbound1}. It then suffice to bound $|(P_h-\hP_h^k)(\utilde{V}_{h+1}^{k,h_0}-V_{h+1}^{\star,h_0})(s,a)|$ as follows.
    
    Because of \cref{lem:klabs}, together with good event $\mathcal{E}(\delta/3)$, we have
    \begin{equation}
        \label{eqn:LCB1}
        |(P_h-\hP_h^k)(\utilde{V}_{h+1}^{k,h_0}-V_{h+1}^{\star,h_0})(s,a)|\le\sqrt{2\Var_{P_h}(V_{h+1}^{\star,h_0}-\utilde{V}_{h+1}^{k,h_0})(s,a)}+\frac{2}{3}H\frac{\beta(n^k_h(s,a),\delta')}{n^k_h(s,a)}.
    \end{equation}
    Moreover, by \cref{lem:klvar}, 
    \begin{equation}
        \label{eqn:LCB1_2}
        \Var_{P_h}(V_{h+1}^{\star,h_0}-\utilde{V}_{h+1}^{k,h_0})(s,a)
        \le 2\Var_{\hP_h^k}(V_{h+1}^{\star,h_0}-\utilde{V}_{h+1}^{k,h_0})(s,a) + 4H^2\frac{\beta(n^k_h(s,a),\delta')}{n^k_h(s,a)} .
    \end{equation}
    Plugging \Cref{eqn:LCB1_2} to \Cref{eqn:LCB1} and applying $\sqrt{x+y}\le \sqrt{x} + \sqrt{y}$ and $\sqrt{xy}\le x+y$, we can bound $|(P_h-\hP_h^k)(\utilde{V}_{h+1}^{k,h_0}-V_{h+1}^{\star,h_0})(s,a)|$ as follows:
    \begin{equation}
        \label{eqn:LCB2}
        |(P_h-\hP_h^k)(\utilde{V}_{h+1}^{k,h_0}-V_{h+1}^{\star,h_0})(s,a)|\le \frac{1}{H}\hP_h^k(\tilde{V}^{k,h_0}_{h+1}-\utilde{V}^{k,h_0}_{h+1})(s,a)+8 H^2\frac{\beta(n^k_h(s,a), \delta')}{n^k_h(s,a)} .
    \end{equation}
    Now, plugging \Cref{eqn:LCB2} and \Cref{eqn:varbound1} back to \Cref{eqn:LCB0}, we have
    \[Q^{k,h_0}_h(s,a) - \utilde{Q}_h^{k,h_0}(s,a)\ge P_h(V^{k,h_0}_{h+1}-\utilde{V}^{k,h_0}_{h+1})(s,a)\ge 0 ,\]
    where $V^{k,h_0}_{h+1}-\utilde{V}^{k,h_0}_{h+1}\ge 0$ comes from the induction hypothesis.
    
  \textbf{Third}, due to \Cref{lem:stepoptimal}, the second inequalities of \Cref{eqn: UCB-LCB-Q} and \Cref{eqn: UCB-LCB-V} hold. In other words, we have $Q_h^{k,h_0}(s,a)\le Q_h^{\star,h_0}(s,a)$ and $V_h^{k,h_0}(s)\le V_h^{\star,h_0}(s)$. 
    
  \textbf{Finally}, with all inequalities in \Cref{eqn: UCB-LCB-Q} being proved, we use it to prove the first and the third inequalities of \Cref{eqn: UCB-LCB-V}. Since $\utilde{V}_h^{k,h_0}(s)$ and $V_h^{k,h_0}(s)$ share the same policy $\pi_h^{k,h_0}$, $\utilde{V}_h^{k,h_0}(s)\le V_h^{k,h_0}(s)$ can be derived from $\utilde{Q}_h^{k,h_0}(s,a)\le Q_h^{k,h_0}(s,a)$. That is
    \[\utilde{V}_h^{k,h_0}(s)=\langle \pi_h^{k,h_0}(\cdot|s),\utilde{Q}_h^{k,h_0}(s,\cdot)\rangle\le\langle \pi_h^{k,h_0}(\cdot|s),Q_h^{k,h_0}(s,\cdot)\rangle = V_h^{k,h_0}(s).\] 
    To show $V_h^{\star,h_0}(s)\le \tilde{V}_h^{k,h_0}(s)$, we consider two cases: $h>h_0$ and $h\le h_0$. When $h>h_0$, policy $\pi^{k,h_0}_h$ is the greedy policy corresponding to $\tilde{Q}_h^{k,h_0}$. Therefore, we have 
    \[\tilde{V}_h^{k,h_0}(s)=\tilde{Q}_h^{k,h_0}(s,\pi^{k,h_0}_h(s))\ge\tilde{Q}_h^{k,h_0}(s,\pi^\star_h(s))\ge Q_h^{k,h_0}(s,\pi^\star_h(s))=V_h^{k,h_0}(s).\] 
    When $h\le h_0$, both policies $\pi^{\star,h_0}_h$ and $\pi^{k,h_0}_h$ are the baseline policy $\pi^b_h$. Thus, we can use $Q_h^{\star,h_0}\le\tilde{Q}_h^{k,h_0}$ from \Cref{eqn: UCB-LCB-Q} to derive $V_h^{\star,h_0}(s)\le \tilde{V}_h^{k,h_0}(s)$. That is,
    \[V_h^{\star,h_0}(s)=\langle \pi_h^b(\cdot|s),Q_h^{\star,h_0}(s,\cdot)\rangle\le\langle \pi_h^b(\cdot|s),\tilde{Q}_h^{k,h_0}(s,\cdot)\rangle = \tilde{V}_h^{k,h_0}(s).\]
    Combining these two cases, we have established $V_h^{\star,h_0}(s)\le \tilde{V}_h^{k,h_0}(s)$  for any $h$. 
\end{proof}
\fi

\begin{proof}
 First, we note that due to \Cref{lem:stepoptimal}, inequalities $(\romannumeral2)$ and $(\romannumeral5)$ hold for any $h\in[1:H]$. 
 
 We then use induction to prove the other four inequalities hold. More specifically, we prove that: \textbf{1)} if $(\romannumeral4)$ and $(\romannumeral6)$ hold for $h+1$, $\forall h\in[1:H]$, then $(\romannumeral1)$ and $(\romannumeral3)$ must hold for $h$, and \textbf{2)} if $(\romannumeral1)$ and $(\romannumeral3)$ hold for any $h\in[1:H]$, then $(\romannumeral4)$ and $(\romannumeral6)$ must hold for $h$ as well. For the base case $h=H+1$, all value functions are zeros. Thus, $(\romannumeral4)$ and $(\romannumeral6)$ hold for $h=H+1$. We now assume $(\romannumeral4)$ and $(\romannumeral6)$ is true for any $h+1$, and prove 1) and 2) recursively through induction.
    
    \textbf{Step 1), part 1, inequality $(\romannumeral3)$:} We prove that inequality $(\romannumeral3)$ in \Cref{eqn: UCB-LCB-Q} holds for any $h\in[1:H]$. It suffices to consider the case when $\tilde{Q}_h^{k,h_0}<H$. We have
    \begin{align}
         \tilde{Q}_h^{k,h_0}-Q^{\star,h_0}_h  = & 3\sqrt{\Var_{\hP_h^k}(\tilde{V}_{h+1}^{k,h_0})(s,a)\frac{\beta^\star(n^k_h(s,a),\delta')}{n^k_h(s,a)}}+ 14 H^2\frac{\beta(n^k_h(s,a),\delta')}{n^k_h(s,a)}\notag\\
        &+\frac{1}{H}\hP_h^k(\tilde{V}^{k,h_0}_{h+1}-\utilde{V}^{k,h_0}_{h+1})(s,a)+\hP_h^k\tilde{V}^{k,h_0}_{h+1}(s,a)-P_h V^{\star,h_0}_{h+1}(s,a)\notag\\
        = & 3\sqrt{\Var_{\hP_h^k}(\tilde{V}_{h+1}^{k,h_0})(s,a)\frac{\beta^\star(n^k_h(s,a),\delta')}{n^k_h(s,a)}}+ 14 H^2\frac{\beta(n^k_h(s,a),\delta')}{n^k_h(s,a)}\notag\\
        &+\frac{1}{H}\hP_h^k(\tilde{V}^{k,h_0}_{h+1}-\utilde{V}^{k,h_0}_{h+1})(s,a)+\hP_h^k(\tilde{V}^{k,h_0}_{h+1}-V^{\star,h_0}_{h+1})(s,a)+(\hP_h^k-P_h) V^{\star,h_0}_{h+1}(s,a)\label{eqn:UCB1}.
    \end{align}
    Under good event $\mathcal{E}^\star(V^{\star,h_0},\delta')$, we can bound the last term $(\hP_h^k-P_h) V^{\star,h_0}_{h+1}(s,a)$ as follows:
    \begin{equation}
        \label{eqn:UCB2}
        |(\hP_h^k-P_h) V^{\star,h_0}_{h+1}(s,a)|\\
        \le \sqrt{2\Var_{P_h}(V^{\star,h_0}_{h+1})(s,a)\frac{\beta^\star(n^k_h(s,a), \delta')}{n^k_h(s,a)}}+3H\frac{\beta^\star(n^k_h(s,a), \delta')}{n^k_h(s,a)}.
    \end{equation}
  We further bound the true variance $\Var_{P_h}(V^{\star,h_0}_{h+1})(s,a)$ with the empirical variance $\Var_{\hP_h}(\tilde{V}^{k,h_0}_{h+1})(s,a)$ as follows:
    \begin{align*}
        \Var_{P_h}(V^{\star,h_0}_{h+1})(s,a)      \overset{(a)}{\le} & 2\Var_{\hP_h}(V^{\star,h_0}_{h+1})(s,a)+4H^2\frac{\beta(n^k_h(s,a), \delta')}{n^k_h(s,a)} \\
        \overset{(b)}{\le} & 4 \Var_{\hP_h}(\tilde{V}^{\star,h_0}_{h+1})(s,a) + 4H\hP_h^k|\tilde{V}_{h+1}^{k,h_0}-V^{\star,h_0}_{h+1}|(s,a)+4H^2\frac{\beta(n^k_h(s,a), \delta')}{n^k_h(s,a)}\\
        \overset{(c)}{\le} & 4 \Var_{\hP_h}(\tilde{V}^{\star,h_0}_{h+1})(s,a) + 4H\hP_h^k(\tilde{V}_{h+1}^{k,h_0}-\utilde{V}^{k,h_0}_{h+1})(s,a)+4H^2\frac{\beta(n^k_h(s,a), \delta')}{n^k_h(s,a)},
    \end{align*}
    where $(a)$ follows from \Cref{lem:klvar} and the definition of good event $\mathcal{E}(\delta/3)$, $(b)$ is due to \Cref{lem:varerr}, and $(c)$ is due to the induction hypothesis. 
    
    Plugging the bound of variance to \Cref{eqn:UCB2} and applying the facts that $\sqrt{x+y}\le \sqrt{x}+\sqrt{y}$ , $\sqrt{xy}\le x+y$ and $\beta^\star<\beta$, we have
    \begin{equation}
        \label{eqn:varbound1}
        |(\hP_h^k-P_h) V^{\star,h_0}_{h+1}(s,a)|\\
        \le 3\sqrt{\Var_{\hP_h}(\tilde{V}^{k,h_0}_{h+1})(s,a)\frac{\beta^\star(n^k_h(s,a), \delta')}{n^k_h(s,a)}}+\frac{1}{H}\hP_h^k(\tilde{V}^{k,h_0}_{h+1}-\utilde{V}^{k,h_0}_{h+1})(s,a)+14H^2\frac{\beta(n^k_h(s,a), \delta')}{n^k_h(s,a)}.
    \end{equation}
    
    Now, plugging \Cref{eqn:varbound1} back to \Cref{eqn:UCB1}, we have
    \[\tilde{Q}_h^{k,h_0}-Q^{\star,h_0}_h \ge \hP_h^k(\tilde{V}^{k,h_0}_{h+1}-V^{\star,h_0}_{h+1})(s,a)\ge 0,\]
    where $\tilde{V}^{k,h_0}_{h+1}-V^{\star,h_0}_{h+1}\ge 0$ comes from the induction hypothesis.
    
    \textbf{Step 1), part 2, inequality $(\romannumeral1)$:} For inequality $(\romannumeral1)$ in \Cref{eqn: UCB-LCB-Q}, it suffices to consider the case when $\utilde{Q}^{k,h_0}_h>0$. We have
    \begin{equation}
    \label{eqn:LCB0}
    \begin{aligned}
         Q^{k,h_0}_h(s,a) - \utilde{Q}_h^{k,h_0}(s,a) = & (P_h-\hP_h^k)V_{h+1}^{\star,h_0}(s,a) + (P_h-\hP_h^k)(\utilde{V}_{h+1}^{k,h_0}-V_{h+1}^{\star,h_0})(s,a) + P_h(V^{k,h_0}_{h+1}-\utilde{V}^{k,h_0}_{h+1})(s,a) \\
        & +
        3\sqrt{\Var_{\hP_h^k}(\tilde{V}_{h+1}^{k,h_0})(s,a)\frac{\beta^\star(n^k_h(s,a),\delta')}{n^k_h(s,a)}} 
        + 22 H^2\frac{\beta(n^k_h(s,a),\delta')}{n^k_h(s,a)} \\ 
        & +\frac{2}{H}\hP_h^k(\tilde{V}^{k,h_0}_{h+1}-\utilde{V}^{k,h_0}_{h+1})(s,a).
    \end{aligned}
    \end{equation}
    We have established a bound for $|(P_h-\hP_h^k)V_{h+1}^{\star,h_0}(s,a)|$ in \Cref{eqn:varbound1}. It then suffice to bound $|(P_h-\hP_h^k)(\utilde{V}_{h+1}^{k,h_0}-V_{h+1}^{\star,h_0})(s,a)|$ as follows.
    
    Because of \cref{lem:klabs}, together with good event $\mathcal{E}(\delta/3)$, we have
    \begin{equation}
        \label{eqn:LCB1}
        |(P_h-\hP_h^k)(\utilde{V}_{h+1}^{k,h_0}-V_{h+1}^{\star,h_0})(s,a)|\le\sqrt{2\Var_{P_h}(V_{h+1}^{\star,h_0}-\utilde{V}_{h+1}^{k,h_0})(s,a)}+\frac{2}{3}H\frac{\beta(n^k_h(s,a),\delta')}{n^k_h(s,a)}.
    \end{equation}
    Moreover, by \cref{lem:klvar}, 
    \begin{equation}
        \label{eqn:LCB1_2}
        \Var_{P_h}(V_{h+1}^{\star,h_0}-\utilde{V}_{h+1}^{k,h_0})(s,a)
        \le 2\Var_{\hP_h^k}(V_{h+1}^{\star,h_0}-\utilde{V}_{h+1}^{k,h_0})(s,a) + 4H^2\frac{\beta(n^k_h(s,a),\delta')}{n^k_h(s,a)} .
    \end{equation}
    Plugging \Cref{eqn:LCB1_2} into \Cref{eqn:LCB1} and applying $\sqrt{x+y}\le \sqrt{x} + \sqrt{y}$ and $\sqrt{xy}\le x+y$, we can bound $|(P_h-\hP_h^k)(\utilde{V}_{h+1}^{k,h_0}-V_{h+1}^{\star,h_0})(s,a)|$ as follows:
    \begin{equation}
        \label{eqn:LCB2}
        |(P_h-\hP_h^k)(\utilde{V}_{h+1}^{k,h_0}-V_{h+1}^{\star,h_0})(s,a)|\le \frac{1}{H}\hP_h^k(\tilde{V}^{k,h_0}_{h+1}-\utilde{V}^{k,h_0}_{h+1})(s,a)+8 H^2\frac{\beta(n^k_h(s,a), \delta')}{n^k_h(s,a)} .
    \end{equation}
    Now, plugging \Cref{eqn:LCB2} and \Cref{eqn:varbound1} back to \Cref{eqn:LCB0}, we have
    \[Q^{k,h_0}_h(s,a) - \utilde{Q}_h^{k,h_0}(s,a)\ge P_h(V^{k,h_0}_{h+1}-\utilde{V}^{k,h_0}_{h+1})(s,a)\ge 0 ,\]
    where $V^{k,h_0}_{h+1}-\utilde{V}^{k,h_0}_{h+1}\ge 0$ comes from the induction hypothesis.

  \textbf{Step 2):} With all inequalities in \Cref{eqn: UCB-LCB-Q} being proved, we use them to prove inequalities $(\romannumeral4)$ and $(\romannumeral6)$ in \Cref{eqn: UCB-LCB-V}. 
  
  Since $\utilde{V}_h^{k,h_0}(s)$ and $V_h^{k,h_0}(s)$ share the same policy $\pi_h^{k,h_0}$, inequality $(\romannumeral4)$ can be derived from $\utilde{Q}_h^{k,h_0}(s,a)\le Q_h^{k,h_0}(s,a)$. That is,
    \[\utilde{V}_h^{k,h_0}(s)=\langle \pi_h^{k,h_0}(\cdot|s),\utilde{Q}_h^{k,h_0}(s,\cdot)\rangle\le\langle \pi_h^{k,h_0}(\cdot|s),Q_h^{k,h_0}(s,\cdot)\rangle = V_h^{k,h_0}(s).\] 
    
    To show inequality $(\romannumeral6)$, we consider two cases: $h>h_0$ and $h\le h_0$. When $h>h_0$, policy $\pi^{k,h_0}_h$ is the optimistic policy corresponding to $\tilde{Q}_h^{k,h_0}$. Therefore, we have 
    \[\tilde{V}_h^{k,h_0}(s)=\tilde{Q}_h^{k,h_0}(s,\pi^{k,h_0}_h(s))\ge\tilde{Q}_h^{k,h_0}(s,\pi^\star_h(s))\ge Q_h^{k,h_0}(s,\pi^\star_h(s))=V_h^{k,h_0}(s).\] 
    When $h\le h_0$, both policies $\pi^{\star,h_0}_h$ and $\pi^{k,h_0}_h$ are the baseline policy $\pi^b_h$. Thus, we can use $Q_h^{\star,h_0}\le\tilde{Q}_h^{k,h_0}$ from \Cref{eqn: UCB-LCB-Q} to derive $V_h^{\star,h_0}(s)\le \tilde{V}_h^{k,h_0}(s)$. That is,
    \[V_h^{\star,h_0}(s)=\langle \pi_h^b(\cdot|s),Q_h^{\star,h_0}(s,\cdot)\rangle\le\langle \pi_h^b(\cdot|s),\tilde{Q}_h^{k,h_0}(s,\cdot)\rangle = \tilde{V}_h^{k,h_0}(s).\]
    Combining these two cases, we have established $V_h^{\star,h_0}(s)\le \tilde{V}_h^{k,h_0}(s)$  for any $h$. 
\end{proof}

After verifying the validity of UCBs and LCBs in \Cref{lem:stepucblcb}, we provide the following lemma to quantify the estimation error induced by the lower bound. 

\begin{lemma}
    \label{lem:ringgap}
    Define $G_h^{k,h_0}$ as
    \begin{align}
        \label{eqn:G}
        &G_h^{k,h_0}(s,a) \\
        &\quad=\min\Biggr(H, 6\sqrt{\Var_{\hP_h^k}(\tilde{V}^{k,h_0}_{h+1})(s,a)\frac{\beta^\star(n^k_h(s,a),\delta')}{n^k_h(s,a)}}+36H^2\frac{\beta(n^k_h(s,a),\delta')}{n^k_h(s,a)}+\left(1+\frac{3}{H} \right)\hP_h^k\pi_{h+1}^{k,h_0}G_{h+1}^{k,h_0}(s,a)\Biggr).
    \end{align}
Then, the estimation error between $Q_h^{\star,h_0}$, $V_h^{\star,h_0}(s)$ and $\utilde{Q}_h^{k,h_0}$, $\utilde{V}_h^{k,h_0}$ can be bounded as
\begin{align*}
    Q_h^{\star,h_0}(s,a)-\utilde{Q}_h^{k,h_0}(s,a) &\le G_h^{k,h_0}(s, a), \\
    V_h^{\star,h_0}(s)-\utilde{V}_h^{k,h_0}(s) &\le \langle \pi_h^{k,h_0}(\cdot|s), G_h^{k,h_0}(s,\cdot) \rangle.
\end{align*}
\end{lemma}

\begin{proof}
    $Q_h^{\star,h_0}(s,a)-\utilde{Q}_h^{k,h_0}(s,a)$ can be directly calculated as follows.
    \begin{align*}
        Q_h^{\star,h_0}(s,a)-\utilde{Q}_h^{k,h_0}(s,a)
        \overset{(a)}{\le} & \tilde{Q}_h^{k,h_0}(s,a)-\utilde{Q}_h^{k,h_0}(s,a) \\
        \overset{(b)}{\le} & \hP_h^k (\tilde{V}_{h+1}^{k,h_0}-\utilde{V}_{h+1}^{k,h_0})(s,a)+6\sqrt{\Var_{\hP_h^k}(\tilde{V}_{h+1}^{k,h_0})(s,a)\frac{\beta^\star(n^k_h(s,a),\delta')}{n^k_h(s,a)}}\\
        & + 36 H^2\frac{\beta(n^k_h(s,a),\delta')}{n^k_h(s,a)} + \frac{3}{H}\hP_h^k(\tilde{V}^{k,h_0}_{h+1}-\utilde{V}^{k,h_0}_{h+1})(s,a) \\
        \le & 6\sqrt{\Var_{\hP_h^k}(\tilde{V}_{h+1}^{k,h_0})(s,a)\frac{\beta^\star(n^k_h(s,a),\delta')}{n^k_h(s,a)}} + 36 H^2\frac{\beta(n^k_h(s,a),\delta')}{n^k_h(s,a)}\\
        & + (1+\frac{3}{H})\hP_h^k(\tilde{V}^{k,h_0}_{h+1}-\utilde{V}^{k,h_0}_{h+1})(s,a),
    \end{align*}
    where (a) follows from \Cref{lem:stepucblcb} and (b) follows from the definitions of $\tilde{Q}_h^{k,h_0}(s,a)$ and $\utilde{Q}_h^{k,h_0}(s,a)$ in \Cref{Xb} and \Cref{Xd}, respectively. 
    
    
    Then, following the same argument, for $V$ functions, we have
    \[V_h^{\star,h_0}(s) - \utilde{V}_h^{k,h_0}(s) \le \tilde{V}_h^{k,h_0}(s) - \utilde{V}_h^{k,h_0}(s)
        \le  \langle \pi_h^{k,h_0}(\cdot|s), (\tilde{Q}_h^{k,h_0} - \utilde{Q}_h^{k,h_0})(s,\cdot)\rangle.\]
    
    Combining the above two inequalities with the definition of $G_h^{k,h_0}$, we have
    \begin{align*}
        Q_h^{\star,h_0}(s,a)-\utilde{Q}_h^{k,h_0}(s,a) &\le G_h^{k,h_0}(s,a), \\
        V_h^{\star,h_0}(s)-\utilde{V}_h^{k,h_0}(s) &\le \langle\pi_h^{k,h_0}(\cdot|s),G_h^{k,h_0}(s,\cdot)\rangle,
    \end{align*}
   which completes the proof.
\end{proof}

In the following lemma, we  aim to upper bound $\pi_{1}^{k,h_0} G_{1}^{k,h_0}$. 

\begin{lemma}[Upper bound $\pi_{1}^{k,h_0} G_{1}^{k,h_0}$]
    \label{lem:Gbound}
    For any $k$ and $h_0$, we have
    \begin{equation*}
    \begin{aligned}
        \pi_{1}^{k,h_0} G_{1}^{k,h_0}(s_1)
        \le & 24e^{13}\sum_{h=1}^H \sum_{s,a} d_h^{k,h_0}(s,a) \sqrt{\Var_{P_h}(V_{h+1}^{k,h_0})(s,a)\biggr(\frac{\beta^\star(\bar{n}^k_h(s,a),\delta')}{\bar{n}^k_h(s,a)\vee 1}\biggr)} \\
        & + 336 e^{13} H^2 \sum_{h=1}^H \sum_{s,a} d_h^{k,h_0}(s,a)\biggr(\frac{\beta(\bar{n}^k_h(s,a),\delta')}{\bar{n}^k_h(s,a)\vee 1}\biggr),
    \end{aligned}
    \end{equation*}
    where $G_h^{k,h_0}$ is defined in \Cref{eqn:G} and $d_h^{k,h_0}$ is the {occupancy measure} under policy $\pi^{k,h_0}$.
\end{lemma}

\begin{proof}
    From the definition of $G_h^{k,h_0}$ in \Cref{eqn:G}, we have
    \begin{equation}
    \label{eqn:Gbound1}
    \begin{aligned}
        & G_h^{k,h_0}(s,a) \\
        &\le  6\sqrt{\underbrace{\Var_{\hP_h^k}(\tilde{V}^{k,h_0}_{h+1})(s,a)}_{\text{(I)}}\frac{\beta^\star(n^k_h(s,a),\delta')}{n^k_h(s,a)}}+36H^2\frac{\beta(n^k_h(s,a),\delta')}{n^k_h(s,a)}+\left(1+\frac{3}{H}\right)\underbrace{\hP_h^k\pi_{h+1}^{k,h_0}G_{h+1}^{k,h_0}(s,a)}_{\text{(II)}}.
    \end{aligned}
    \end{equation}
    
   In order to bound term (II), we use \Cref{lem:klabs} and the fact that $\sqrt{xy}\le x + y$ to obtain
    \begin{equation}
    \label{eqn:Gbound1p}
    \begin{aligned}
        (\hP_h^k-P_h)\pi_{h+1}^k G_{h+1}^{k,h_0}(s,a)\le & \sqrt{2\Var_{P_h}(\pi_{h+1}^k G_{h+1}^{k,h_0})(s,a)\frac{\beta(n^k_h(s,a),\delta')}{n^k_h(s,a)}} + \frac{2}{3}H\frac{\beta(n^k_h(s,a),\delta')}{n^k_h(s,a)} \\
        \le & \frac{1}{H} P_h \pi_{h+1}^k G_{h+1}^{k,h_0}(s,a) + 3H^2\frac{\beta(n^k_h(s,a),\delta')}{n^k_h(s,a)}.
    \end{aligned}
    \end{equation}
    
    
 For term (I), we have
    \begin{align*}
        \Var_{\hP_h^k}(\tilde{V}^{k,h_0}_{h+1})(s,a)
        \overset{(a)}{\le} & 2\Var_{P_h}(\tilde{V}^{k,h_0}_{h+1})(s,a)+4H^2\frac{\beta(n^k_h(s,a), \delta')}{n^k_h(s,a)} \\
        \overset{(b)}{\le} & 4 \Var_{P_h}(V^{\star,h_0}_{h+1})(s,a) + 4H P_h|\tilde{V}_{h+1}^{k,h_0}-V^{\star,h_0}_{h+1}|(s,a)+4H^2\frac{\beta(n^k_h(s,a), \delta')}{n^k_h(s,a)}\\
        \overset{(c)}{\le} & 4 \Var_{\hP_h}(\tilde{V}^{\star,h_0}_{h+1})(s,a) + 4H\hP_h^k(\tilde{V}_{h+1}^{k,h_0}-\utilde{V}^{k,h_0}_{h+1})(s,a)+4H^2\frac{\beta(n^k_h(s,a), \delta')}{n^k_h(s,a)},
    \end{align*}
  where $(a)$ follows from \cref{lem:klvar}, $(b)$ follows from \cref{lem:varerr}, and $(c)$ follows from \Cref{lem:stepucblcb}. 
  
  Moreover, we can use $\sqrt{x+y}\le\sqrt{x}+\sqrt{y}$, $\sqrt{xy}\le x + y$ to obtain
    \begin{equation}
    \label{eqn:Gbound1V}
    \begin{aligned}
        & \sqrt{\Var_{\hP_h^k}(\tilde{V}^{k,h_0}_{h+1})(s,a)\frac{\beta^\star(n^k_h(s,a),\delta')}{n^k_h(s,a)}} \\
        & \le  2\sqrt{\Var_{P_h}(V^{k,h_0}_{h+1})(s,a)\frac{\beta^\star(n^k_h(s,a),\delta')}{n^k_h(s,a)}} + 6H^2\frac{\beta(n^k_h(s,a),\delta')}{n^k_h(s,a)} + \frac{1}{H}P_h\pi_{h+1}^{k,h_0}G_{h+1}^{k,h_0}(s,a). 
    \end{aligned}
    \end{equation}
   Applying \Cref{eqn:Gbound1p} and \Cref{eqn:Gbound1V} to \Cref{eqn:Gbound1}, we have
    \begin{align*}
         G_h^{k,h_0}&(s,a) \\
        \le & 6\sqrt{\Var_{\hP_h^k}(\tilde{V}^{k,h_0}_{h+1})(s,a)\frac{\beta^\star(n^k_h(s,a),\delta')}{n^k_h(s,a)}}+36H^2\frac{\beta(n^k_h(s,a),\delta')}{n^k_h(s,a)}+\left(1+\frac{3}{H}\right)\hP_h^k\pi_{h+1}^{k,h_0}G_{h+1}^{k,h_0}(s,a) \\
        \le & 12\sqrt{\Var_{P_h}(V^{k,h_0}_{h+1})(s,a)\frac{\beta^\star(n^k_h(s,a),\delta')}{n^k_h(s,a)}} + 36H^2\frac{\beta(n^k_h(s,a),\delta')}{n^k_h(s,a)} + \frac{6}{H}P_h\pi_{h+1}^{k,h_0}G_{h+1}^{k,h_0}(s,a) \\
        & + 36H^2\frac{\beta(n^k_h(s,a),\delta')}{n^k_h(s,a)} + \left(1+\frac{3}{H}\right) P_h\pi_{h+1}^{k,h_0}G_{h+1}^{k,h_0}(s,a) \\
        & + \left(1+\frac{3}{H}\right) \left(\frac{1}{H} P_h \pi_{h+1}^k G_{h+1}^{k,h_0}(s,a) + 3H^2\frac{\beta(n^k_h(s,a),\delta')}{n^k_h(s,a)}\right) \\
        \le & 12\sqrt{\Var_{P_h}(V^{k,h_0}_{h+1})(s,a)\frac{\beta^\star(n^k_h(s,a),\delta')}{n^k_h(s,a)}} + 84H^2\frac{\beta(n^k_h(s,a),\delta')}{n^k_h(s,a)} + \left(1+\frac{13}{H}\right) P_h \pi_{h+1}^k G_{h+1}^{k,h_0}(s,a).
    \end{align*}
    In addition, since $G_h^{k,h_0}$ is upper bounded by $H$ by definition, and $\beta(n^k_h(s,a),\delta')>\beta^\star(n^k_h(s,a),\delta')$, we have
    \begin{align*}
         G_h^{k,h_0}(s,a) &\le \min\Biggr\{H, 12\sqrt{\Var_{P_h}(V^{k,h_0}_{h+1})(s,a)\frac{\beta^\star(n^k_h(s,a),\delta')}{n^k_h(s,a)}} \\
        &\quad\quad+ 84H^2\frac{\beta(n^k_h(s,a),\delta')}{n^k_h(s,a)} + \left(1+\frac{13}{H}\right) P_h \pi_{h+1}^k G_{h+1}^{k,h_0}(s,a)\Biggr\} \\
        &\le 12\sqrt{\Var_{P_h}(V^{k,h_0}_{h+1})(s,a)\left(\frac{\beta^\star(n^k_h(s,a),\delta')}{n^k_h(s,a)}\wedge 1\right)} \\ 
        &\quad\quad+ 84H^2\left(\frac{\beta(n^k_h(s,a),\delta')}{n^k_h(s,a)}\wedge 1\right) + \left(1+\frac{13}{H}\right) P_h \pi_{h+1}^k G_{h+1}^{k,h_0}(s,a).
    \end{align*}
    Using $(1+\frac{13}{H})^H\le e^{13}$ and unfolding the above inequality, we have
    \begin{align*}
         \pi_{1}^{k,h_0} G_{1}^{k,h_0}(s_1) 
        \le & 12e^{13}\sum_{h=1}^H \sum_{s,a} d_h^{k,h_0}(s,a) \sqrt{\Var_{P_h}(V_{h+1}^{k,h_0})(s,a)\left(\frac{\beta^\star(n^k_h(s,a),\delta')}{n^k_h(s,a)}\wedge 1\right)} \\
        & + 84 e^{13} H^2 \sum_{h=1}^H \sum_{s,a} d_h^{k,h_0}(s,a)\left(\frac{\beta(n^k_h(s,a),\delta')}{n^k_h(s,a)}\wedge 1\right).
    \end{align*}
    
Finally, we use \Cref{lem:cnt} to transform $n^k_h(s,a)$ to $\bar{n}^k_h(s,a)$ and obtain
    \begin{align*}
        \pi_{1}^{k,h_0} G_{1}^{k,h_0}(s_1)
        \le & 24e^{13}\sum_{h=1}^H \sum_{s,a} d_h^{k,h_0}(s,a) \sqrt{\Var_{P_h}(V_{h+1}^{k,h_0})(s,a)\biggr(\frac{\beta^\star(\bar{n}^k_h(s,a),\delta')}{\bar{n}^k_h(s,a)\vee 1}\biggr)} \\
        & + 336 e^{13} H^2 \sum_{h=1}^H \sum_{s,a} d_h^{k,h_0}(s,a)\biggr(\frac{\beta(\bar{n}^k_h(s,a),\delta')}{\bar{n}^k_h(s,a)\vee 1}\biggr),
    \end{align*}
    which completes the proof.
\end{proof}

\subsection{Step Three: Finite Episodes for Step Mixture Policies}
\label{appx:stepfinite}
 In the previous section, we have characterized the UCBs and LCBs for the step-wise optimistic policies ($\pi^{k,h_0}$) and bound the corresponding estimation errors. In this step, we extend the result for $\pi^{k,h_0}$ to step mixture policies $\pi^k$ and prove that the number of the episodes of which the executed policy $\pi^k$ is not equal to the optimistic policy $\bar{\pi}^k$ (i.e., $\pi^{k,0}$), is finite under the StepMix algorithm. We refer to this result as the \emph{finite non-optimistic policy lemma}.

To establish the finite non-optimistic policy lemma, we first extend the result of \Cref{lem:Gbound} from $\pi^{k,h_0}$ to step mixture policies.

\begin{lemma}
    \label{lem:Gboundcomb}
    For a step mixture policy $\pi^k$ mixed from two policies $\pi^{k,h_0}$ and $\pi^{k,h_0-1}$, denoted as $\pi^k = (1-\rho)\pi^{k,h_0}+\rho \pi^{k,h_0-1}$ for some $\rho\in(0,1)$, the following inequality holds:
    \begin{equation}
    \label{eqn:combined_G_bound}
    \begin{aligned}
         (1-&\rho)  \pi_{1}^{k,h_0} G_{1}^{k,h_0}(s_1) + \rho \pi_{1}^{k,h_0-1} G_{1}^{k,h_0-1}(s_1) \\
        & \le 24e^{13} H \sqrt{\sum_{h=1}^H\sum_{s,a}d_h^k(s,a)\biggr(\frac{\beta(\bar{n}^k_h(s,a),\delta')}{\bar{n}^k_h(s,a)\vee 1}\biggr)} + 336 e^{13}H^2\sum_{h=1}^H\sum_{s,a}d_h^k(s,a)\biggr(\frac{\beta(\bar{n}^k_h(s,a),\delta')}{\bar{n}^k_h(s,a)\vee 1}\biggr),
    \end{aligned}
    \end{equation}
    where $d_h^k(s,a)$ is the occupancy measure under policy $\pi^k$ and $G_h^{k,h_0}$ is defined in \Cref{eqn:G}.
\end{lemma}

\begin{proof}
    Since $\pi^k = (1-\rho)\pi^{k,h_0}+\rho \pi^{k,h_0-1}$ is the step mixture policy mixed from two policies that differ at only one step, by \Cref{lem:stepvisitratio}, the occupancy measure under $\pi^k$ satisfies
    \[d_h^k(s,a)=(1-\rho)d^{k,h_0}(s,a)+\rho d^{k,h_0-1}(s,a).\]
    Using \Cref{lem:Gbound}, we have
    \begin{equation}
    \label{eqn:Gboundcomb0}
    \begin{aligned}
        (1-\rho)&\pi_{1}^{k,h_0} G_{1}^{k,h_0}(s_1) + \rho \pi_{1}^{k,h_0-1} G_{1}^{k,h_0-1}(s_1) \\
        \le &\rho\Biggr(24e^{13}\sum_{h=1}^H \sum_{s,a} d_h^{k,h_0-1}(s,a) \sqrt{\Var_{P_h}(V_{h+1}^{k,h_0-1})(s,a)\biggr(\frac{\beta^\star(\bar{n}^k_h(s,a),\delta')}{\bar{n}^k_h(s,a)\vee 1}\biggr)} \\
        &\quad\quad  + 336 e^{13} H^2 \sum_{h=1}^H \sum_{s,a} d_h^{k,h_0-1}(s,a) \left(  \frac{\beta(\bar{n}^k_h(s,a),\delta')}{\bar{n}^k_h(s,a)\vee 1} \right)\Biggr) \\
         & + (1-\rho)\Biggr(24e^{13}\sum_{h=1}^H \sum_{s,a} d_h^{k,h_0}(s,a) \sqrt{\Var_{P_h}(V_{h+1}^{k,h_0})(s,a) \left(\frac{\beta^\star(\bar{n}^k_h(s,a),\delta')}{\bar{n}^k_h(s,a)\vee 1} \right) } \\
         &\quad\quad + 336 e^{13} H^2 \sum_{h=1}^H \sum_{s,a} d_h^{k,h_0}(s,a) \left( \frac{\beta(\bar{n}^k_h(s,a),\delta')}{\bar{n}^k_h(s,a)\vee 1} \right)\Biggr).
    \end{aligned}
    \end{equation}

    It is worth noting that 
    \begin{equation}
        \label{eqn:Gboundcomb1}
        \begin{aligned}
        &\sum_{h=1}^H\sum_{s,a}d_h^k(s,a)\biggr(\frac{\beta(\bar{n}^k_h(s,a),\delta')}{\bar{n}^k_h(s,a)\vee 1}\biggr) \\
        &\quad =\rho\sum_{h=1}^H \sum_{s,a} d_h^{k,h_0-1}(s,a)\biggr(\frac{\beta(\bar{n}^k_h(s,a),\delta')}{\bar{n}^k_h(s,a)\vee 1}\biggr)+(1-\rho)\sum_{h=1}^H\sum_{s,a}d_h^{k,h_0}(s,a)\biggr(\frac{\beta(\bar{n}^k_h(s,a),\delta')}{\bar{n}^k_h(s,a)\vee 1}\biggr).
        \end{aligned}
    \end{equation}
    Thus, to prove \Cref{eqn:combined_G_bound}, it suffices to show that
\begin{align}
      & \sum_{h=1}^H \sum_{s,a}  \rho d_h^{k,h_0-1}(s,a) \sqrt{\Var_{P_h}(V_{h+1}^{k,h_0-1})(s,a)\biggr(\frac{\beta^\star(\bar{n}^k_h(s,a),\delta')}{\bar{n}^k_h(s,a)\vee 1}\biggr)} \nonumber\\
        &\quad\quad   + \sum_{h=1}^H \sum_{s,a}  (1-\rho) d_h^{k,h_0}(s,a) \sqrt{\Var_{P_h}(V_{h+1}^{k,h_0})(s,a)\biggr(\frac{\beta^\star(\bar{n}^k_h(s,a),\delta')}{\bar{n}^k_h(s,a)\vee 1}\biggr)} \nonumber\\
       &\leq H \sqrt{\sum_{h=1}^H\sum_{s,a}d_h^k(s,a)\biggr(\frac{\beta(\bar{n}^k_h(s,a),\delta')}{\bar{n}^k_h(s,a)\vee 1}\biggr)}.\label{eqn:sum1}
\end{align}
  Due to the Cauchy's inequality, we have
    \begin{equation}
    \label{eqn:ComG2}
    \begin{aligned}
         \mbox{LHS of (\ref{eqn:sum1})} &\le   \sqrt{\sum_{h=1}^H \sum_{s,a}\rho  d_h^{k,h_0-1}(s,a) \Var_{P_h}(V_{h+1}^{k,h_0-1})(s,a) + (1-\rho) d_h^{k,h_0}(s,a) \Var_{P_h}(V_{h+1}^{k,h_0})(s,a)} \\
        &\quad \quad \quad  {\times}\sqrt{\sum_{h=1}^H \sum_{s,a} (\rho  d_h^{k,h_0-1}(s,a) + (1-\rho) d_h^{k,h_0}(s,a))\biggr(\frac{\beta^\star(\bar{n}^k_h(s,a),\delta')}{\bar{n}^k_h(s,a)\vee 1}\biggr)}.
    \end{aligned}
    \end{equation}
Besides, due to \Cref{lem:variancesum}, we have 
\begin{align*}
    &\sum_{h=1}^H \sum_{s,a} d_h^{k,h_0}(s,a) \Var_{P_h}(V_{h+1}^{k,h_0})(s,a) \leq \mathbb{E}_{\pi^{k,h_0}}\left[\left(\sum_{h=1}^H r_h(s_h,a_h)-V_1^{k,h_0}(s_1)\right)^2\right] \le  H^2.
\end{align*}
   
\if{0}    
    \begin{align*}
    &\sum_{h=1}^H \sum_{s,a} d_h^{k,h_0}(s,a) \Var_{P_h}(V_{h+1}^{k,h_0})(s,a) \\
    &\quad \overset{(a)}{=}\mathbb{E}_{\pi^{k,h_0}}\left[ \sum_{h=1}^H \left( V_{h+1}^{k,h_0}(s_{h+1})-Q_h^{k,h_0}(s_h,a_h) \right)^2 \right] \\
    &\quad = \mathbb{E}_{\pi^{k,h_0}}\left[ \sum_{h=1}^H \left( V_{h+1}^{k,h_0}(s_{h+1}) \right)^2  - 2 V_{h+1}^{k,h_0}(s_{h+1})Q_h^{k,h_0}(s_h,a_h)  + \left(Q_h^{k,h_0}(s_h,a_h)\right)^2 \right] \\
    &\quad =  \mathbb{E}_{\pi^{k,h_0}}\left[ \sum_{h=1}^H -\left( V_{h+1}^{k,h_0}(s_{h+1}) \right)^2  + \left(Q_h^{k,h_0}(s_h,a_h)\right)^2 \right] \\
    &\quad =  \mathbb{E}_{\pi^{k,h_0}}\left[ \sum_{h=1}^H -\left( V_{h+1}^{k,h_0}(s_{h+1}  ) \right)^2  + \left(Q_h^{k,h_0}(s_h,a_h) - V_h^{k,h_0}(s_h) \right)^2  \right] \\
    &\quad \overset{(b)}{=}  \mathbb{E}_{\pi^{k,h_0}}\left[\left(\sum_{h=1}^H r_h(s_h,a_h)-V_1^{k,h_0}(s_1)\right)^2\right] \le  H^2,
    \end{align*}
    where (a) follows from the definitions of occupancy measure $d_h^{k,h_0}$ and variance $\Var_{P_h}(V_{h+1}^{k,h_0})(s,a)$, (b) is because that the summation of the variances equals the variance of the summation.
\fi

    Similarly, we also have $\sum_{h=1}^H \sum_{s,a} d_h^{k,h_0-1}(s,a)\Var_{P_h}(V_{h+1}^{k,h_0-1})(s,a)\le  H^2$. Together with \cref{eqn:Gboundcomb1}, we have
    \begin{align*}
        \mbox{RHS of (\ref{eqn:ComG2})}&\le H \sqrt{\sum_{h=1}^H \sum_{s,a}  d_h^k(s,a) \biggr(\frac{\beta^\star(\bar{n}^k_h(s,a),\delta')}{\bar{n}^k_h(s,a)\vee 1}\biggr)},
    \end{align*}
 \if{0}   
    With these quantities bounded by $H^2$, we have the following bound: 
    \begin{equation}
        \label{eqn:Gbound_var}
        \sqrt{\sum_{h=1}^H \sum_{s,a}\rho  d_h^1(s,a) \Var_{P_h}(V_{h+1}^1)(s,a) + (1-\rho) d_h^2(s,a) \Var_{P_h}(V_{h+1}^2)(s,a)}\le H.
    \end{equation} 
    Plugging \Cref{eqn:Gbound_var} into \Cref{eqn:ComG2}, the RHS of \Cref{eqn:ComG2} can be reduced to\jingc{unclear}\hrq{edited}
    \begin{equation}
    \label{eqn:Gboundcomb2}
    \begin{aligned}
        & \sqrt{\sum_{h=1}^H \sum_{s,a}\rho  d_h^{k,h_0-1}(s,a) \Var_{P_h}(V_{h+1}^{k,h_0-1})(s,a) + (1-\rho) d_h^{k,h_0}(s,a) \Var_{P_h}(V_{h+1}^{k,h_0})(s,a)} \\
        &\quad \quad \quad  {\times}\sqrt{\sum_{h=1}^H \sum_{s,a} (\rho  d_h^{k,h_0-1}(s,a) + (1-\rho) d_h^{k,h_0}(s,a))\biggr(\frac{\beta^\star(\bar{n}^k_h(s,a),\delta')}{\bar{n}^k_h(s,a)\vee 1}\biggr)} \\
        &\quad \le H \sqrt{\sum_{h=1}^H \sum_{s,a}  d_h^k(s,a) \biggr(\frac{\beta^\star(\bar{n}^k_h(s,a),\delta')}{\bar{n}^k_h(s,a)\vee 1}\biggr)},
    \end{aligned}
    \end{equation}
    \fi
which completes the proof.
\end{proof}


Equipped with \Cref{lem:Gboundcomb}, we are ready to establish the \emph{finite non-optimistic policy lemma}, which states that for step mixture policies, there are only finite episodes in which $\pi^k\neq \pi^{k,0}$.

\begin{lemma}[Finite non-optimistic policy lemma]
    \label{lem:finite}
    Define $\mathcal{N}=\{k|k\in[K],\pi^k\neq \pi^{k,0}\}$. Then, the cardinality of $\mathcal{N}$ is upper bounded by
    \[|\mathcal{N}|\le \bigg(\frac{4608 e^{26}}{\kappa^2}+\frac{2688 e^{13}S}{\kappa}\bigg)H^3SA\log^2(K+1)=\tilde{O}\bigg(\bigg(\frac{1}{\kappa^2}+\frac{S}{\kappa}\bigg)H^3SA\bigg),\]
    where $\kappa=V_1^{\pi^b}-\gamma$.
\end{lemma}

\begin{proof}
By the definition of $\mathcal{N}$, we have
    \begin{align}
          |\mathcal{N}| \kappa         &\overset{(a)}{\le}   \sum_{k\in\mathcal{N}} \left( V_1^{\pi^b} - (\rho\utilde{V}_1^{k,h_0-1}+(1-\rho)\utilde{V}_1^{k,h_0}) \right) \label{eqn:regret_baseline}\\
        & \overset{(b)}{\le}  \sum_{k\in\mathcal{N}} \left( (\rho V_1^{\star,h_0-1}+(1-\rho)V_1^{\star,h_0}) - (\rho\utilde{V}_1^{k,h_0-1}+(1-\rho)\utilde{V}_1^{k,h_0}) \right) \notag \\
        & =   \sum_{k\in\mathcal{N}} \left( \rho(V_1^{\star,h_0-1}-\utilde{V}_1^{k,h_0-1})+(1-\rho)(V_1^{\star,h_0}-\utilde{V}_1^{k,h_0}) \right) \notag \\
        & \overset{(c)}{\le}   \sum_{k\in\mathcal{N}} \left( \rho \pi_1^{k,h_0-1} G_1^{k,h_0-1} + (1-\rho)\pi_1^{k,h_0} G_1^{k,h_0} \right) \label{eqn:sum_comb_G}\\
        &\overset{(d)}{\le}  24 e^{13} H \sum_{k\in\mathcal{N}}\sqrt{\sum_{h=1}^H \sum_{s,a} d_h^{k}(s,a) \left(\frac{\beta^\star(\bar{n}^k_h(s,a),\delta')}{\bar{n}^k_h(s,a)\vee 1}\right)} \notag\\
        &\quad\quad + 336 e^{13} H^2 \sum_{k\in\mathcal{N}}\sum_{h=1}^H \sum_{s,a} d_h^{k}(s,a)\left(\frac{\beta(\bar{n}^k_h(s,a),\delta')}{\bar{n}^k_h(s,a)\vee 1}\right)\notag \\
        &\overset{(e)}{\le}  24 e^{13} H \sqrt{|\mathcal{N}|} \sqrt{\sum_{k\in\mathcal{N}}\sum_{h=1}^H \sum_{s,a} d_h^{k}(s,a) \left(\frac{\beta^\star(\bar{n}^k_h(s,a),\delta')}{\bar{n}^k_h(s,a)\vee 1}\right)} \notag \\
        &\quad\quad+ 336 e^{13} H^2 \sum_{k\in\mathcal{N}}\sum_{h=1}^H \sum_{s,a} d_h^{k}(s,a)\left(\frac{\beta(\bar{n}^k_h(s,a),\delta')}{\bar{n}^k_h(s,a)\vee 1}\right),\label{eqn:finite0}
    \end{align}
    where $(b)$ follows from \Cref{lem:stepoptimal}, $(c)$ is due to \Cref{lem:ringgap}, $(d)$ follows from \Cref{lem:Gboundcomb}, $(e)$ is due to the Cauchy's inequality. For inequality $(a)$, by the design of StepMix, when $\pi^k\neq\pi^{k,0}$, we must have $\rho\utilde{V}_1^{k,h_0-1}+(1-\rho)\utilde{V}_1^{k,h_0}=\gamma$ or $\rho\utilde{V}_1^{k,h_0-1}+(1-\rho)\utilde{V}_1^{k,h_0}=\utilde{V}_1^{\pi^{k,H}}\le \gamma$. Both cases indicate that inequality $(a)$ holds.
    
    We can also bound the summation as follows: 
    \begin{align}
        \sum_{k\in\mathcal{N}}\sum_{h=1}^H \sum_{s,a} d_h^k(s,a) \left(\frac{\beta^\star(\bar{n}^k_h(s,a),\delta')}{\bar{n}^k_h(s,a)\vee 1}\right) 
        & \le \sum_{h=1}^H \sum_{s,a} \sum_{k\in\mathcal{N}} d_h^k(s,a) \left(\frac{\beta^\star(\bar{n}^k_h(s,a),\delta')}{\bar{n}^k_h(s,a)\vee 1}\right) \notag\\
        & {\le} \beta^\star(K,\delta') \sum_{h=1}^H \sum_{s,a}\sum_{k\in\mathcal{N}}  \left( \frac{d_h^k(s,a)}{\bar{n}^k_h(s,a)\vee 1}\right) \notag\\
        & \overset{(a)}{\le} \beta^\star(K,\delta')\sum_{h=1}^H \sum_{s,a}  4\log{(|\mathcal{N}|+1)} \notag \\
        & \le 4HSA \beta^\star(K,\delta')\log{(|\mathcal{N}|+1)}\label{eqn:summation1},
    \end{align}
    where inequality (a) follows from \Cref{lem:partial_sum}.
    
    Similar to \Cref{eqn:summation1}, we also have
    \begin{equation}
        \label{eqn:summation2}
        \sum_{k\in\mathcal{N}}\sum_{h=1}^H \sum_{s,a} d_h^k(s,a) \left(\frac{\beta(\bar{n}^k_h(s,a),\delta')}{\bar{n}^k_h(s,a)\vee 1}\right) \le 4HSA \beta(K,\delta')\log{(|\mathcal{N}|+1)}.
    \end{equation}
    
    Hence, putting the bound of summations into \Cref{eqn:finite0}, we have
    \[|\mathcal{N}|\kappa\le 48 e^{13}  \sqrt{|\mathcal{N}|} \sqrt{H^3SA\beta^\star(K,\delta')\log{(|\mathcal{N}|+1)}} + 1344 e^{13} H^3 SA\beta(K,\delta')\log{(|\mathcal{N}|+1)}.\]
    
    Rearranging the terms, we conclude that  
    \[|\mathcal{N}|\le (\frac{4608 e^{26}}{\kappa^2}+\frac{2688 e^{13}S}{\kappa})H^3SA\log^2(K+1) = \tilde{O}\left(\big(\frac{1}{\kappa^2}+\frac{S}{\kappa} \big)H^3SA \right),\] 
    which completes the proof.
\end{proof}

\subsection{Step Four: Putting Everything Together}
\label{appx:stepthm}

Finally, with all the results above, we can prove \Cref{thm:step}.

\begin{theorem}[The complete version of \Cref{thm:step}]
Given $\delta\in(0,1)$, set $\delta' = \frac{\delta}{3(H + 1)}$, $\beta=\log(SAH/\delta') + S\log(8e(K+1))$, and $\beta^\star = \log(SAH/\delta') + \log(8e(K+1))$. Then, with probability at least $1-\delta$, 
StepMix satisfies constraint in (\ref{eqn:constraint}) and achieves a regret upper bounded as 
    \begin{align*}
       \mbox{Reg}(K) \leq & \tilde{O}\left(\sqrt{H^3SAK}+H^3S^2A+H^3SA\Delta_0 \left(\frac{1}{\kappa^2}+\frac{S}{\kappa} \right) \right),
    \end{align*}
    where $\Delta_0=V_1^\star-V_1^{\pi^b}$ and $\kappa = V_1^{\pi^b} - \gamma$.
\end{theorem}

\begin{proof}
We use the same notations specified in \Cref{appx:notation}. Then, under the previously defined good events (which occur with probability at least $1-\delta$), we have 
    \begin{align}
        \mbox{Reg}(K) &=  \sum_{k=1}^K (V_1^\star - V_1^{\pi^k}) \notag\\
        &\overset{(a)}{=}   \sum_{k\notin\mathcal{N}} (V^{\star,0}-V^{k,0}) + \sum_{k\in\mathcal{N}} (V^\star-V^{\pi^b}) + \sum_{k\in\mathcal{N}} (V^{\pi^b}-V^{\pi^k})\notag \\
        & \overset{(b)}{\le}   \sum_{k\notin\mathcal{N}} (V^{\star,0}-\utilde{V}^{k,0}) + |\mathcal{N}| \Delta_0 + \sum_{k\in\mathcal{N}} (\rho \pi_1^{k,h_0-1}G_1^{k,h_0-1}+(1-\rho) \pi_1^{k,h_0}G_1^{k,h_0}) \notag\\
        & \overset{(c)}{\le}   \sum_{k=1}^{K} (\rho \pi_1^{k,h_0-1}G_1^{k,h_0-1}+(1-\rho) \pi_1^{k,h_0}G_1^{k,h_0}) + |\mathcal{N}| \Delta_0 \notag\\
        &\overset{(d)}{\le}   \sum_{k=1}^K \Biggr(24 e^{13} H \sqrt{\sum_{h=1}^H \sum_{s,a}  d_h^{k}(s,a) \left(\frac{\beta^\star(\bar{n}^k_h(s,a),\delta')}{\bar{n}^k_h(s,a)\vee 1}\right)} \notag\\
        &\quad\quad\quad\quad + 336 e^{13} H^2 \sum_{h=1}^H \sum_{s,a} d_h^{k}(s,a)\left(\frac{\beta(\bar{n}^k_h(s,a),\delta')}{\bar{n}^k_h(s,a)\vee 1}\right)\Biggr) + |\mathcal{N}| \Delta_0,\label{eqn:thmstep1}
    \end{align}
    where $(a)$ is due to that $\pi^{\star,0}=\pi^\star$ and $\pi^k=\pi^{k,0}$ when $k\notin\mathcal{N}$; $(b)$ follows from the fact that $\utilde{V}_1^{k,0}$ is the LCB of $V_1^{k,0}$ and \Cref{eqn:regret_baseline} to \Cref{eqn:sum_comb_G} in the proof of \Cref{lem:finite}; $(c)$ is due to $V_1^{\star,0}-\utilde{V}_1^{k,0}\le\pi_1^{k,0}G_1^{k,0}$ as shown in \Cref{lem:ringgap}, and $(d)$ is from \Cref{lem:Gboundcomb}.

    \if{0}
    \begin{align*}
        &\tilde{V}_1^{k,0}-V^{\pi^{k,0}}\le\tilde{V}_1^{k,0}-\utilde{V}_1^{k,0}\le \pi^{k,0}_1G_1^{k,0}({s_1}) \\
        \le & 24 e^{13} H \sqrt{\sum_{h=1}^H \sum_{s,a}  d_h^{k,0}(s,a) (\frac{\beta^\star(\bar{n}^k_h(s,a),\delta)}{\bar{n}^k_h(s,a)\vee 1})} + 336 e^{13} H^2 \sum_{h=1}^H \sum_{s,a} d_h^{k,0}(s,a)(\frac{\beta(\bar{n}^k_h(s,a),\delta')}{\bar{n}^k_h(s,a)\vee 1})
    \end{align*}
    \fi
    
    Using \Cref{eqn:summation1,eqn:summation2} from the proof of Lemma \ref{lem:finite}, we obtain
    \begin{align}
        & \sum_{k=1}^K \left(24 e^{13} H \sqrt{\sum_{h=1}^H \sum_{s,a}  d_h^{k}(s,a) \left(\frac{\beta^\star(\bar{n}^k_h(s,a),\delta)}{\bar{n}^k_h(s,a)\vee 1}\right)} + 336 e^{13} H^2 \sum_{h=1}^H \sum_{s,a} d_h^{k}(s,a)\left(\frac{\beta(\bar{n}^k_h(s,a),\delta')}{\bar{n}^k_h(s,a)\vee 1}\right)\right)\nonumber\\
        &\quad \le 48 e^{13}  \sqrt{H^3SAK \beta^\star(K,\delta')\log{(K+1)}} + 1344 e^{13} H^3 SA\beta(K,\delta')\log{(K+1)}.\label{eqn:sum2}
    \end{align}
    Plugging (\ref{eqn:sum2}) and the result of Lemma \ref{lem:finite} into \Cref{eqn:thmstep1}, we further have:
    \begin{align*}
        \mbox{Reg}(K) &\le   48 e^{13} \sqrt{H^3SAK \beta^\star(K,\delta')\log{(K+1)}} + 1344 e^{13} H^3 SA\beta(K,\delta')\log{(K+1)} \\
        &\quad  + \left(\frac{4608 e^{26}}{\kappa^2}+\frac{2688 e^{13}S}{\kappa}\right)H^3SA\log(K+1)\Delta_0 \\
        & =  \tilde{O}\left(\sqrt{H^3SAK}+H^3S^2A+H^3SA\Delta_0\left(\frac{1}{\kappa^2}+\frac{S}{\kappa}\right)\right).
    \end{align*}

    Finally, we prove that StepMix satisfies the constraint episodically. Specifically, for any online policy $\pi^k$, we have 
    \[V_1^{\pi^k}=\rho V_1^{k,h_0-1}+(1-\rho) V_1^{k,h_0}\ge\rho \utilde{V}_1^{k,h_0-1}+(1-\rho) \utilde{V}_1^{k,h_0}.\]
    If $\pi^k=\pi^{k,H}=\pi^b$, $\pi^k$ must satisfy the constraint, because $\pi^b$ is assumed to be safe. Otherwise, if $h_0 = 0$, that means $\pi^k=\pi^{k,0}$ and $\utilde{V}_1^{k,0}>\gamma$, so $\pi^k$ must be safe; if $h_0\neq 0$, we have 
    $\utilde{V}_1^{k,h_0}\ge\gamma$ and $\utilde{V}_1^{k,h_0-1}<\gamma$, so that $\rho = \frac{\utilde{V}_1^{k,h_0}({s_1})-\gamma}{\utilde{V}_1^{k,h_0}({s_1})-\utilde{V}_1^{k,h_0-1}}$ guarantees that
    $\gamma=\rho \utilde{V}_1^{k,h_0-1}+(1-\rho) \utilde{V}_1^{k,h_0}$. Since $\rho \utilde{V}_1^{k,h_0-1}+(1-\rho) \utilde{V}_1^{k,h_0}\le \rho V_1^{k,h_0-1}+(1-\rho) V_1^{k,h_0} = V_1^{\pi^k}$, $\pi^k$ is also safe. 
\end{proof}

  We remark that when $\gamma= 0$ the additive term $\tilde{O}\left(H^3 S A \Delta_0\left(\frac{1}{\kappa^2}+\frac{S}{\kappa}\right)\right)$ in \Cref{thm:step} can be dropped, as formally stated in the following corollary.
  
\begin{corollary}[Vanishing additive term]
    \label{re:more}
When $\gamma= 0$, with all the parameters specified in \Cref{thm:step}, StepMix satisfies constraint (\ref{eqn:constraint}) and achieves a  regret upper bounded as
    \begin{align*}
        \mbox{Reg}(K) \leq & \tilde{O}\left(\sqrt{H^3SAK}+H^3S^2A\right),
    \end{align*}
    where $\Delta_0=V_1^\star-V_1^{\pi^b}$, $\kappa = V_1^{\pi^b} - \gamma$.
\end{corollary} 

\begin{proof}
    If $\gamma= 0$, based on the definition of $\utilde{Q}^{k,h_0}_h$ in \Cref{Xd}, we have $\utilde{Q}^{k,h_0}_h\ge 0= \gamma$ for any $k$, $h_0$ and $h$. Thus, under StepMix, the executed policy $\pi^k$ must be the optimistic policy $\bar{\pi}^k$. Recall that the definition of $\mathcal{N}$ is $\mathcal{N}=\{k|k\in[K],\pi^k\neq \bar{\pi}^{k}\}$. Therefore, we have $\mathcal{N}=\emptyset$ and $|\mathcal{N}|=0$.
    
    With $|\mathcal{N}|=0$ and \Cref{eqn:thmstep1}, we have   
    \begin{align*}
        \mbox{Reg}(K) &\le \sum_{k=1}^K \left(24 e^{13} H \sqrt{\sum_{h=1}^H \sum_{s,a}  d_h^{k}(s,a) \left(\frac{\beta^\star(\bar{n}^k_h(s,a),\delta')}{\bar{n}^k_h(s,a)\vee 1}\right)} + 336 e^{13} H^2 \sum_{h=1}^H \sum_{s,a} d_h^{k}(s,a)\left(\frac{\beta(\bar{n}^k_h(s,a),\delta')}{\bar{n}^k_h(s,a)\vee 1}\right)\right) \\
        &\quad\quad+ |\mathcal{N}| \Delta_0 \\
        &=\sum_{k=1}^K \left(24 e^{13} H \sqrt{\sum_{h=1}^H \sum_{s,a}  d_h^{k}(s,a) \left(\frac{\beta^\star(\bar{n}^k_h(s,a),\delta')}{\bar{n}^k_h(s,a)\vee 1}\right)} + 336 e^{13} H^2 \sum_{h=1}^H \sum_{s,a} d_h^{k}(s,a)\left(\frac{\beta(\bar{n}^k_h(s,a),\delta')}{\bar{n}^k_h(s,a)\vee 1}\right)\right)\\
        & \overset{(a)}{\le} 48 e^{13}  \sqrt{H^3SAK \beta^\star(K,\delta')\log{(K+1)}} + 1344 e^{13} H^3 SA\beta(K,\delta')\log{(K+1)} \\
        &= \tilde{O}\left(\sqrt{H^3SAK}+H^3S^2A\right),
    \end{align*}
    where (a) is due to \Cref{eqn:summation1,eqn:summation2}. 
\end{proof}

\section{Algorithm Design and Analysis of EpsMix Algorithm}
\label{appx:eps_proof}

In this section, we present the detailed design and analysis of the EpsMix algorithm. 
\if{0}
Different from StepMix in Algorithm~\ref{alg:step}, EpsMix does not construct step mixture policies during the learning process. Rather, it adopts a randomization mechanism at the beginning of each episode, and designs episodic mixture policies \citep{wiering2008ensemble,baram2021maximum} defined as follows.

\begin{definition}[Episodic Mixture Policy]
Given two policies $\pi^1$ and $\pi^2$ with a parameter $\rho\in (0,1)$, the episodic mixture policy, denoted by $\rho\pi^1 \oplus (1-\rho)\pi^2$, randomly picks $\pi^1$ with probability $\rho$ and $\pi^2$ with probability $1-\rho$ at the beginning of an episode and plays it for the entire episode.
\end{definition}
\fi

\subsection{Algorithm Design}

\begin{algorithm}[!hbt]
    \caption{The EpsMix Algorithm}
    \label{alg:eps}
\begin{algorithmic}
\STATE {\bf Input:} $\pi^b$, $\gamma$, $\beta$, {$\beta^\star$}, $\Dc_0=\emptyset$
\FOR{$k$ = $1$ to $K$}
    \STATE Update the model estimate \begin{align*}
        \hP_h^k(s'|s,a)=\left\{\begin{aligned}
        & n_h^k(s,a,s')/n_h^k(s,a), &\mbox{ if } n_h^k(s,a)>0, \\
        & 1/S, &\mbox{ if } n_h^k(s,a)=0.
    \end{aligned}\right.
    \end{align*}
    \STATE \textcolor{blue}{\it \# Optimistic policy identification}
    \STATE $\tilde{Q}_{H+1}^{k}=\utilde{Q}_{H+1}^{k}=0$.
    \FOR{$h$ = $H$ to $1$}
        \STATE Update $\tilde{Q}_{h}^{k}(s,a)$, $\utilde{Q}_{h}^{k}(s,a), \forall (s,a)\in \Sc\times \Ac$ according to \Cref{eqn:tildeQ}.
        \STATE $\gp_h^{k}(s)\gets\argmax_a \tilde{Q}_{h}^{k}(s,a)$, $\tilde{V}^{k}_h(s)\gets\tilde{Q}^{k}_h(s,\gp_h^k(s))$, $\utilde{V}^{k}_h(s)\gets \utilde{Q}^{k}_h(s,\gp_h^k(s)),\forall s\in\Sc$.
    \ENDFOR
    \STATE \textcolor{blue}{\it \# Evaluate the baseline policy}
    \STATE $\tilde{Q}_{H+1}^{k,b}=\utilde{Q}_{H+1}^{k,b}=0$.
    \FOR{$h$ = $H$ to $1$}
        \STATE Update $\tilde{Q}_{h}^{k,b}(s,a)$, $\utilde{Q}_{h}^{k,b}(s,a), \forall (s,a)\in \Sc\times \Ac$ according to \Cref{eqn:tildeQb}.
        \STATE $\tilde{V}^{k,b}_h(s)\gets\tilde{Q}^{k,b}_h(s,\pi_{h}^{b}(s))$, $\utilde{V}^{k,b}_h(s)\gets \utilde{Q}^{k,b}_h(s,\pi_h^b(s)),\forall s\in\Sc$.
    \ENDFOR
    \STATE  \textcolor{blue}{\it  \# Safe exploration policy selection}
    \IF{$\utilde{V}_1^{k} \ge \gamma$}
        \STATE $\pi^k = \gp^{k}$. 
    \ELSIF{$\utilde{V}_1^{k,b} < \gamma$}
        \STATE $\pi^k = \pi^b$. 
    \ELSE
        \STATE $\rho = \frac{\utilde{V}_1^{k,b}({s_1})-\gamma}{\utilde{V}_1^{k,b}({s_1})-\utilde{V}_1^{k}}$,
        \STATE $\pi^k = \rho \gp^{k}\oplus(1-\rho) \pi^b$. 
    \ENDIF
    \STATE Execute $\pi^k$ and collect $\{(s_h^k,a_h^k,s_{h+1}^k)\}_{h=1}^H$.
    \STATE $\Dc_n\gets \Dc_{n-1}\cup\{(s_h^k,a_h^k,s_{h+1}^k)\}_{h=1}^H$.
\ENDFOR
\end{algorithmic}
\end{algorithm}

The EpsMix algorithm is presented in \Cref{alg:eps}. 
\if{0}
Similar to StepMix, at the beginning of each episode $k$, it first constructs an optimistic policy, denoted as $\gp^{k}$. It then evaluates the LCB of the expected total rewards under both $\gp^{k}$ and ${\pi}^b$, denoted as $\utilde{V}^{k}$ and $\utilde{V}^{k,b}$, respectively.  
If $\utilde{V}_1^{k}$ is above the threshold $\mathbf{\gamma}$, it indicates that the optimistic policy $\gp^{k}$ satisfies the conservative constraint with high probability. The learner thus executes $\gp^{k}$ in the following episode $k$. Otherwise, if $\utilde{V}_1^{k,b}$ is above the threshold $\gamma$ while $\utilde{V}_1^{k}$ is not, it constructs an episodic mixture policy $\rho_k \gp^{k} \oplus (1-\rho_k)\pi^b$ in \Cref{alg:eps} so that $\rho_k\cdot\utilde{V}_1^{k}+(1-\rho_k)\cdot\utilde{V}_1^{k,b}=\mathbf{\gamma}$. It implies that the episodic mixture policy satisfies the conservative constraint in expectation. If neither $\utilde{V}_1^{k}$ nor $\utilde{V}_1^{k,b}$ is above the threshold, EpsMix will resort to the baseline policy to collect more information.
\fi 
The update rule of $\tilde{Q}^k_h$, $\utilde{Q}^k_h$ is given below, where $\delta'=\delta/4$.
\begin{equation}
    \label{eqn:tildeQ}
    \begin{aligned}
    & \tilde{Q}^{k}_h(s,a)\triangleq  \min\biggr(H, r_h(s,a)+3\sqrt{\Var_{\hP_h^k}(\tilde{V}_{h+1}^{k})(s,a)\frac{\beta^\star(n^k_h(s,a),\delta')}{n^k_h(s,a)}} + 14 H^2\frac{\beta(n^k_h(s,a),\delta')}{n^k_h(s,a)}\\
    &\hspace{2.5cm} +\frac{1}{H}\hP_h^k(\tilde{V}^{k}_{h+1}-\utilde{V}^{k}_{h+1})(s,a)+\hP_h^k\tilde{V}^{k}_{h+1}(s,a)\biggr),\\
    & \utilde{Q}^{k}_h(s,a)\triangleq  \max \biggr(0, r_h(s,a)-3\sqrt{\Var_{\hP_h^k}(\tilde{V}_{h+1}^{k})(s,a)\frac{\beta^\star(n^k_h(s,a),\delta')}{n^k_h(s,a)}} - 22 H^2\frac{\beta(n^k_h(s,a),\delta')}{n^k_h(s,a)}\\
    & \hspace{2.5cm} -\frac{2}{H}\hP_h^k(\tilde{V}^{k}_{h+1}-\utilde{V}^{k}_{h+1})(s,a)+\hP_h^k\utilde{V}^{k}_{h+1}(s,a) \biggr).
\end{aligned}
\end{equation}

Similarly, the update rule of $\tilde{Q}_h^{k,b}$ and $\utilde{Q}_h^{k,b}$ is defined as
\begin{equation}
    \label{eqn:tildeQb}
    \begin{aligned}
    & \tilde{Q}^{k,b}_h(s,a)\triangleq  \min\biggr(H, r_h(s,a)+3\sqrt{\Var_{\hP_h^k}(\tilde{V}_{h+1}^{k,b})(s,a)\frac{\beta^\star(n^k_h(s,a),\delta')}{n^k_h(s,a)}} + 14 H^2\frac{\beta(n^k_h(s,a),\delta')}{n^k_h(s,a)}\\
    &\hspace{2.5cm} +\frac{1}{H}\hP_h^k(\tilde{V}^{k,b}_{h+1}-\utilde{V}^{k,b}_{h+1})(s,a)+\hP_h^k\tilde{V}^{k,b}_{h+1}(s,a)\biggr),\\
    & \utilde{Q}^{k,b}_h(s,a)\triangleq  \max \biggr(0, r_h(s,a)-3\sqrt{\Var_{\hP_h^k}(\tilde{V}_{h+1}^{k,b})(s,a)\frac{\beta^\star(n^k_h(s,a),\delta')}{n^k_h(s,a)}} - 22 H^2\frac{\beta(n^k_h(s,a),\delta')}{n^k_h(s,a)}\\
    &\hspace{2.5cm} -\frac{2}{H}\hP_h^k(\tilde{V}^{k,b}_{h+1}-\utilde{V}^{k,b}_{h+1})(s,a)+\hP_h^k\utilde{V}^{k,b}_{h+1}(s,a) \biggr).
\end{aligned}
\end{equation}

\subsection{Theoretical Analysis}
The performance of the EpsMix Algorithm is characterized in the following theorem.

\begin{theorem}[Regret of EpsMix]
    \label{thm:episodic} 
    Given $\delta\in(0,1)$, set $\delta' = \frac{\delta}{4}$, $\beta=\log(SAH/\delta') + S\log(8e(K+1))$, and $\beta^\star = \log(SAH/\delta') + \log(8e(K+1))$. Then, with probability at least $1-\delta$, 
    EpsMix (\Cref{alg:eps}) simultaneously (i) satisfies the conservative constraint in (\ref{eqn:constraint}), and (ii) achieves a total regret that is upper bounded by
        \begin{align*}      \tilde{O}\bigg(\sqrt{H^3SAK}+H^3S^2A+H^3SA\Delta_0\bigg(\frac{1}{\kappa^2}+\frac{S}{\kappa}\bigg)\bigg),
        \end{align*}
    where {$\Delta_0 = V_1^\star - V_1^{\pi^b}$} is the suboptimality gap of the baseline policy, and $\mathbf{\kappa} = V_1^{\pi^b} - \mathbf{\gamma}$ is the tolerable value loss from the baseline policy.
\end{theorem}

\if{0}
\begin{remark}
We note that EpsMix has the same performance guarantees as StepMix.
At the same time, we note that EpsMix is less conservative than StepMix in the sense that, the expected total return under a {\it selected} policy in an episode may be below the threshold when $\utilde{V}_1^{\bar{\pi}^k}< {\mathbf{\gamma}}$. However, when taking the randomness in the policy mixture procedure into consideration, we can still guarantee that the expected total return under an episodic mixture policy is above the threshold with probability at least $1-\delta$.
\end{remark}
\fi

Before we proceed to prove \cref{thm:episodic}, we sketch the proof as follows: {First}, we establish the UCBs and LCBs of the value functions for the baseline policy $\pi^b$ and the optimal policy $\pi^\star$ in each episode, following similar approaches as in the proof of Theorem~\ref{thm:step}. We {then} show that the total number of episodes where the algorithm executes $\pi^b$ or the episodic mixture policy is bounded, which ensures that the performance degradation compared with BPI-UCBVI~\citep{menard2021fast} is bounded. Finally, the established LCBs ensure that the conservative constraint is satisfied in each episode.


\begin{lemma}
    \label{lem:epsgood}
    With probability at least $1 - \delta$, the following good events occur simultaneously:
    \[\mathcal{E}(\delta'), \mathcal{E}^{\text{cnt}}(\delta'), \mathcal{E}^\star(V^{\star},\delta'),\mathcal{E}^\star(V^{\pi^b},\delta'),\]
    where $\delta'=\delta/4$.
\end{lemma}
\begin{proof}
This result can be obtained by noting that each of those events hold with probability at least $1-\delta/4$ under \Cref{thm:goodevent1}, \Cref{thm:goodevent2} and \Cref{thm:goodevent3}, and then taking the union bound.
\end{proof}

In the following proof of EpsMix, we set $\delta' = \delta/4$. 
We note that \Cref{eqn:tildeQ} and \Cref{eqn:tildeQb} are defined in a similar form as \Cref{eqn:ucbQ}. As a result, \Cref{lem:stepucblcb,lem:ringgap,lem:Gbound} can be directly extended for EpsMix, as stated below. We note that $Q^{k,h_0}$ need to be bounded for every $h_0\in [H]\cup\{0\}$ in StepMix, while in EpsMix, we only need to bound $Q^{k}$ and $Q^{k,b}$.

\begin{lemma}[UCB and LCB for EpsMix]
    \label{lem:epsucblcb}
    The relationship between $\tilde{Q}_h^{k}$, $\utilde{Q}^{k}_h$, $\tilde{V}^{k}_h$, $\utilde{V}^{k}_h$ and the corresponding true value functions $Q^{k}_h$, $V^{k}_h$, $Q^{\star}_h$, $V_h^{\star}$ are specified in the following inequalities:
    \begin{align*}
        \utilde{Q}_h^{k}(s,a)\le Q_h^{k}(s,a)&\le Q_h^{\star}(s,a)\le \tilde{Q}_h^{k}(s,a),\\
        \utilde{V}_h^{k}(s)\le V_h^{k}(s)&\le V_h^{\star}(s)\le \tilde{V}_h^{k}(s),
    \end{align*}
    In addition, the relationships between $\tilde{Q}^{k,b}_h$, $\utilde{Q}^{k,b}_h$, $\tilde{V}^{k,b}_h$, $\utilde{V}^{k,b}_h$, and the true value functions $Q^{\pi^b}_h,V^{\pi^b}_h$ are specified in the following inequalities:
    \begin{align*}
        \utilde{Q}_h^{k,b}(s,a)\le &Q_h^{\pi^b}(s,a)\le \tilde{Q}_h^{k,b}(s,a),\\
        \utilde{V}_h^{k,b}(s)\le &V_h^{\pi^b}(s)\le \tilde{V}_h^{k,b}(s).
    \end{align*}
\end{lemma}

\begin{lemma}
    \label{lem:epsgap}
    Define $G_h^{k}$ and $G_h^{k,b}$ as 
    \begin{equation}
        \label{eqn:epsG}
        G_h^{k}(s,a)=\min\Biggr(H, 6\sqrt{\Var_{\hP_h^k}(\tilde{V}^{k}_{h+1})(s,a)\frac{\beta^\star(n^k_h(s,a),\delta')}{n^k_h(s,a)}}+36H^2\frac{\beta(n^k_h(s,a),\delta')}{n^k_h(s,a)}+(1+\frac{3}{H})\hP_h^k\pi_{h+1}^{k}G_{h+1}^{k}(s,a)\Biggr),
    \end{equation}
    \begin{equation}
        \label{eqn:epsGb}
        G_h^{k,b}(s,a)=\min\Biggr(H, 6\sqrt{\Var_{\hP_h^k}(\tilde{V}^{k,b}_{h+1})(s,a)\frac{\beta^\star(n^k_h(s,a),\delta')}{n^k_h(s,a)}}+36H^2\frac{\beta(n^k_h(s,a),\delta')}{n^k_h(s,a)}+(1+\frac{3}{H})\hP_h^k\pi_{h+1}^{k,b}G_{h+1}^{k,b}(s,a)\Biggr).
    \end{equation}
    Then, the estimation error between $Q_h^{\star}$, $V_h^{\star}$ and $\utilde{Q}_h^{k}$, $\utilde{V}_h^{k}$ can be bounded as
    \[Q_h^{\star}(s,a)-\utilde{Q}_h^{k}(s,a)\le G_h^{k}(s, a),\]
    \[V_h^{\star}(s)-\utilde{V}_h^{k}(s)\le \langle \hat{\pi}_h^{k}(\cdot|s), G_h^{k}(s,\cdot) \rangle .\]
    Moreover, the estimation error between $Q_h^{\pi^b}$, $V_h^{\pi^b}(s)$ and $\utilde{Q}_h^{k,b}$, $\utilde{V}_h^{k,b}$ can be bounded as
    \[Q_h^{\pi^b}(s,a)-\utilde{Q}_h^{k,b}(s,a)\le G_h^{k,b}(s, a),\]
    \[V_h^{\pi^b}(s)-\utilde{V}_h^{k,b}(s)\le \langle \pi_h^{b}(\cdot|s), G_h^{k,b}(s,\cdot) \rangle .\]
\end{lemma}

\begin{lemma}[Upper bound $\pi_{1}^{k} G_{1}^{k}$ and 
$\pi_{1}^{k,b} G_{1}^{k,b}$]
    \label{lem:epsGbound}
    Recall the functions of $G^k$ and $G^{k,b}$ in \Cref{eqn:epsG,eqn:epsGb}. We have
    \begin{equation*}
    \begin{aligned}
        \pi_{1}^{k} G_{1}^{k}({s_1}) 
        \le & 24e^{13}\sum_{h=1}^H \sum_{s,a} d_h^{k}(s,a) \sqrt{\Var_{P_h}(V_{h+1}^{k})(s,a)\left(\frac{\beta^\star(\bar{n}^k_h(s,a),\delta')}{\bar{n}^k_h(s,a)\vee 1}\right)} \\
        & + 336 e^{13} H^2 \sum_{h=1}^H \sum_{s,a} d_h^{k}(s,a)\left(\frac{\beta(\bar{n}^k_h(s,a),\delta')}{\bar{n}^k_h(s,a)\vee 1}\right),
    \end{aligned}
    \end{equation*}
    and
    \begin{equation*}
    \begin{aligned}
        \pi_{1}^{k,b} G_{1}^{k,b}({s_1}) 
        \le & 24e^{13}\sum_{h=1}^H \sum_{s,a} d_h^{b}(s,a) \sqrt{\Var_{P_h}(V_{h+1}^{k,b})(s,a)\left(\frac{\beta^\star(\bar{n}^k_h(s,a),\delta')}{\bar{n}^k_h(s,a)\vee 1}\right)} \\
        & + 336 e^{13} H^2 \sum_{h=1}^H \sum_{s,a} d_h^{b}(s,a)\left(\frac{\beta(\bar{n}^k_h(s,a),\delta')}{\bar{n}^k_h(s,a)\vee 1}\right),
    \end{aligned}
    \end{equation*}
    where $d_h^{k}$ and $d_h^{b}$ are the {occupancy measures} under policy $\bar{\pi}^{k}$ and $\pi^{b}$, respectively.
\end{lemma}

The proofs of the above three lemmas follow the same approaches as those for StepMix, and thus are omitted.


Besides, although the construction of the mixture policy under EpsMix is different from that under StepMix, the linearlity of the occupancy measure and the corresponding value function is preserved under EpsMix.

\begin{lemma}
    \label{lem:epsvisitratio}
    Let $\pi=\rho\pi^1\oplus(1-\rho)\pi^2$, 
    and $d_h^1(s,a)$ and $d_h^2(s,a)$ be the {occupancy measures} under $\pi^1$ and $\pi^2$, respectively. Then, the following equality holds: \[d_h^{\pi}(s,a)=\rho d_h^1(s,a) + (1-\rho) d_h^2(s,a).\]
    
    Recall that the {occupancy measure} under a policy $\pi$ is defined as $d_h^{\pi}(s,a)=\mathbb{E}_{\pi}[\mathds{1}\{s_h=s,a_h=a\}]$.
\end{lemma}
\begin{proof}
    Let ${B}_\rho$ be an independent Bernoulli random variable with mean $\rho$, and let $\pi$ be $\pi^1$ if $B_\rho=1$ and be $\pi^2$ otherwise. Then, 
    \begin{align*}
        d_h^{\pi}(s,a) = & \mathbb{E}_{\pi}[\mathds{1}\{s_h=s,a_h=a\}] \\
        = & \mathbb{E}_{\pi}[\mathds{1}\{s_h=s,a_h=a\}|{B}_\rho=1]{\Pb}[{B}_\rho=1]+\mathbb{E}_{\pi}[\mathds{1}\{s_h=s,a_h=a\}|{B}_\rho=0]{\Pb}[{B}_\rho=0] \\
        = & \mathbb{E}_{\pi^1}[\mathds{1}\{s_h=s,a_h=a\}]\cdot \rho+\mathbb{E}_{\pi^2}[\mathds{1}\{s_h=s,a_h=a\}](1-\rho) \\
        = & \rho d^1_h(s,a) + (1-\rho) d_h^2(s,a).
    \end{align*}
\end{proof}

\begin{proposition}
    \label{lem:epscomb}
    Under the same condition as in \Cref{lem:epsvisitratio}, the following equality holds:
    \[V_1^{\pi}=\rho V_1^{\pi^1} + (1-\rho) V_1^{\pi^2}.\]
\end{proposition}

Based on the linearity shown in \Cref{lem:epsvisitratio}, we obtain a result similar to that in \Cref{lem:Gboundcomb} for episodic mixture policies.
\begin{lemma}
    \label{lem:Gboundcomb_eps}
    For an episodic mixture policy $\pi^k$ mixed from two policies $\gp^{k}$ and $\pi^{b}$, defined as $\pi^k = (1-\rho)\pi^{b}\oplus\rho \gp^{k}$, the following bound holds:
    \begin{equation}
    \label{eqn:combined_G_bound_eps}
    \begin{aligned}
        & (1-\rho)\pi_{1}^{b} G_{1}^{k,b}({s_1}) + \rho \gp_{1}^{k} G_{1}^{k}({s_1}) \\
        &\quad \le   24e^{13} H \sqrt{\sum_{h=1}^H\sum_{s,a}d_h^k(s,a)\left(\frac{\beta(\bar{n}^k_h(s,a),\delta')}{\bar{n}^k_h(s,a)\vee 1} \right)} + 336 e^{13}H^2\sum_{h=1}^H\sum_{s,a}d_h^k(s,a)\left(\frac{\beta(\bar{n}^k_h(s,a),\delta')}{\bar{n}^k_h(s,a)\vee 1} \right),
    \end{aligned}
    \end{equation}
    where $d_h^k(s,a)$ is the {occupancy measure} under policy $\pi^k$.
\end{lemma}
\Cref{lem:Gboundcomb_eps} can be proved following a similar approach as in the proof of \Cref{lem:Gboundcomb}. 


Now we establish the EpsMix version of the finite non-optimistic policy lemma.

\begin{lemma}
    \label{lem:epsfinite}
    Define $\mathcal{N}=\{k|k\in[K],\pi^k\neq \gp^{k}\}$. Then, the cardinality of $\mathcal{N}$ in the EpsMix algorithm can be bounded as
    \[|\mathcal{N}|\le \left(\frac{4608 e^{26}}{\kappa^2}+\frac{2688 e^{13}S}{\kappa} \right)H^3SA\log^2(K+1)=\tilde{O}\left(\left(\frac{1}{\kappa^2}+\frac{S}{\kappa} \right)H^3SA \right),\]
    where $\kappa=V_1^{\pi^b}-\gamma$.
\end{lemma}

\begin{proof}
    If $\pi^k\neq \hat{\pi}^{k}$, we must have $\utilde{V}_1^{k}<\gamma=V_1^{\pi^b}-\kappa$.
    There are two possible cases for $\utilde{V}_1^{k,b}$. Case 1: $\utilde{V}_1^{k,b}<\gamma=V_1^{\pi^b}-\kappa$. For this case, the algorithm will choose $\pi^k=\pi^{b}$. Thus, $V_1^{\pi^b} - \utilde{V}_1^{k,b}>\kappa $. Case 2: $\utilde{V}_1^{k,b}\ge\gamma$. For this case, the algorithm will choose $\pi^k=\rho\gp^{k} \oplus (1-\rho)\pi^{b}$. 
    The design of $\rho$ ensures that $\rho \utilde{V}_1^{k} + (1 - \rho) \utilde{V}_1^{k,b}=\gamma$. Therefore, $V_1^{\pi^b} - (\rho \utilde{V}_1^{k} + (1 - \rho) \utilde{V}_1^{k,b})=\kappa$. We note that $\pi^{k}=\pi^{b}$ can also be viewed as $1\cdot \pi^{b}\oplus 0\cdot\gp^{k}$. Thus, for any $k\in\mathcal{N}$, we have
    \[V_1^{\pi^b} - (\rho \utilde{V}_1^{k} + (1 - \rho) \utilde{V}_1^{k,b})\ge\kappa.\] 
    Furthermore, due to the optimality of $\pi^\star$, we have $V_1^{\pi^b}\le \rho V_1^{\star} + (1 - \rho) V_1^{\pi^b}$. Thus,
    \begin{align*}
        |\mathcal{N}|\kappa \le & \sum_{k\in\mathcal{N}} \left( V_1^{\pi^b} - (\rho \utilde{V}_1^{k} + (1 - \rho) \utilde{V}_1^{k,b}) \right) \\
        \le & \sum_{k\in\mathcal{N}} \left( (\rho V_1^{\star} + (1 - \rho) V_1^{\pi^b}) - (\rho \utilde{V}_1^{k} + (1 - \rho) \utilde{V}_1^{k,b}) \right)\\
        = & \sum_{k\in\mathcal{N}} \left( \rho (V_1^{\star} - \utilde{V}_1^{k}) + (1 - \rho) (V_1^{\pi^b} - \utilde{V}_1^{k,b}) \right) \\
        \overset{(a)}{\le} & \sum_{k\in\mathcal{N}} \left( \rho \gp^{k}_1G^{k}_1 + (1 - \rho) \pi^{b}_1 G^{k,b}_1 \right),
    \end{align*}
    where inequality $(a)$ is based on \Cref{lem:epsgap}.
    
    Then, leveraging \cref{lem:Gboundcomb_eps}, we have
    \begin{align*}
         |\mathcal{N}|\kappa & \le   24 e^{13} H \sum_{k\in\mathcal{N}}\sqrt{\sum_{h=1}^H \sum_{s,a} d_h^{k}(s,a) \left(\frac{\beta^\star(\bar{n}^k_h(s,a),\delta')}{\bar{n}^k_h(s,a)\vee 1} \right)} \\
         &\quad\quad + 336 e^{13} H^2 \sum_{k\in\mathcal{N}}\sum_{h=1}^H \sum_{s,a} d_h^{k}(s,a)\left(\frac{\beta(\bar{n}^k_h(s,a),\delta' )}{\bar{n}^k_h(s,a)\vee 1} \right) \\
        & \overset{(b)}{\le}  24 e^{13} H \sqrt{|\mathcal{N}|} \sqrt{\sum_{k\in\mathcal{N}}\sum_{h=1}^H \sum_{s,a} d_h^{k}(s,a) \left(\frac{\beta^\star(\bar{n}^k_h(s,a),\delta')}{\bar{n}^k_h(s,a)\vee 1} \right)} \\
        &\quad\quad + 336 e^{13} H^2 \sum_{k\in\mathcal{N}}\sum_{h=1}^H \sum_{s,a} d_h^{k}(s,a)\left(\frac{\beta(\bar{n}^k_h(s,a),\delta')}{\bar{n}^k_h(s,a)\vee 1} \right) ,
    \end{align*}
    where inequality $(b)$ follows from the Cauchy's inequality and $\delta'=\delta/4$.
    
    Similar to \Cref{eqn:summation1} and \Cref{eqn:summation2} in \Cref{lem:finite}, we have
    \begin{equation}
        \label{eqn:epssum1}
        \sum_{k\in\mathcal{N}}\sum_{h=1}^H \sum_{s,a} d_h^k(s,a) \left(\frac{\beta^\star(\bar{n}^k_h(s,a),\delta')}{\bar{n}^k_h(s,a)\vee 1} \right) \le 4HSA \beta^\star(K,\delta')\log{(|\mathcal{N}|+1)},
    \end{equation}
    and
    \begin{equation}
        \label{eqn:epssum2}
        \sum_{k\in\mathcal{N}}\sum_{h=1}^H \sum_{s,a} d_h^k(s,a) \left(\frac{\beta(\bar{n}^k_h(s,a),\delta')}{\bar{n}^k_h(s,a)\vee 1} \right) \le 4HSA \beta(K,\delta')\log{(|\mathcal{N}|+1)}.
    \end{equation}
    Therefore, we have
    \[|\mathcal{N}|\kappa\le 48 e^{13}  \sqrt{|\mathcal{N}|} \sqrt{H^3SA\beta^\star(K,\delta')\log{(|\mathcal{N}|+1)}} + 1344 e^{13} H^3 SA\beta(K,\delta')\log{(|\mathcal{N}|+1)}.\]
    
    By rearranging terms, we conclude that $|\mathcal{N}|\le (\frac{4608e^{26}}{\kappa^2}+\frac{2688e^{13}S}{\kappa})H^3SA\log^2(K+1)=\tilde{O}((\frac{1}{\kappa^2}+\frac{S}{\kappa})H^3SA).$
\end{proof}


Finally, we are ready to prove the regret upper bound of EpsMix.

\if{0}
\begin{theorem}[Regret of EpsMix]
    \label{thm:eps}
    Given a safe baseline $\pi^b$  and a performance constraint $\gamma$ satisfying $V^{\pi^b}\ge \gamma$, by setting $\beta=\log(SAH/\delta') + S\log(8e(K+1)), \beta^\star = \log(SAH/\delta') + \log(8e(K+1)), \delta' = \frac{\delta}{4}$, the EpsMix algorithm does not violate \Cref{eqn:constraint}, and the regret of EpsMix can be upper bounded by:
    \begin{align*}
        \mbox{Reg}(K) &\leq 48 e^{13} \sqrt{H^3SAK\beta^\star(K,\delta')\log{(K+1)}} + 1344 e^{13} H^3 SA\beta(K,\delta')\log{(K+1)} \\
        &\quad + \left(\frac{4608 e^{26}}{\kappa^2}+\frac{2688 e^{13}S}{\kappa}\right)H^3SA\log^2(K+1)\Delta_0 \\
        &=  \tilde{O}\left(\sqrt{H^3SAK}+H^3S^2A+H^3SA\Delta_0\left(\frac{1}{\kappa^2}+\frac{S}{\kappa}\right)\right),
    \end{align*}
    where $\Delta_0=V_1^\star-V_1^{\pi^b}$, $\kappa = V_1^{\pi^b} - \gamma$.
\end{theorem}
\fi
\begin{proof}[Proof of \Cref{thm:episodic}]
First, we have
    \begin{align*}
        \sum_{k=1}^K \left( V_1^{\star} - V_1^{\pi^k} \right) &= \sum_{k\notin\mathcal{N}} \left(V_1^{\star} - V_1^{\pi^k}\right) + \sum_{k\in\mathcal{N}} \left(V_1^{\star} - V_1^{\pi^b}\right) + \sum_{k\in\mathcal{N}} \left(V_1^{\pi^b} - V_1^{\pi^k}\right) \\
        & \le  \sum_{k\notin\mathcal{N}} \left(V_1^{\star} - \utilde{V}_1^{k}\right) + \sum_{k\in\mathcal{N}}  \biggr(\rho^k (V_1^{\star} - \utilde{V}_1^{k}) + (1 - \rho^k) (V_1^{\pi^b} - \utilde{V}_1^{k,b})\biggr) + |\mathcal{N}|\Delta_0  \\
        & =  \sum_{k=1}^{K}  \biggr(\rho^k (V_1^{\star} - \utilde{V}_1^{k}) + (1 - \rho^k) (V_1^{\pi^b} - \utilde{V}_1^{k,b})\biggr) +
        |\mathcal{N}|\Delta_0 \\
        & \overset{(a)}{\le}  \sum_{k=1}^{K}  \biggr(\rho^k \gp_1^{k}G_1^{k} + (1 - \rho^k) \pi_1^{b}G_1^{k,b}\biggr) +
        |\mathcal{N}|\Delta_0, 
    \end{align*}
    where inequality $(a)$ follows from \Cref{lem:epsgap}. 
    
    Then, we use \Cref{lem:Gboundcomb_eps} and the result of \Cref{lem:epsfinite} to bound the regret as follows:
    \begin{align*}
        \sum_{k=1}^K & \left( V_1^{\star} - V_1^{\pi^k} \right) \\
        \le & \sum_{k=1}^{K} \biggr( 24e^{13} H \sqrt{\sum_{h=1}^H\sum_{s,a}d_h^k(s,a)\biggr(\frac{\beta(\bar{n}^k_h(s,a),\delta')}{\bar{n}^k_h(s,a)\vee 1}\biggr)} + 336 e^{13}H^2\sum_{h=1}^H\sum_{s,a}d_h^k(s,a)\biggr(\frac{\beta(\bar{n}^k_h(s,a),\delta')}{\bar{n}^k_h(s,a)\vee 1}\biggr) \biggr) \\
        & + \biggr(\frac{4608e^{26}}{\kappa^2}+\frac{2688e^{13}}{\kappa}\biggr)H^3SA\log(K+1)\Delta_0 \\
        \le & 24e^{13} H \sqrt{K} \sqrt{\sum_{k=1}^{K} \sum_{h=1}^H\sum_{s,a}d_h^k(s,a)\biggr(\frac{\beta(\bar{n}^k_h(s,a),\delta')}{\bar{n}^k_h(s,a)\vee 1}\biggr)} + 336 e^{13}H^2\sum_{k=1}^{K} \sum_{h=1}^H\sum_{s,a}d_h^k(s,a)\biggr(\frac{\beta(\bar{n}^k_h(s,a),\delta')}{\bar{n}^k_h(s,a)\vee 1}\biggr) \\
        & +\biggr(\frac{4608e^{26}}{\kappa^2}+\frac{2688e^{13}S}{\kappa}\biggr)H^3SA\log(K+1)\Delta_0,
    \end{align*}
    where the last inequality is due to the Cauchy's inequality. Then we use the bound of summation in \Cref{lem:epsfinite} and plug \Cref{eqn:epssum1,eqn:epssum2} into the above inequality, to conclude that 
    \begin{align*}
         \mbox{Reg}(K)&=\sum_{k=1}^K \left( V_1^{\star} - V_1^{\pi^k}\right) \\
         &\le 48 e^{13} \sqrt{H^3SAK\beta^\star(K,\delta')\log{(K+1)}} + 1344 e^{13} H^3 SA\beta(K,\delta')\log{(K+1)} \\
        &\quad\quad +\left(\frac{4608 e^{26}}{\kappa^2}+\frac{2688 e^{13}S}{\kappa} \right)H^3SA\log^2(K+1)\Delta_0 \\
        &= \tilde{O}\left(\sqrt{H^3SAK}+H^3S^2A+H^3SA\Delta_0\left(\frac{1}{\kappa^2}+\frac{S}{\kappa} \right) \right).
    \end{align*}
\end{proof}

\section{From Baseline Policy to Offline Dataset}
\label{appx:offline}

\subsection{Offline Algorithm}
\label{appx:offline1}

The offline VI-LCB algorithm is detailed in \Cref{alg:offline}. 

\begin{algorithm}[!hbt]
    \caption{Offline VI-LCB (Algorithm 3 in \citet{xie2021policy})}
    \label{alg:offline}
\begin{algorithmic}
\REQUIRE Dataset $D=\{(s_h^{(i)},a_h^{(i)},r_h^{(i)},s_{h+1}^{(i)})_{h=1}^H\}^n_{i=1}$ collected using an unknown baseline policy $\mu$
\STATE Randomly divide $D$ into $H$ sets $\{D_h\}_{h=1}^H$ such that $|D_h|=n/H$.
\STATE Estimation $\hP_h(s'|s,a)$ and $b_h(s,a)$ using $D_h$.
\STATE Set $\hat{V}_{H+1}(s,a) = 0, \forall s,a$.
\FOR{$h$ = $H$ to $1$}
    \STATE $\hat{Q}_{h}(s,a)=\max(0, r_h(s,a) + \hP_h \hat{V}_{h+1}(s,a) - b_h(s,a)), \forall s,a$.
    \STATE Let $\hat{\pi}_h(s)=\argmax_a \hat{Q}_{h}(s,a), \forall s$.
    \STATE $\hat{V}_h(s) = \hat{Q}_h(s, \hat{\pi}_h(s)), \forall s$.
\ENDFOR
\STATE Return $\hat{\pi}$.
\end{algorithmic}
\end{algorithm}

\subsection{Theoretical Analysis}

In \cref{alg:offline}, $\hP_h(s'|s,a)$ and $b_h(s,a)$ are defined as:
\begin{align*}
    \hP_h(s'|s,a) &= \frac{n_h(s,a,s')}{1\vee n_h(s,a)},    \quad b_h(s,a)=c \sqrt{\frac{H^2\iota}{n_h(s,a)\vee 1}}, 
\end{align*}
where $\iota = \log(HSA/\delta)$, $n_h(s,a)=\sum_{s_h,a_h\in D_h}\mathds{1}\{s_h=s,a_h=a\}$ is the count of visitations of state-action pair $(s,a)$ at step $h$, and $n_h(s,a,s')=\sum_{s_h,a_h,s_{h+1}\in D_h}\mathds{1}\{s_h=s,a_h=a,s_{h+1}=s'\}$ is the count of visiting state-action pair $(s,a)$ at step $h$ while having state $s'$ as the next state. Both counts are only for samples in dataset $D_h$. 

Following the approach in \citet{xie2021policy}, we first define the good events as follows.

\begin{lemma}[Lemma B.1 in \citet{xie2021policy}]
    \label{lem:offlinegood}
    With probability at least $1-\delta$, there exists a finite constant $c$ such that the following good events hold: 
    \begin{align*}
        \forall h\in[H], (s,a)\in \Sc\times\Ac,~~ |(P_h-\hP_h)\hat{V}_{h+1}(s,a)| &\le c \sqrt{\frac{H^2\iota}{n_h(s,a)}}=b_h(s,a),\quad   \frac{1}{n_h(s,a)} \le c \frac{H\iota}{n d^{\mu}(s,a)},
    \end{align*}
    where $\iota=\log(HSA/\delta)$ and $d^{\mu}(s,a)$ is the occupancy measure under the behavior policy $\mu$.
\end{lemma}

Under the good events, $\hat{Q}_h(s,a)$ can be proved to be the lower confidence bound of  $Q_h^{\hat{\pi}}(s,a)$ 
as shown in the following lemma.

\begin{lemma}[Lemma B.2 in \citet{xie2021policy}]
    Let $\hat{Q}_h(s,a)=\max(0,r_h(s,a)+\hP_h\hat{V}_{h+1}(s,a)-b_h(s,a))$. Then, {under the good events defined in \Cref{lem:offlinegood}, we have}
    \[\hat{Q}_h(s,a)\le Q_h^{\hat{\pi}}(s,a).\]
\end{lemma}

The above two lemmas are the same as \citet{xie2021policy}, and thus we omit the proofs. 

We first bound the value difference $V_1^{\mu}-V_1^{\hat{\pi}}$ with the offline estimation bonus $b_h^k(s,a)$ in the following lemma.

\begin{lemma}
    \label{lem:offline_bound_by_bonus}
    Suppose there are $n$ trajectories collected under the behavior policy $\mu$. Then, {under the good events defined in \Cref{lem:offlinegood}}, the extracted policy $\hat{\pi}$ satisfies
    \[V_1^{\mu}-V_1^{\hat{\pi}}\le 2\sum_{h=1}^H \sum_{(s,a)\in\mathcal{S}\times\mathcal{A}} d^\mu_h(s,a) b_h(s,a),\]
    where $b_h(s,a)=c\sqrt{\frac{H^2\iota}{n_h(s,a)\vee 1}}$. 
\end{lemma}
\begin{proof}
    We directly calculate the suboptimality gap as follows
    \begin{align*}
         V_h^{\mu}(s)-V_h^{\hat{\pi}}(s) = & V_h^{\mu}(s)-\max_a Q_h^{\hat{\pi}}(s, a) \\
        \le & \mathbb{E}_{a\sim \mu_h(\cdot|s)}[Q_h^{\mu}(s,a) - Q_h^{\hat{\pi}}(s,a)] \\
        \le & \mathbb{E}_{a\sim \mu_h(\cdot|s)}[b_h(s,a) + P_h V^\mu_{h+1}(s,a) - \hP_h \hat{V}_{h+1}(s,a)] \\
        = & \mathbb{E}_{a\sim \mu_h(\cdot|s)}[b_h(s,a) + P_h (V^\mu_{h+1}-\hat{V}_{h+1})(s,a) + (P_h - \hP_h) \hat{V}_{h+1}(s,a)] \\
        \le & 2 \mathbb{E}_{a\sim \mu_h(\cdot|s)}[b_h(s,a)] + \mathbb{E}_{a\sim \mu_h(\cdot|s), s'\sim P(\cdot|s,a)}[V_{h+1}^{\mu}(s')-V_{h+1}^{\hat{\pi}}(s')].
    \end{align*}
    Recursively unfolding the above inequality from $h=1$, we have: \[V_1^{\mu}-V_1^{\hat{\pi}}\le 2\sum_{h=1}^H \sum_{(s,a)\in\mathcal{S}\times\mathcal{A}} d^\mu_h(s,a) b_h(s,a).\]
\end{proof}

The following theorem establishes an upper bound for the gap between the learned policy $\hat{\pi}$ and the behavior policy $\mu$.

\begin{theorem}[Adapted from Theorem 1 in \citet{xie2021policy}]
    Suppose $n$ trajectories are collected in the offline dataset collected under policy $\mu$. Then, with probability at least $1-\delta$, the output policy $\hat{\pi}$ of the offline  Algorithm  VI-LCB satisfies 
    \[V_1^\mu - V_1^{\hat{\pi}}\le 2c\iota\sqrt{\frac{H^5SA}{n}},\]
    where $\iota=\log(HSA/\delta)$.
\end{theorem}

\begin{proof}
 Under the good events defined in \Cref{lem:offlinegood}, we have
    \begin{align*}
        V_1^\mu - V_1^{\hat{\pi}} & \overset{(a)}{\le} 2\sum_{h=1}^H\sum_{(s,a)\in\mathcal{S}\times\mathcal{A}} d^\mu_h(s,a) b_h(s,a)\\
        & \le 2 c \sum_{h=1}^H\sum_{(s,a)\in\mathcal{S}\times\mathcal{A}} d^\mu_h(s,a) \sqrt{\frac{H^2\iota}{n_h(s,a)\vee 1}} \\
        & \overset{(b)}{\le} 2 c  \sqrt{H^2\iota} \sum_{h=1}^H\sum_{(s,a)\in\mathcal{S}\times\mathcal{A}} d^\mu_h(s,a) \sqrt{\frac{H\iota}{n d^\mu_h(s,a)}} \\
        & \le 2 c \iota\sqrt{\frac{H^3}{n}} \sum_{h=1}^H\sum_{(s,a)\in\mathcal{S}\times\mathcal{A}} \sqrt{d_h^\mu(s,a)} \\
        & \overset{(c)}{\le} 2c\iota\sqrt{\frac{H^3}{n}} \sqrt{ \sum_{h=1}^H\sum_{(s,a)\in\mathcal{S}\times\mathcal{A}} 1} \sqrt{\sum_{h=1}^H\sum_{(s,a)\in\mathcal{S}\times\mathcal{A}} d_h^\mu(s,a)} \\
        & =  2c\iota\sqrt{\frac{H^5SA}{n}},
    \end{align*}
   where $(a)$ is from \Cref{lem:offline_bound_by_bonus}, $(b)$ follows from the definition of good events in \Cref{lem:offlinegood}, and $(c)$ is based on the Cauchy's inequality.
\end{proof}

Based on the above theorem, we have the following corollary regarding the sample complexity.
\begin{corollary}
    \label{coro:off}
    With probability at least $1-\delta/2$, if $n\ge\frac{16c^2\iota'^2H^5SA}{(V_1^\mu-\gamma)^2}$ and $V_1^\mu > \gamma$, the output $\hat{\pi}$ of offline VI-LCB satisfies $V_1^{\hat{\pi}}\ge (V_1^\mu+\gamma)/2$, where $\iota'=\log(2HSA/\delta)$. 
\end{corollary}

\begin{proof}
    To ensure $V_1^{\hat{\pi}}>(V_1^{\mu}+\gamma)/2$, we need to establish $V_1^\mu - V_1^{\hat{\pi}}<(V_1^\mu-\gamma)/2$. We note that if 
    \begin{equation}
        \label{eqn:corooff1}
        2c\iota\sqrt{\frac{H^5SA}{n}}<(V_1^\mu-\gamma)/2, 
    \end{equation}
    then $V_1^\mu - V_1^{\hat{\pi}}<(V_1^\mu-\gamma)/2$. Rearranging the terms in \Cref{eqn:corooff1} leads to 
    \begin{equation*}
        n> \frac{16c^2\iota'^2H^5SA}{(V_1^\mu-\gamma)^2},
    \end{equation*}
    which completes the proof.
\end{proof}

Combining \Cref{coro:off} and \Cref{thm:step}, we can prove the following theorem.
\begin{theorem}
    \label{thm:offstep_appx}
    Assume that there are at least $n\ge\frac{16c^2\iota'^2H^5SA}{(V_1^\mu-\gamma)^2}$ offline trajectories collected under a safe behavior policy $\mu$. If we let \Cref{alg:offline} run on the offline dataset and pass the output $\hat{\pi}$ to \Cref{alg:step} as the baseline $\pi^b$, then, with probability at least $1-\delta$, StepMix does not violates the constraint in \Cref{eqn:constraint} and achieves a regret that scales in 
    \[\tilde{O}\left(\sqrt{H^3SAK}+H^3S^2A+H^3SA\bar{\Delta}_0 \left(\frac{1}{\bar{\kappa}^2}+\frac{S}{\bar{\kappa}} \right) \right),\]
    where $\bar{\kappa}=(V_1^\mu-\gamma)/2 > 0$ and $\bar{\Delta}_0=V^\star_1-V^\mu_1+\bar{\kappa}$.
\end{theorem}

\begin{proof}
    \Cref{coro:off} states that with $n\ge\frac{16c^2\iota'^2H^5SA}{(V_1^\mu-\gamma)^2}$ offline trajectories, $\Pb[V_1^{\hat{\pi}}>(V_1^{\mu}+\gamma)/2]\ge 1-\delta/2$. Denote $A$ the event that $\hat{\pi}$ satisfies $V^{\hat{\pi}}>(V_1^{\mu}+\gamma)/2$. Then, we have $\mathbb{P}[A]\ge 1-\delta/2$. 
    
    \Cref{thm:step} states that if $V^{\hat{\pi}}>(V_1^{\mu}+\gamma)/2>\gamma$, with probability $1-\delta/2$, StepMix does not violate the constraint and achieves a regret at most
    \begin{align*}
        & 48 e^{13} \sqrt{H^3SAK\beta^\star(K,\delta')\log{(K+1)}} + 1344 e^{13} H^3 SA\beta(K,\delta')\log{(K+1)} \\
        & \quad\quad + \left(\frac{4608 e^{26}}{\bar{\kappa}^2}+\frac{2688 e^{13}S}{\bar{\kappa}} \right) H^3SA\log^2(K+1)\Delta_0 \\
        &\quad = \tilde{O} \left( \sqrt{H^3SAK}+H^3S^2A+H^3SA\bar{\Delta}_0 \left(\frac{1}{\bar{\kappa}^2}+\frac{S}{\bar{\kappa}} \right) \right) ,
    \end{align*}
    where $\delta'=\frac{\delta}{6(H+1)}$, $V_1^{\hat{\pi}}-\gamma\ge\bar{\kappa}= (V_1^\mu - \gamma)/2$, and $\bar{\Delta}_0=V^\star_1-V^\mu_1+\bar{\kappa}$. 
    
    Since we use $\hat{\pi}$ as baseline, by letting $B$ denote the event that StepMix achieves the regret in \cref{thm:offstep_appx} and does not violate the constraint, we have $\Pb[B|A]\ge 1-\delta/2$.
    Because $\Pb[B]=\Pb[B|A] \Pb[A]=(1-\delta/2)(1-\delta/2)\ge 1-\delta$, we have that, with overall probability at least $1-\delta$, when StepMix uses the output of offline UCB-VI as the baseline policy, it achieves a regret that is at most 
    \[\tilde{O}\left(\sqrt{H^3SAK}+H^3S^2A+H^3SA\bar{\Delta}_0 \left(\frac{1}{\bar{\kappa}^2}+\frac{S}{\bar{\kappa}} \right) \right)\]
    without violating the constraint.
\end{proof}

We can also combine \Cref{coro:off} and \Cref{thm:episodic} to prove the following theorem for EpsMix, which is similar to \Cref{thm:offstep_appx} for StepMix.
\begin{theorem}
    \label{thm:offeps}
    Assume that there are at least $n\ge\frac{16c^2\iota'^2H^5SA}{(V_1^\mu-\gamma)^2}$ trajectories collected under a safe behavior policy $\mu$. If we let \Cref{alg:offline} run on the offline dataset and pass the output $\hat{\pi}$ to \Cref{alg:eps} as the baseline $\pi^b$, then, with probability at least $1-\delta$, EpsMix does not violate the constraint in \Cref{eqn:constraint} and achieves a regret that scales in
    \[\tilde{O}\left(\sqrt{H^3SAK}+H^3S^2A+H^3SA \bar{\Delta}_0 \left(\frac{1}{\bar{\kappa}^2}+\frac{S}{\bar{\kappa}} \right) \right),\]
    where $\bar{\kappa}=(V_1^\mu-\gamma)/2 > 0$ and $\bar{\Delta}_0=V^\star_1-V^\mu_1+\bar{\kappa}$.
\end{theorem}

The proof is similar to that of \Cref{thm:offstep_appx} and is thus omitted. 

\section{Technical Lemmas}
In this section, we list several technical lemmas that are used in the main proof.
\begin{lemma}[Lemma 11 in \citet{menard2021fast}]
    \label{lem:klvar}
    Let $p$ and $q$ be two probability distributions supported by the state set $\Sc$, and $f$ be a function on $\Sc$. If $\mbox{KL}(p,q)\le\alpha$, $0\le f(s)\le b, \forall s\in \Sc$, we have
    \[\mathrm{Var}_q(f)\le 2\mathrm{Var}_p(f)+4b^2\alpha,\]
    \[\mathrm{Var}_p(f)\le 2\mathrm{Var}_q(f)+4b^2\alpha.\]
\end{lemma}

\begin{lemma}[Lemma 12 in \citet{menard2021fast}]
    \label{lem:varerr}
    Let $p$ and $q$ be two probability distributions supported by the state set $\Sc$, and $f,g$ be functions on $\Sc$. If $0\le g(s), f(s) \le b,\forall s\in\Sc$, we have
    \[\mathrm{Var}_p(f)\le 2\mathrm{Var}_p(g)+2 b \Eb_p[|f-g|],\]
    \[\mathrm{Var}_q(f)\le \mathrm{Var}_p(f) + 3b^2\left\|p-q\right\|_1.\]
\end{lemma}

\begin{lemma}[Lemma 10 in \citet{menard2021fast}]
    \label{lem:klabs}
    Let $p$ and $q$ be two probability distributions supported by the state set $\Sc$, and $f$ be a function on $\Sc$. If $\mbox{KL}(p,q)\le\alpha$ and $0\le f\le b$, we have
    \[|\Eb_p[f]-\Eb_q[f]|\le \sqrt{2\Var_q(f)\alpha}+\frac{2}{3} b\alpha.\]
\end{lemma}

\begin{lemma}[Lemma 6 in \citet{huang2022safe}]
    \label{lem:variancesum}
    Given transition kernel $P_h$, policy $\pi$ and reward $r_h:\Sc\times\Ac\rightarrow [0,1]$, we have:
    \[\sum_{h=1}^H \sum_{s,a} d_h^{\pi}(s,a) \Var_{P_h}(V_{h+1}^{\pi})(s,a)=\mathbb{E}_{\pi,P}\left[\left(\sum_{h=1}^H r_h(s_h,a_h) - V_1^\pi(s_1)\right)^2\right]\le H,\]
    where $d_h^\pi$ is the occupancy measure under policy $\pi$ and $V_h^\pi$ is the value function.
\end{lemma}

\begin{lemma}
    \label{lem:cnt}
    Under event $\mathcal{E}^\text{cnt}$ and using the same notations defined in \Cref{table:notation}, we have
    \[\left(\frac{\beta(n^k_h(s,a),\delta')}{n^k_h(s,a)}\wedge 1\right)\le 4 \left(\frac{\beta(\bar{n}^k_h(s,a),\delta')}{\bar{n}^k_h(s,a)\vee 1}\right).\]
   Similarly, for $\beta^\star$, the following inequality holds: 
   \[\left(\frac{\beta^\star(n^k_h(s,a),\delta')}{n^k_h(s,a)}\wedge 1\right)\le 4 \left(\frac{\beta^\star(\bar{n}^k_h(s,a),\delta')}{\bar{n}^k_h(s,a)\vee 1}\right).\]
\end{lemma}

\begin{proof}
    Event $\mathcal{E}^\text{cnt}$ means that
    \[n_h^k(s,a)\ge \frac{1}{2}\bar{n}^k_h(s,a)-\beta^{\text{cnt}}(\delta').\]
    
    If $\beta^{\text{cnt}}(\delta')\le\frac{1}{4}\bar{n}^k_h(s,a)$, we directly have the result $n_h^k(s,a)\ge \frac{1}{4}\bar{n}^k_h(s,a)$, which proves $\left(\frac{\beta(n^k_h(s,a),\delta')}{n^k_h(s,a)}\wedge 1\right)\le 4 \left(\frac{\beta(\bar{n}^k_h(s,a),\delta')}{\bar{n}^k_h(s,a)\vee 1}\right)$. On the other hand, if $\beta^{\text{cnt}}(\delta')\ge\frac{1}{4}\bar{n}^k_h(s,a)$, {based on the fact that $\beta(n^k_h(s,a),\delta')\ge\beta^{\text{cnt}}(\delta')>1$}, we have that $\left(\frac{\beta(n^k_h(s,a),\delta')}{n^k_h(s,a)}\wedge 1\right)\le 1\le 4 \left(\frac{\beta(\bar{n}^k_h(s,a),\delta')}{\bar{n}^k_h(s,a)\vee 1}\right)$. 
    {The same arguments can also be applied to the inequality with $\beta^\star$.}
\end{proof}

\begin{lemma}
    \label{lem:partial_sum}
    For any state-action pair $(s,a)$, step $h$ and a subset of episodes $\mathcal{N}\subseteq [K]$, we have
    \[\sum_{k\in\mathcal{N}} \frac{d_h^k(s,a)}{\bar{n}^k_h(s,a)\vee 1}\le 4\log(|\mathcal{N}|+1),\]
    where $d_h(s,a)$ is the occupancy measure, $\bar{n}^k_h(s,a)$ is the expected visitation count.
\end{lemma}

\begin{proof}
    We have
    \begin{align*}
        \sum_{k\in\mathcal{N}} \frac{d_h^k(s,a)}{\bar{n}^k_h(s,a)\vee 1} 
        & = \sum_{k\in\mathcal{N}} 
        \frac{d_h^k(s,a)}{(\sum_{t=1}^{h-1}d_t^k(s,a))\vee 1} \\
        & \le \sum_{k\in\mathcal{N}} 
        \frac{d_h^k(s,a)}{(\sum_{t\in \mathcal{N},t<k}d_t^k(s,a))\vee 1}\\
        & \le \sum_{k\in\mathcal{N}} 
        \frac{4 d_h^k(s,a)}{2\sum_{t\in \mathcal{N},t<k}d_t^k(s,a) + 2} \\
        & \le 4\sum_{k\in\mathcal{N}} 
        \frac{d_h^k(s,a)}{\sum_{t\in \mathcal{N},t<k}d_t^k(s,a) + d_h^k(s,a) + 1} \\
        & \le 4\sum_{k\in\mathcal{N}} 
        \frac{d_h^k(s,a)}{\sum_{t\in \mathcal{N},t\le k}d_t^k(s,a)+ 1}. 
    \end{align*}
    Then, by noting $f(k)=\sum_{t\in \mathcal{N},t\le k}d_t^k(s,a)$ and $k'=\max_t \{t\in\mathcal{N}\cup\{0\}|t<k\}$, we have
    \begin{align*}
        4\sum_{k\in\mathcal{N}} 
        \frac{d_h^k(s,a)}{\sum_{t\in \mathcal{N},t\le k}d_t^k(s,a)+ 1} &\le 4\sum_{k\in\mathcal{N}} 
        \frac{f(k)-f(k')}{f(k)+ 1}\\
        &\le 4\sum_{k\in\mathcal{N}} 
        \int_{x=f(k')}^{f(k)}\frac{dx}{x+1}\\
        & \le 4\int_{x=1}^{|\mathcal{N}|+1} \frac{1}{x}dx \\
        & = 4\log(|\mathcal{N}|+1).
    \end{align*}
\end{proof}

\begin{theorem}[Proposition 1 in \citet{jonsson2020planning}]
    \label{thm:goodevent1}
    For a categorical distribution with probability distribution $p\in\Sigma_m$, denoting $\hP_n$ as a frequency estimation of $p$, we have
    \[\mathbb{P}\bigg(\exists n\in \mathbb{N}^\star, n \mbox{KL}(\hat{P}_n,p) > \log(1/\delta)+(m-1)\log(e(1+n/(m-1)) \bigg)\le\delta.\]
\end{theorem}

\begin{theorem}[Lemma F.4 in \citet{dann2017unifying}] 
    \label{thm:goodevent2}
    Let $\{\mathcal{F}_t\}_{t=1}^n$ be a filtration, $\{X_t\}_{t=1}^n$ be a series of Bernoulli random variables with $\mathbb{P}[{X}_t=1|\mathcal{F}_{t-1}]=p_t$, where $p_t$ is $\mathcal{F}_{t-1}$-measurable. Then
    \[\forall \delta>0, \mathbb{P}\left(\exists n:\sum_{t=1}^n {X}_t<\sum_{t=1}^n p_t/2-\log(1/\delta) \right)\le\delta.\]
\end{theorem}

\begin{theorem}[Theorem 5 in \citet{menard2021fast}, Lemma 3 in \cite{domingues2021kernel}]
    \label{thm:goodevent3}
    Suppose $(Y_t)_{t\in\mathbb{N}}$ and $(w_t)_{t\in\mathbb{N}}$ are two sequences from filtration $(\mathcal{F}_t)_{t\in\mathbb{N}}$, subject to $w_t\in[0,1]$, $|Y_t|\le b$ and $\mathbb{E}[Y_t|\mathcal{F}_t]=0$. Define
    \[\mathcal{S}_t=\sum_{s=1}^t w_s Y_s,\hspace{10pt} \mathcal{V}_t=\sum_{s=1}^t w_s^2 \mathbb{E}[Y_s^2|\mathcal{F}_s],\]
    then, for any $\delta\in(0,1)$, we have
    \[\mathbb{P} \left(\exists t\ge 1, (\mathcal{V}_t/b^2+1)h\left(\frac{b|\mathcal{S}_t|}{\mathcal{V}_t+b^2} \right)\ge \log(1/\delta)+\log(4e(2t+1)) \right)\le \delta,\]
    where $h(x)=(x+1)\log(x+1)-x$. 
 
    This result can be equivalently stated as: with probability at least $1-\delta$, the following inequality holds:
    \[|\mathcal{S}_t|\le \sqrt{2\mathcal{V}_t\log{4e(2t+1)/\delta}}+3b\log(4e(2t+1)/\delta).\]
\end{theorem}

\end{document}

%% file: math_commands.tex
\usepackage{amsmath,amsfonts,bm}

\def\eqref#1{equation~\ref{#1}}

\def\1{\bm{1}}

\DeclareMathAlphabet{\mathsfit}{\encodingdefault}{\sfdefault}{m}{sl}
\SetMathAlphabet{\mathsfit}{bold}{\encodingdefault}{\sfdefault}{bx}{n}

\newcommand{\Var}{\mathrm{Var}}

\DeclareMathOperator*{\argmax}{arg\,max}

%% file: personal_math_commands.tex
\usepackage{amsthm}
\usepackage{accents}

\newcommand{\Rb}{\mathbb{R}}

\newcommand{\Eb}{\mathbb{E}}
\newcommand{\Pb}{\mathbb{P}}

\newcommand{\Zb}{\mathbb{Z}}

\newcommand{\Ac}{\mathcal{A}}

\newcommand{\Dc}{\mathcal{D}}

\newcommand{\Mc}{\mathcal{M}}
\newcommand{\Nc}{\mathcal{N}}
\newcommand{\Sc}{\mathcal{S}}

\newcommand{\bunderline}[1]{\underline{#1\mkern-4mu}\mkern4mu }

\newcommand{\utilde}[1]{\underaccent{\tilde}{#1}}

\newcommand{\jing}[1]{{\color{red}#1}}

\newcommand{\hP}{\hat{P}}
\newcommand{\gp}{\bar{\pi}}